\documentclass[10pt]{article}

\usepackage{microtype}
\usepackage{graphicx}
\usepackage{subfigure}
\usepackage{booktabs} 
\usepackage{bbm}
\usepackage{comment}

\PassOptionsToPackage{table}{xcolor}
\usepackage[table]{xcolor}

\usepackage{natbib}
\bibpunct{(}{)}{;}{a}{,}{,}

\usepackage{url}            
\usepackage[backref=page]{hyperref}       

\hypersetup{
    bookmarks=true,         
    unicode=false,          
    pdftoolbar=true,        
    pdfmenubar=true,        
    pdffitwindow=false,     
    pdfstartview={FitH},    
    pdftitle={Sample-efficient Learning of Concepts},    
    pdfauthor={Hidde Fokkema},
    pdfsubject={Machine Learning},   
    pdfcreator={pdflatex},   
    pdfproducer={Producer}, 
    pdfkeywords={Interpretability} {Concepts} {Causal Representation Learning}
    pdfnewwindow=true,      
    colorlinks=true,        
    linkcolor=black,       
}

\usepackage{thm-restate}

\usepackage{macros}

\usepackage{arxiv}

\usepackage[textsize=tiny]{todonotes}

\usepackage{tikz-cd}
\usetikzlibrary{fit}
\usetikzlibrary{calc}

\usetikzlibrary{patterns.meta}

\tikzset{
    noise/.style={
        circle,
        draw=black, 
        fill=black!30, 
        very thick, 
        inner sep=0pt,
        minimum size=8mm},
    feat/.style={
        circle,
        draw=black, 
        fill=orange, 
        very thick, 
        inner sep=0pt,
        minimum size=8mm},
    rep/.style={
        circle,
        draw=black, 
        fill=blue!50, 
        very thick, 
        inner sep=0pt,
        minimum size=8mm},
    bicolor/.style args={#1 and #2}{
        path picture={
            \tikzset{rounded corners=0} 
            \path [fill=#1,
                postaction={
                    pattern={
                        Lines[angle=45,distance={4pt},line width=2pt]
                    },
                    pattern color=blue!40
                },
                postaction={draw, semithick}
                ] (path picture bounding box.west)
                rectangle (path picture bounding box.north east);
            \draw [black, very thick] (path picture bounding box.west)
                rectangle (path picture bounding box.north east);
            \fill [#2] (path picture bounding box.west)
                rectangle (path picture bounding box.south east);
            \draw [black, very thick] (path picture bounding box.west)
                rectangle (path picture bounding box.south east);}}}

\tikzstyle{plate} = [draw, rectangle, rounded corners, fit=#1, fill=black!5]
\tikzstyle{wrap} = [inner sep=0pt, fit=#1]
\tikzstyle{gate} = [draw, rectangle, dashed, fit=#1]

\tikzstyle{caption} = [font=\footnotesize, node distance=0] %
\tikzstyle{plate caption} = [caption, node distance=0, inner sep=0pt,
below left=5pt and 0pt of #1.south east] %
\tikzstyle{factor caption} = [caption] %
\tikzstyle{every label} += [caption] %

\DeclarePairedDelimiterX{\infdivx}[2]{(}{)}{%
  #1\;\delimsize\|\;#2%
}

\newcommand{\infdivKL}{\text{KL}\infdivx}

\newcommand{\pmin}{p_{\textup{\textrm{min}}}}
\newcommand{\pmax}{p_{\textup{\textrm{max}}}}

\newcommand{\Mvec}{\textbf{\textup{\textrm{M}}}}
\newcommand{\Hvec}{\textbf{\textup{\textrm{C}}}}

\makeatletter
\def\moverlay{\mathpalette\mov@rlay}
\def\mov@rlay#1#2{\leavevmode\vtop{%
   \baselineskip\z@skip \lineskiplimit-\maxdimen
   \ialign{\hfil$\m@th#1##$\hfil\cr#2\crcr}}}
\newcommand{\charfusion}[3][\mathord]{
    #1{\ifx#1\mathop\vphantom{#2}\fi
        \mathpalette\mov@rlay{#2\cr#3}
      }
    \ifx#1\mathop\expandafter\displaylimits\fi}
\makeatother

\newcommand{\reals}{\mathbb{R}}

\title{Sample-efficient Learning of Concepts with Theoretical Guarantees:\linebreak
from Data to Concepts without Interventions}
\shorttitle{Sample-efficient Learning of Concepts}

\author{%
    \name{Hidde Fokkema} \email{h.j.fokkema@uva.nl}\\
    \and
    \name{Tim van Erven\thanks{Equal contribution}} \email{tim@timvanerven.nl}\\
    \addr{Korteweg-de Vries Institute for Mathematics, University of Amsterdam}
    \AND
    \name{Sara Magliacane\footnotemark[1]} \email{s.magliacane@uva.nl} \\
    \addr{Informatics Institute, University of Amsterdam}
}

\begin{document}

\maketitle

\begin{abstract} 
    Machine learning is a vital part of many real-world systems,
    but several concerns remain about the lack of interpretability,
    explainability and robustness of black-box AI systems. Concept Bottleneck
    Models (CBM) address some of these challenges by learning interpretable
    \emph{concepts} from high-dimensional data, e.g.\ images, which are used to
    predict labels. An important issue in CBMs are spurious correlation between
    concepts, which effectively lead to learning ``wrong'' concepts. Current
    mitigating strategies have strong assumptions, e.g., they assume that the
    concepts are statistically independent of each other, or require
    substantial interaction in terms of both interventions and labels provided
    by annotators. In this paper, we describe a framework that provides
    theoretical guarantees on the correctness of the learned concepts and on
    the number of required labels, without requiring any interventions. Our
    framework leverages causal representation learning (CRL) methods to learn
    latent causal variables from high-dimensional observations in a
    unsupervised way, and then learns to align these variables with
    interpretable concepts with few concept labels. We propose a linear and a
    non-parametric estimator for this mapping, providing a finite-sample high
    probability result in the linear case and an asymptotic consistency result
    for the non-parametric estimator. We evaluate our framework in synthetic
    and image benchmarks, showing that the learned concepts have less
    impurities and are often more accurate than other CBMs, even in settings
    with strong correlations between concepts. 
\end{abstract}

\section{Introduction}

Machine learning is a vital part of many real-world systems, but concerns
remain about the lack of interpretability, robustness and safety of current
systems \citep{BengioEtAl2025}. These issues might be exacerbated by the lack
of guarantees in explaining the behavior of AI systems in terms of
interpretable, high-level \emph{concepts}. 
The field of interpretable machine learning and explainable AI has developed
many techniques to interpret models and explain their predictions
\citep{molnar2022}, either by extracting known concepts from the internals of
black-box models \citep{kim2018interpretability, goyal2019explaining,
graziani2023concept, lovering2022unit}, or by building the explicit use of
concepts into the internals of these systems, e.g.\ as in \emph{concept
bottleneck models} (CBM) \citep{koh2020conceptbottleneck,
ismail2023concept,marconato2022glance, zarlenga2022concept}. 

\begin{figure}[t]
    \centering
    \resizebox{0.98\textwidth}{!}{
    \begin{tikzpicture}[
        neuron/.style={
            circle, 
            draw=black!60, 
            very thick, 
            fill=white, 
            inner sep=0pt,
            minimum size=4pt
        },
        output/.style={
            circle, 
            draw=black, 
            very thick, 
            fill=white, 
            inner sep=0pt,
            minimum size=20pt
        },
        concept/.style={
            circle,
            draw=black,
            very thick,
            inner sep=0pt,
            minimum size=20pt
        },
        c_text/.style={
            rectangle, 
            rounded corners, 
            draw=black, 
            fill=pearTwo!30,
            thick, 
            anchor=west
        },
        >=stealth
    ]

    \draw[fill=red!50, rounded corners, dashed, very thick, opacity=0.6] 
        (-2.5, 2) rectangle (3, -2);
    \draw[fill=yellow!50, rounded corners, dashed, very thick, opacity=0.6] 
        (4, 2) rectangle (9, -2);
    \draw[fill=blue!50,rounded corners, dashed, very thick, opacity=0.6] 
        (1.7, 2.2) rectangle (7.1, -2.2);
    
    \node[fill=red!50, rectangle, anchor=center] 
        (exp1) at (-0.3, 2.7) {\small CRL on unlabelled data};
    \node[fill=blue!50, rectangle, anchor=center] 
        (exp2) at (4.5, 2.7) {\small Learning alignments with concept labels};
    \node[fill=yellow!50, rectangle, anchor=center] 
        (exp3) at (8.2, 2.7) {\small Label predictor};

    \node[
        rectangle, 
        rounded corners, 
        draw=black, 
        very thick, 
        fill=purple!70]
        (pic) at (-1.1, 0) 
        {\includegraphics[scale=0.8]{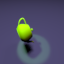}};

    \foreach \i in {1,2,3} {
        \node[neuron] (I\i) at (0,0.6-0.3*\i) {};
        \draw[->] (pic) -- (I\i);
    }

    \foreach \i in {1,2,3,4,5,6} {
        \node[neuron] (H\i) at (0.5,1.05-0.3*\i) {};
    }

    \foreach \i in {1,2,3,4,5,6} {
        \node[neuron] (H_2\i) at (1,1.05-0.3*\i) {};
    }

    \foreach \i in {1,2,3,4} {
        \node[output, fill=blue!50] (O\i) at (2.5,2.5 - 1*\i) {$M_\i$};
    }

    \foreach \i in {1,2,3,4} {
        \node[concept, fill=red!50] (C\i) at (4.5,2.5-1*\i) {$C_\i$};
    }

    \node[c_text]
        (concept_1) at (5.2, 1.5) {\scriptsize Color object};
    \node[c_text]
        (concept_2) at (5.2, 0.5) {\scriptsize Rot.\ object};   
    \node[c_text]
        (concept_3) at (5.2,-0.5) {\scriptsize Rot.\ spotlight};
    \node[c_text]
        (concept_4) at (5.2,-1.5) {\scriptsize Object};

    \node[draw=black, circle, very thick, fill=orange] 
        (Y) at (8, 0) {\large $Y$};

    \foreach \i in {1,2,3} {
        \foreach \j in {1,2,3,4,5,6} {
            \draw[->, opacity=0.4] (I\i) -- (H\j);
        }
    }

    \foreach \i in {1,2,3,4,5,6} {
        \foreach \j in {1,2,3,4,5,6} {
            \draw[->, opacity=0.4] (H\i) -- (H_2\j);
        }
    }

    \foreach \i in {1,2,3,4,5,6} {
        \foreach \j in {1,2,3,4} {
            \draw[->, opacity=0.4] (H_2\i) -- (O\j);
        }
    }

    \foreach \i in {1, 2, 3, 4} {
        \draw[->, thick] (concept_\i) -- (Y);
    }  

    \draw[->, dashed, color=red, very thick] (O1) -- (C3);
    \draw[->, dashed, color=red, very thick] (O2) -- (C4);
    \draw[->, dashed, color=red, very thick] (O3) -- (C2);
    \draw[->, dashed, color=red, very thick] (O4) -- (C1);

    \node (alpha) at ($(O4)!0.5!(C4)$) {\large $\alpha$};
    \draw [shorten >= 0.2cm] (alpha) -- ($(O2)!0.55!(C4)$);


    \node [noise, fill=white, right=5 of concept_1] (G_1) {$G$};       
    \node [rep, fill=purple!70, below=0.5 of G_1] (X) {$X$};

    \node [rep, below=0.5 of X, fill=orange] (C_hat) {$Y$};
    \node [rep, left=0.5 of C_hat] (M_1) {$M$};

    \node [rep, right=0.5 of C_hat, fill=red!50] (H_d) {$C$};

    \draw [-stealth, thick] (G_1) -- (X);
    \draw [-stealth, thick] (H_d) -- (C_hat);

    \draw [-stealth, thick] (X) --
        node [midway, above left=0.25em and 0em, black] {$g_\phi$}
        (M_1);
    \draw [-stealth, thick] (X) -- 
        node [midway, above right=0.25em and 0em, black] {$g_{\phi'}$}
        (H_d);
 
    \draw [dashed, very thick, red, shorten <= 0.1cm, shorten >= 0.1cm, -stealth] ($(M_1.south) + (0, 0)$)
        to[out=330, in=210] 
        node [midway, below=0.25em, black] {$\alpha$}
        ($(H_d.south) + (0, 0)$);        
    
    \end{tikzpicture}}
    \caption{\textbf{Left}: An overview of our framework: we learn the
        alignment function $\alpha$ that maps causal representations $M_i$
        learned on cheap unlabelled data by a causal representation learning
        (CRL) encoder, to interpretable concepts $C_j$ using only few concept
        labels. As in standard CBMs, these concepts are used in a downstream
        task like regression or classification of $Y$. \textbf{Right}: Data
        generating process, where $G$ are the latent causal variables, $X$ is
        an observation, $M$ are the representations learned by a model
    $g_\psi$, $C$ are the interpretable concepts and $Y$ is the final label.}
    \label{fig:pipeline-dgp}
\end{figure}

An advantage of CBMs is that they can in many cases provide similar accuracy in terms
of prediction compared to black-box methods, while also ensuring
interpretability by construction \citep{koh2020conceptbottleneck,
ismail2023concept, zarlenga2022concept}, as opposed to post-hoc
methods~\citep{belinkov2022probing}. Users may also interact with the models by
verifying and, if necessary, correcting the detected concepts at test time
\citep{koh2020conceptbottleneck}. However, existing methods suffer from two
important limitations. First, \emph{spurious correlations} in the training data
may lead to learned concepts that do not capture the intended semantics of the
concept \citep{goyal2019explaining,margeloiu2021conceptbottleneckmodelslearn},
or encode more information than intended
\citep{mahinpei2021promisespitfallsblackboxconcept, zarlenga2023towards}, even
with concept-level supervision. Recent works propose mitigation strategies, but
implicitly assume that the concepts are statistically independent
\citep{marconato2022glance,sheth2023auxiliary} or require interventions on the
data generating process with labels \citep{marconato2023humanintepretable}.
Secondly, in addition to labels required for the learning task at hand, CBMs
require \emph{many expensive concept labels} that may not always be available.
Attempts at working around this issue have so far focused on obtaining concept
labels from additional sources like GPT-3 \citep{oikarinen2023labelfree} or
using multi-model information \citep{Yang_2023_CVPR}.

In this paper, we propose a framework that provides theoretical guarantees on
the correctness of the learned concepts and on the number of required labels
without requiring any interventions. We assume that the concepts that we want
to learn correspond to latent variables in a causal system and allow arbitrary
dependences between them. As illustrated in Fig.~\ref{fig:pipeline-dgp} (left),
we propose a two-step pipeline: (i)~leverage out-of-the-box causal
representation learning (CRL) methods to learn provably disentangled
representations of concepts $M$ even in case of spurious correlations;
(ii)~learn an alignment map $\alpha$ between these representations and the
interpretable concepts $C$ with theoretical guarantees and as few concept
labels as possible. In our setting, the map $\alpha$ consists of a permutation
and element-wise transformations. Unlike standard CBMs, our pipeline only uses
concept labels in the second stage to disambiguate the relation between the
representations learned by CRL and the concepts. This requires much fewer
concept labels, addressing the second limitation for CBMs.

We introduce two principled estimators to learn $\alpha$ based on variants of
the group lasso \citep{yuan2006model} and a weighted matching procedure. The
first estimator is based on a linear model with an optional feature map, and
comes with a high probability finite-sample result on its correctness
(Thm.~\ref{thm:group_structure}). This result provides explicit guidance for
tuning the regularization hyperparameter, and shows how the required number of
concept labels scales with problem parameters like the number of concepts and
the dimension of the feature map. Our second estimator is based on a kernelized
procedure, allowing for more flexibility, but comes only with asymptotic
guarantees (Thm.~\ref{thm:kernel-prob}). 

We evaluate our framework in synthetic and image benchmarks with both
real-valued and binary concepts. The experiments show that the learned concepts
have less impurities, measured in terms of OIS and NIS
\citep{zarlenga2023towards}, and are often more accurate than other CBMs. The
experimental results additionally indicate that our estimator performs well
beyond the assumptions of our theorems. For example, our estimators are able to
find the correct permutation, even in settings with strong correlations between
the concepts. 

\section{Related Work}
Extracting high-level concepts from the inner workings of machine learning
models has gained traction over the last years \citep{alain2017understanding,
ghorbani2019towards,mcgrath2022acquisition}. These concepts can be used to
create interpretable explanations. In concept bottleneck models, the concepts
are hard-coded into the structure of the model \citep{koh2020conceptbottleneck,
marconato2022glance, ismail2023concept, zarlenga2022concept}. This approach has
been quite successful, with multiple variants, e.g.\ probabilistic CBMs
\citep{kim2023probabilistic} or energy-based CBMs \citep{xu2024energybased}.
One drawback of CBMs is that the learned concepts may not correctly capture the
intended semantics \citep{margeloiu2021conceptbottleneckmodelslearn} and that
their representations can encode impurities, e.g.\ more information than
intended, \citep{mahinpei2021promisespitfallsblackboxconcept,
zarlenga2023towards}. \citet{havasi2022addressing}~address this problem by
allowing a more expressive autoregressive concept predictor that can take
advantage of the correlations between concepts. On the other hand, the issue is
exacerbated when there are spurious correlations among concepts in the training
data. In this case, \citet{marconato2022glance} encourage disentangled
representations of the concepts by adding a regularization term in the concept
predictor. Similarly, \citet{sheth2023auxiliary} introduce a concept orthogonal
loss that encourages separation between concept representations. Both
approaches implicitly assume the concepts to be statistically independent,
which means that they can fail to learn the correct concepts in settings in
which there are dependences. 

We instead propose to leverage causal representation learning (CRL)
\citep{scholkopf2021crl} as the first step of our pipeline. CRL aims at
recovering potentially dependent latent variables from observations.
Identifying these latent variables is often only possible up to a permutation
and element-wise transformation. Many CRL methods exist with different
assumptions on the available data, e.g.\ the availability of interventional,
counterfactual or temporal data, or parametric assumptions about the underlying
system and observations,  e.g.,\
\citep{hyvarinen2019nonlinear,khemakhem2020vaeica,
KugelgenSGBSBL21,lachapelle2022dms, LippeMLACG22, ahujaMWB2023,
lachapelle2024nonparametric, YaoXLMTMKL24} and many others. Most methods
focus on the infinite sample setting, with the exception of
\citep{acarturk2024sample}, which provides sample complexity results for the CRL task with interventions. Our framework is agnostic to the CRL method used and focuses on the supervised learning task of aligning concepts efficiently to the causal representations. 

Recent work relates identifiable representation learning, including causal
representation learning, with concept-based models, but in an unsupervised
concept discovery setting, as opposed to our supervised setting.
\citet{leemann2023when} propose an identifiable concept discovery method that
uses classic methods like PCA and ICA for independent concepts, but also
introduces novel disentanglement methods for dependent concepts, based on
disjoint mechanisms or independent mechanisms. \citet{rajendran2024} argue that
the identifiability of CRL methods often requires too many assumptions, e.g.\
access to interventional data, and hence relaxes the identifiability of
concepts by defining them as affine subspaces of the latent space of the causal
variables. Given a set of concept conditional datasets, they prove they can
recover these concepts up to linear transformations.

The most closely related work is by \citet{marconato2023humanintepretable}, who
describe theoretical framework that shows how that disentangled representations
can be used in CBMs to avoid concept leakage. Their theoretical framework
suggests that to align learned embeddings and concepts one might require
interventions on the data generating process and labels. Instead, we do not
require interventions. Moreover, we provide a theoretical analysis of the
alignment function, proving guarantees about error bounds on the learned
concepts or the required number of labels. 

\section{Framework and Main Definitions}
\label{sec:framework}
\begin{figure}
\end{figure}
Our setting takes inspiration from causal representation learning (CRL)
\citep{scholkopf2021crl} and the connections between CRL and CBMs described by
\citet{marconato2022glance,marconato2023humanintepretable}. We assume that the
interpretable concepts are related to causal variables in an unknown data
generating process.

As illustrated in Fig.~\ref{fig:pipeline-dgp} (right), we assume a causal
system with latent causal variables $G = (G_1, \ldots, G_d) \in \mathcal{G}
\subseteq \mathbb{R}^{d}$, which can potentially have causal relations between
them. The \textit{observation} is denoted by $X \in \mathcal{X} \subseteq
\mathbb{R}^{D}$ and is generated by an unobserved invertible mixing function 
$f\colon \mathcal{G} \to\mathcal{X}$ as $ X = f(G)$.
The goal of CRL is to recover the causal variables by learning a function
$g_{\psi} \colon \mathcal{X} \to  \mathbb{R}^{d}$ that approximates
$f^{-1}$. This requires specific assumptions and we cannot identify the
ground-truth variables exactly, but only up to an
equivalence class. A common notion of identifiability is up to a
permutation, scaling and translation
\citep{identifiability2023ica,khemakhem2020vaeica,lachapelle2022dms},
which means that the learned representation satisfies $g_{\psi} (X) = P\Lambda G+b$ for a permutation matrix $P$, an invertible
diagonal matrix $\Lambda$ and a vector $b$.
A more general notion is \emph{identifiability up to permutation and
element-wise transformations} in which, instead of a diagonal matrix, we
consider a diffeomorphism, which we formalize as: 
\begin{defn}
    Let $\pi\colon \{1,\ldots, d\} \to \{1, \ldots, d\} $ be a permutation
    of the variable indices. Let $P \in \mathbb{R}^{d \times d}$ be the permutation matrix associated 
    with $\pi$, meaning that $P_{i\pi(i)} = 1$ and $0$ otherwise, and $T :
    \reals^d \to \reals^d$ a map.
    A representation $Z$ identifies the ground-truth causal variables up to a permutation and element-wise transformation if:
    $$
    PT(Z) =
    \left [
        T_{\pi(1)}(Z_{\pi(1)}),
        \ldots,
    T_{\pi(d)}(Z_{\pi(d)})\right]^{\top}
    =
    \left [
        G_1,
        \ldots
        G_d
    \right]^{\top}
    =
    G
    $$
\end{defn}
There are many more notions of identifiability in the literature, e.g., in some
cases the causal variables can be multidimensional \citep{LippeMLACG22} or
identified up to a \emph{block}, i.e.\ a group of causal variables
\citep{ahujaMWB2023, KugelgenSGBSBL21, YaoXLMTMKL24,
lachapelle2024nonparametric}. In the main paper we focus mostly on the
single-dimensional case in which we identify each individual variable, but we
provide extensions to the multidimensional case in
App.~\ref{app:proofs_params}. 

We denote the learned causal variables by $M=(M_1,\ldots, M_d) \in
\mathbb{R}^{d}$. For simplicity of exposition, in the main paper we will assume
that the  interpretable concepts $C=(C_1,\ldots, C_d)^{\top} \in
\mathbb{R}^{d}$ correspond to the ground-truth causal variables $G_1, \dots,
G_d$ up to permutation and element-wise transformations; in other words,
$g_{\psi'}$ shown in Fig.~\ref{fig:pipeline-dgp} (right) also identifies
ground-truth causal variables up to the same identifiability class. In
App.~\ref{app:proofs_params} and Sec.~\ref{sec:non_param_learning} we extend
this to allow each of the concepts to be a transformation of a group of causal
variables. Our goal is to learn the alignment map $\alpha$ that transforms the
learned representations $M$ to the concepts $C$ efficiently and accurately.

Note that in our theoretical results we assume that both concepts and causal
variables are continuous-valued. Our motivation for this choice is technical:
most current CRL methods assume that the underlying causal variables are
continuous-valued, and only very few methods identify discrete-valued causal
variables under very specific assumptions, e.g., in cases where different
causal variables always affect different parts of an observation
\citep{kong2024learning}. Since in concept-based models it is common to work
with binary concepts, we also test these cases in our experiments and show that
our estimators still provide good results, but we leave the extension of the
theoretical results to discrete causal variables and concepts for future work.

For our theoretical analysis, we assume that we are given a function $g_{\psi}$
from a CRL method that correctly identifies the ground-truth causal variables
up to some identifiability class and that the interpretable concepts correctly
identify the causal variables up to the same class.
\begin{asump}
    \label{asump:central}
    Let $M=g_{\psi} (X)$ be the causal representations learned by a pretrained
    CRL method and let $C=g_{\psi'} (X)$ be the interpretable concepts
    annotated from the observation $X$. We assume that $M$ and $C$ identify $G$
    up to the same identifiability class. This implies that    
    \begin{align}
        \label{eq:human_machine_relation}
        PT(M) =
        \left [
            T_{\pi(1)}(M_{\pi(1)}),
            \ldots,
        T_{\pi(d)}(M_{\pi(d)})\right]^{\top}
        =
        \left [
            C_1,
            \ldots
            C_d
        \right]^{\top}
        =
        C.
    \end{align}
\end{asump}
While this assumption is needed for our theoretical analysis, we will show that
empirically our framework works even when the CRL methods do not fully identify
the causal variables. Finding the alignment function $\alpha$ in
Fig.~\ref{fig:pipeline-dgp} reduces to learning the permutation $\pi$ and a
separate regression per concept $C_i$ to learn the transformation from learned
causal representation $M_{\pi(i)}$ to interpretable concept $C_i$. The main
difficulty here is identifying $\pi$ from observational data, i.e.\ without
performing interventions, and with few samples. In the following we introduce
two estimators for this setting, one assuming the element-wise transformation
is linear, e.g.\ as is the case in some CRL methods like
\citep{hyvarinen2019nonlinear, khemakhem2020vaeica}, for which we will be able
to provide finite sample results based on a tunable parameters, and a second,
non-parametric method based on kernel methods that allows for arbitrary
invertible element-wise transformations, and can hence be applied to most CRL
methods. After recovering the interpretable concepts, we can use them as in
CBMs as inputs to a label predictor for $Y$.

\section{Linear Regression Alignment Learning with the Group Lasso}
\label{sec:param_learning}

In this section, we describe a linear regression approach based on the Group
Lasso to learn the permutation $\pi$ and transformation $T$
in~\eqref{eq:human_machine_relation}. We will prove that this method
simultaneously provides accurate regression estimates for $T$ and identifies
$\pi$ correctly with high probability. To simplify the exposition, we focus
here on the case of scalar variables. A pseudo-code description can be found
in Alg.~\ref{alg:param}. Proofs are in
App.~\ref{app:proofs_params}, which also contains discussion of the
assumptions and a
generalization to block variables. The proof combines techniques from
high-dimensional statistics \citep{buhlmannVG11,lounici2011oracle}.

\paragraph{Method.}
Linear regression can describe non-linear relations by transforming
covariates using a feature map $\varphi : \reals \to \reals^p$. In this
section, we assume that $T_i$ can be expressed as a linear function
of $\varphi(M_{\pi(i)})$. The choice of $\varphi$ therefore gives
precise control to trade off interpretability with expressive power for
$T_i$. For instance, in the simplest and most easily interpretable case,
$\varphi$ can be the identity function, so that $p=1$ and $C_i$ and
$M_{\pi(i)}$ are related by scaling. In more challenging settings,
richer functional relations may be needed, e.g.\ splines or random Fourier features.
We apply the same feature map to all machine variables in $M$, for
which we write $\varphi(M) = [\varphi(M_1)^{\top}, \ldots,
\varphi(M_d)^{\top}]^{\top}$. Then each $C_i$ is modeled as a linear
function of the transformed variables:
\begin{align}\label{eq:single_model}
    C_i = \varphi(M)\bstar_i 
    + \epsilon_i, 
\end{align}
where $\bstar_i \in \mathbb{R}^{pd}$ is an unknown parameter vector, and
$\epsilon_i \sim \mathcal{N}(0,\sigma^2)$ is Gaussian noise. By assumption,
$C_i$ only depends on $M_{\pi(i)}$ and not on any of the other variables, so
$\bstar_i$ is sparse: only the coefficients for $\varphi(M_{\pi(i)})$ are
non-zero. To express this formally, let $G_j = \{(j-1)p, \ldots, jp\} $ be the
indices that belong to variable $M_j$ and, for any $\beta \in \reals^{pd}$,
define $\beta^{j}= (\beta_k \mid  k \in G_j)$ to be the corresponding
coefficients. Then $(\bstar_i)^j$ is non-zero only for $j = \pi(i)$.

We assume we are given a data set $\mathcal{D}=
\{(C^{(\ell)},M^{(\ell)})\}_{\ell=1}^{n}$ that contains $n$ independent samples
of corresponding pairs $C^{(\ell)} = (C_1^{(\ell)},\ldots,C_d^{(\ell)})$ and
$M^{(\ell)} = (M_1^{(\ell)}, \ldots, M_d^{(\ell)})$. We stack the $C^{(\ell)}$
into a matrix $\Hvec\in\mathbb{R}^{n\times d}$ and the feature vectors
$\varphi(M^{(\ell)})$ into $\Phi \in \mathbb{R}^{n \times pd}$. This leads to
the relation
\begin{align*}
    \Hvec_{i} 
    &= \Phi \bstar_i + \bm{\epsilon}_i, 
\end{align*}
where $\Hvec_i$ is the $i$-th column of $\Hvec$ and the noise vector
$\bm{\epsilon}_i$ consists of $n$ independently drawn $\cN(0, \sigma^2)$
variables.
To estimate $\bstar_i$, we use the Group Lasso with parameter $\lambda > 0$:
\begin{align}\label{eq:group_est}
    \hbeta_i
        = \argmin_{\beta \in \mathbb{R}^{dp}} \tfrac{1}{n}\|\Hvec_i - \Phi\beta \|^2 + 
        \lambda  \sqrt{p} \|\beta\|_{2, 1}, 
        & & 
    \|\beta\|_{2, 1} 
        =  \sum_{j=1}^{d} \|\beta^{j}\|
.\end{align}
The $(2, 1)$-mix norm $\|\beta\|_{2, 1}$ in \eqref{eq:group_est} encourages
group-wise sparsity. It applies the Euclidean norm $\|\beta^{j}\|$ to each
group~$j$ separately, and sums the results over groups, as defined below. We
also define the $(2, \infty)$-mix norm as $\|\beta\|_{2, \infty} = \max_{j=1,
\ldots, d}\|\beta^j\|$

\paragraph{Theoretical Analysis.}
We denote the full covariance matrix by $\hS =
\tfrac{1}{n}\Phi^{\top}\Phi$. For the group of $p$ columns of $\Phi$
that correspond to $\varphi(M_j)$ we write $\Phi_j = \Phi_{G_j}$. Also
let $\hS_{jj'} = \tfrac{1}{n}\Phi_j^{\top} \Phi_j'$ denote the
covariance matrix between groups $j$ and $j'$, and abbreviate $\hS_j =
\hS_{jj}$. Then, w.l.o.g., we can assume that the data within each group
have been centered and decorrelated: 
\begin{align}\label{eqn:standardized}
     \tfrac{1}{n}\mathbbm{1}^{\top}\Phi_j  &= 0
    \quad \text{and} \quad
    &
    \hS_j = I 
    & & 
    \text{for all } j=1,\ldots, d
.\end{align}
This can be achieved by pre-processing: subtract the empirical mean of
$\Phi_j$ and multiply it from the right by the inverse square root of
the empirical covariance matrix. Preprocessing is allowed in our
theoretical results, because they apply to the fixed design setting,
so probabilities refer to the randomness in $\Hvec$
conditional on already having observed $\Phi$.
If $\varphi(M_j)$ and $\varphi(M_{j'})$ are completely correlated, then
$\bstar_i$ is not uniquely identifiable, no matter how much data we
have. To rule out this possibility, we make the following
assumption,
which limits the amount of correlation. This is a standard assumption 
when analyzing the Group Lasso \citep{lounici2011oracle, LouniciPTG09}.
\begin{asump}\label{asump:structure}
    There exists $a > 1$ s.t.
    for all $j\neq j'$,
    \begin{align*}
        \max_{t\in \{1,\ldots,p\}}|(\hS_{jj'})_{t t}| 
            \le \frac{1}{14a} ,
        \max_{t, t' \in \{1,\ldots,p\}}|(\hS_{jj'})_{t t'}| 
            \le \frac{1}{14a p } 
    .\end{align*}
\end{asump}
\begin{restatable}{theorem}{GroupStruct}\label{thm:group_structure}
    Suppose the data have been pre-processed to satisfy
    \eqref{eqn:standardized} and 
    let Assump.~\eqref{asump:structure} hold. Take
    any $\delta \in (0, 1)$ and set $\lambda \ge 4\lambda_0$, where
    \begin{align*}
        \lambda_0 = \frac{2\sigma}{\sqrt{n} }
        \sqrt{1 + \sqrt{\frac{8\log(d/\delta))}{p}} +
        \frac{8\log(d/\delta) }{p}},
    \end{align*}
    and set $c = \left( 1 + \tfrac{24}{7(a - 1)} \right) $. 
    Then, any solution 
    $\hbeta_i$ of the Group Lasso objective \eqref{eq:group_est}
    satisfies
    \begin{align}
        \|\hbeta_i - \bstar_i\|_{2, \infty} \le c\lambda \sqrt{p} 
        \label{eq:param_max_bound_main}
    \end{align}
    with probability at least $1 - \frac{\delta}{d}$.
    If, in addition, $\|(\bstar_i)^{\pi(i)}\| > 2c\lambda\sqrt{p} $, then
    \eqref{eq:param_max_bound_main}
    implies that $\hJ_i = \argmax_{j=1, \ldots, d} \|\hbeta_i^{j}\|$
    estimates $\pi(i)$ correctly.
\end{restatable}
Theorem~\ref{thm:group_structure} gives us an explicit relation between
the parameters $n, p,d,\delta$ of the learning task, the tuning of the
hyperparameter $\lambda$, and the estimation errors for $\bstar_i$ and
$\pi(i)$. For example, if we set $\delta = \tfrac{1}{n}$, $\lambda =
4 \lambda_0$ and let $n\to
\infty$, then $\lambda \to 0$ and $\hJ_i$ estimates the correct
index $\pi(i)$ with probability tending to $1$. So, regardless of the true
parameter magnitude $\|(\bstar_i)^{\pi(i)}\|$, the estimator is
consistent given a sufficient amount of data. Another way to express
this is to ask about \emph{sample complexity}: which sample size $n$ do
we need to reach accuracy $E > 0$? Setting $\lambda = 4 \lambda_0$ and
solving for $n$ large enough that $c\lambda \sqrt{p} \leq E$, we see
that
\[
  n \geq \frac{64 c^2\sigma^2
        \big(p + \sqrt{8p\log(d/\delta))} +
      8\log(d/\delta)\big)}{E^2}
\]
is sufficient. For estimating the permutation $\pi(i)$ correctly, 
the required accuracy is $E \leq \|(\bstar_i)^{\pi(i)}\|/2$, so the
larger the true parameters, the easier this task becomes.

\begin{algorithm}[t]
    \caption{Estimating the permutation 
        using linear regression with Group Lasso regularization}
    \label{alg:param}
    \begin{algorithmic}[1]
        \State Input: regularization parameter $\lambda > 0$
        \State Data: $\{(C^{(\ell)}, M^{(\ell)})\}_{\ell=1}^{n}$
        \For{$i=1, \ldots, d$}
            \State 
            $\displaystyle\hbeta_i \gets \argmin_{\beta \in \mathbb{R}^{dp}}
            \|\Hvec_{i} - \Phi \beta\|^2 + \lambda \sqrt{p} \|\beta\|_{2,1}$
        \EndFor
        \State $\displaystyle\hpi \gets \argmax_{\pi \in \Pi} 
            \sum_{i=1}^{d}\|\hbeta^{\pi(i)}_{i}\|$
    \end{algorithmic} 
\end{algorithm}

The estimation is performed separately for each concept $C_i$, and, if
$\hJ_i$ is correct for all $i$ simultaneously, we can construct a
valid estimate of the permutation by $\tpi(i) = \hJ_i$. However, this estimate is
not robust to estimation errors and may even produce functions
$\tpi$ that are not permutations if some $\hJ_i$ are incorrect. The
actual estimator of the permutation, $\hpi$, therefore optimizes a weighted
matching problem, which leads to the same estimate as $\tpi$ if the $\hJ_i$
together produce a valid permutation, but forces $\hpi$ to be a valid
permutation even if they do not:
\begin{align}\label{eq:perm-match}
    \hpi = \argmax_{\pi \in \Pi} \sum_{i=1}^d 
    \|\hbeta_i^{\pi(i)}\|
.\end{align}
Here, $\Pi$ is the set of all permutations. This assignment can be
solved without cycling through all permutations, with cubic runtime
in the dimension $d$. By a union bound over $i$, it follows from
Theorem~\ref{thm:group_structure} that $\hpi$ estimates the true
permutation $\pi$ correctly with high probability:
\begin{restatable}{cor}{TotalProb}\label{thm:total_prob}
    Assume the same setting as Theorem~\ref{thm:group_structure}
    such that for each $i=1,\ldots, d, \|(\bstar_i)^{\pi(i)}\| >2c\lambda \sqrt{p}$ 
    and consider the estimator $\hpi$ as defined in \eqref{eq:perm-match}.
    Then $\hpi = \pi$ with probability 
    at least $1 - \delta$.
\end{restatable}

\section{Kernelized Alignment Learning}
\label{sec:non_param_learning}
The previous section describes how to learn functions with
finite-dimensional representations. We now extend the estimator to use
general functions from a \emph{reproducing kernel Hilbert space} (RKHS)
\citep{HofmannScholkopfSmola2008}. This may be interpreted as a
(typically infinite-dimensional) feature map $\varphi$ that maps to the
RKHS. However, using a representer theorem, all computations can
be performed on finite-dimensional representations. We summarize the
method in Alg.~\ref{alg:non_param}. 
\paragraph{Method.} Define again $M = (M_1, \ldots, M_d)$, where we now allow each machine
variable $M_j$ to take values in an abstract space $\mathcal{Z}_j$. Let
$\mathcal{Z} = \mathcal{Z}_1 \times \ldots \mathcal{Z}_d$. We then
generalize \eqref{eq:single_model} to
\[
    C_i = \bstar_i(M) + \epsilon_i, 
\]
where $\bstar_i \in \mathcal{H}$ is a function from $\mathcal{Z}$ to $\reals$,
and $\epsilon_i \sim \mathcal{N}(0,\sigma^2)$. The space of possible functions
$\mathcal{H}$ will be an RKHS containing functions of the form $\beta(M) =
\sum_{j=1}^d \beta^j(M_j)$, where each $\beta^j$ is an element of an RKHS
$\mathcal{H}_j$ that captures the effect of variable $M_j$ on $C_i$. The
assumption that $C_i$ depends only on $M_{\pi(i)}$ means that $\bstar_i(M) =
(\bstar_i)^{\pi(i)}(M_{\pi(i)})$. Each $\mathcal{H}_j$ can be freely chosen,
and is typically specified indirectly by the choice of a positive definite
\emph{kernel} $\kappa_j : \mathcal{Z}_j \times \mathcal{Z}_j \to \reals$
\citep{HofmannScholkopfSmola2008}. This kernel defines a measure of similarity
between inputs: $\kappa_j(M_j,M_j') = \left \langle \varphi_j(M_j),
\varphi_j(M_j') \right\rangle_{\mathcal{H}_j}$, where $\varphi_j :
\mathcal{Z}_j \to \mathcal{H}_j$ is the corresponding feature map. See
p.\,\pageref{p:kernel_examples} for examples.

Given data $\mathcal{D}= \{(C^{(\ell)},M^{(\ell)})\}_{\ell=1}^{n}$, let
$\Mvec \in \mathcal{Z}^n$ denote the matrix with the machine variables
$(M^{(\ell)})^\top$ stacked as rows. If we further define $\beta(\Mvec)
= [\beta(M^{(1)}), \ldots, \beta(M^{(\ell)})]^{ \top}$, then the Group
Lasso objective \eqref{eq:group_est} generalizes to
\begin{equation}\label{eq:feature_optim}
  \hbeta_i = \argmin_{\beta \in \mathcal{H}} 
        \tfrac{1}{n}
        \|\Hvec_i - \beta( \Mvec ) \|^2
        + 
        \lambda \sum_{j=1}^{d} \|\beta^{j}\|_{\mathcal{H}_j}
,\end{equation}
where $\|\beta^{j}\|_{\mathcal{H}_j}$ is the norm associated with
$\mathcal{H}_j$. 
To optimize the objective in \eqref{eq:feature_optim}, we 
need a finite dimensional objective to give to a Group Lasso solver. We
provide a version of the Representer Theorem showing that
the solution of $\eqref{eq:feature_optim}$ lives in a subspace of
$\mathcal{H}$ that can be described by finite-dimensional parameters
$\hat c_i^{1}, \ldots, \hat c_i^{j} \in \mathbb{R}^{n}$:
\begin{restatable}{theorem}{FeatKernEquiv}
    \label{thm:feature_kernel_equivalence}   Let $\varphi_1, \ldots,
    \varphi_d$ be the feature maps associated with $\mathcal{H}_1,
    \ldots, \mathcal{H}_d$. Then there exist $\hat c_i^{1}, \ldots, \hat
    c_i^{d} \in \mathbb{R}^{n}$ such that the optimization problem in
    \eqref{eq:feature_optim} has solution $\hbeta_i$ with each
    $\hbeta_i^j$ of the form
    \begin{align*}
        \hbeta_i^{j} = \sum_{\ell=1}^{n} 
            \varphi_j(M_j^{(\ell)})(\hat{c}_{i}^{j})_\ell
    .\end{align*}
\end{restatable}
Substitution of this form into \eqref{eq:feature_optim} gives that $\hat
c_i^{1}, \ldots, \hat c_i^{d}$ will be the minimizers of the following
finite-dimensional optimization problem:
\begin{align}\label{eq:kernel_optim_unsquared}
      \min_{c^{1}, \ldots, c^{d} \in \mathbb{R}^{n}}
    \tfrac{1}{n}\|\Hvec_i- \sum_{j=1}^{d} K_jc^{j}\|^2
    + 
    \lambda\sum_{j=1}^{d} \|c^{j}\|_{K_j}
,\end{align}
where $K_j \in \mathbb{R}^{n \times n}$ with $(K_j)_{\ell k} 
= \kappa_j(M_j^{(\ell)}, M_j^{(k)})$, and
$\|c^{j}\|_{K_j} = \sqrt{{c^j}^{\top}K_jc^{j}} $.
\\
This procedure is performed for  $i=1,\ldots d$.
The permutation is estimated as in the linear case:
\begin{align*}
    \hpi
    =\argmax_{\pi \in \Pi}\sum_{i=1}^{d} 
    \|\hat{c}^{\pi(i)}_i\|_{K_{\pi(i)}}
.\end{align*}

\paragraph{Theoretical Analysis.}
Using a result by \citet{bach2008consistency}, we  prove that our estimator
for $\pi$ is consistent under suitable conditions,  discussed
in App.~\ref{app:proofs_non_params}. This holds for random
design, so for the joint randomness of $\mathcal{D}$.
\begin{restatable}{theorem}{KernelProb}\label{thm:kernel-prob}
    Assume ($A$--$D$) in App.~\ref{app:kernel_consistency}. Then, for any 
    sequence of regularization parameters $\lambda_n$ such that $\lambda_n \to 0$
    and $\sqrt{n} \lambda_n \to +\infty$ when $n \to \infty$,
    the estimated permutation $\hpi$ converges in probability to $\pi$. 
\end{restatable}
\begin{algorithm}[t]
    \caption{Estimating the permutation 
        using kernels}
    \label{alg:non_param}
    \begin{algorithmic}[1]
        \State Input: reg.\ parameter $\lambda > 0$, kernels
        $\kappa_1,\ldots,\kappa_d$
        \State Data: $\{(C^{(\ell)}, M^{(\ell)})\}_{\ell=1}^{n}$
        \For{$j=1,\ldots,d$}
            \State $(K_j)_{\ell k} \gets \kappa_j(M^{(\ell)}, M^{(k)})$
            \State $L_j \gets \text{CholeskyDecomposition}(K_j)$
        \EndFor
        \For{$i=1, \ldots d$}
            \State 
            $\displaystyle\widehat{\gamma}_i \gets 
            \argmin_{
                \gamma \in \reals^{np}
            }
            \tfrac{1}{n}  
            \|\Hvec_i -  \sum_{j=1}^{d} L_j \gamma^{j}\|^2
            + 
            \lambda \|\gamma\|_{2, 1}$
        \EndFor
        \State $\displaystyle\hpi \gets \argmax_{\pi \in \Pi} 
        \sum_{i=1}^{d}
            \|\widehat{\gamma}^{\pi(i)}_{i}\|$
    \end{algorithmic}
\end{algorithm}

\paragraph{Implementation.}
To use a standard Group Lasso solver we need to reparametrize the
optimization problem in \eqref{eq:kernel_optim_unsquared}, because of
the scaled norms $\|\cdot\|_{K_j}$. We can do this with a Cholesky
decomposition:
\begin{restatable}{lemma}{KernelImp}\label{lem:kernel-imp}
    For each $j=1,\ldots, d$ let $L_j$ be the Cholesky 
    decomposition of the Gramm matrix $K_j$, then $\hat c_i^{j}= (L_j^{\top})^{-1}\hat \gamma_i^{j}$
    if the parameters $\hat \gamma_i^1,
    \ldots, \hat \gamma_i^d$ are minimizers of
    \begin{align}\label{eq:kernel_optim_lstsq}
          \min_{\gamma^{1}, \ldots, \gamma^{d}\in \mathbb{R}^{n}}
            \tfrac{1}{n}  
            \|\Hvec_i - \sum_{j=1}^{d} L_j\gamma^{j}\|^2 
            + 
            \lambda \|\gamma^{j}\|_{2, 1}
    .\end{align}
\end{restatable}
\begin{figure}[t]
    \def\trimLeft{1cm}
    \def\trimRight{0.5cm}
    \def\scaleValue{3cm}

    \def\incFigure#1#2#3{%
        \node{\scalebox{1}{\includegraphics[
            trim={{#1} 0 {#2} 0.2cm},
            clip
    ]{./figs/spline/#3.pdf}}};
        }
    \centering
    
    \resizebox{\textwidth}{!}{
    \begin{tikzpicture}
        \node (legend) {\includegraphics[scale=1.6]{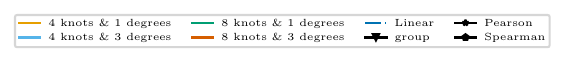}};
 
        \matrix[below=-0.5cm of legend.south,
            row sep=-0.3cm, 
            column sep=-0.1cm, 
            ampersand replacement=\&] (plotmatrix) {
            \incFigure{0.25cm}{0cm}{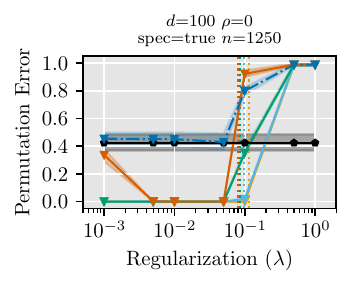}\&[-0.37cm]
            \incFigure{1.38cm}{0cm}{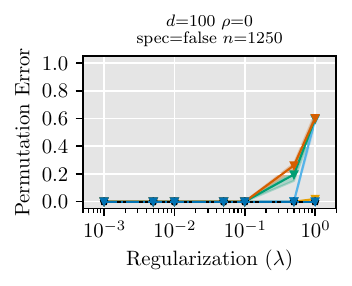}\&

            \incFigure{0.65cm}{0cm}{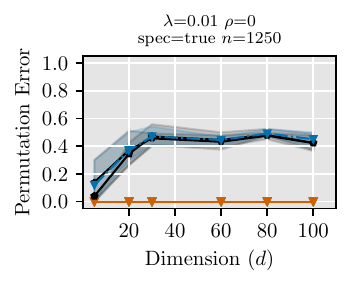}\&[-0.38cm]
            \incFigure{1.38cm}{0cm}{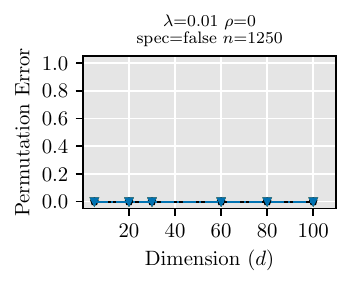}\&\\

            \incFigure{0.25cm}{0cm}{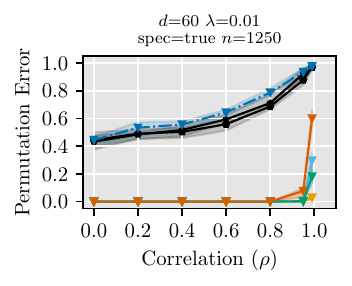}\&
            \incFigure{1.38cm}{0cm}{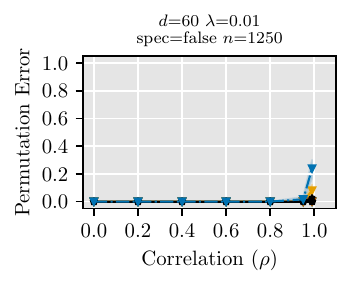}\&

            \incFigure{0.65cm}{0cm}{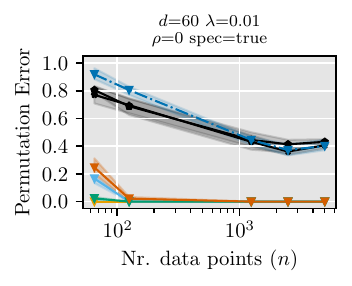}\&
            \incFigure{1.38cm}{0cm}{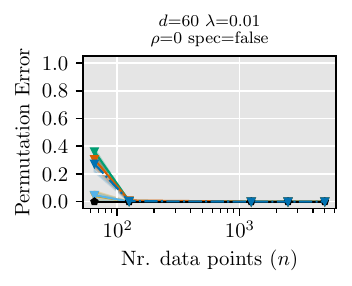}\&\\
        };

    \end{tikzpicture}}
    \vspace{-1.2em}
    \caption{Permutation error rate for spline features.
        From top left to bottom right we vary: (i) the
        regularization parameter, (ii) the number of dimensions, (iii) the correlation of
        the variables and (iv) the number of labels. The first plot of each
        pair shows the wellspecified and the second the
        misspecified case. We average over $10$ seeds and shade the
    25-75th percentile.}
    \label{fig:synth_data}
\end{figure}

\section{Experiments}
\label{sec:experiments}
We perform two types of experiments on four different datasets to evaluate our
estimators. In the first type of experiment, we evaluate how well the estimator
can learn the permutations and recover the original concepts, which can be
either continuous or binary. In the second type of experiment we assess the
usefulness of our estimator to perform a downstream classification task, where
the target label is binary. For this task, concepts are binarized by setting
the concept to $0$, if the value is lower than the midpoint of the range of
that concept, and $1$ otherwise. The label is randomly generated by selecting a
sub-selection of concepts and checking if some of those columns are $1$. This
creates a random classification task for each seed. While the theoretical
guarantees for our estimator do not translate directly to a setting where the
concepts are binary, we can use the logistic Group Lasso \cite{meier2008group}
and we see empirically that this variant also performs well.

The first dataset, called ``Toy dataset'', is  synthetic. Here the concepts are
generated using either a linear combination of features (which we call the
wellspecified case) or diffeomorphisms (which we call the misspecified case) of
the representations. The concepts are then permuted. We evaluate on datasets
common in CRL: Action and Temporal Sparsity \citep{lachapelle2022dms} and
Temporal Causal3Dident \citep{LippeMLACG22}. We train several CRL methods as a
first step of our pipeline:  DMS-VAE \citep{lachapelle2022dms}, CITRIS-VAE
\citep{LippeMLACG22},  iVAE \citep{khemakhem2020vaeica} and TCVAE
\citep{chen2018isolating}. In the downstream task experiments, we compare to
CBM \citep{koh2020conceptbottleneck}, CEM \citep{zarlenga2022concept} and
HardCBM \citep{havasi2022addressing}. Details are in
App.~\ref{app:experiments}.

\paragraph{Performance Metrics.}
To assess our estimator in terms of concept reconstruction capabilities we
report the mean error in the learned permutation of the variables: $\text{MPE}
= \frac{1}{d}\sum_{i=1}^{d} \ind\{\tpi(i) \neq \pi(i)\}$. In the experiments
where we test its capabilities for downstream tasks such as classification, we
report the accuracy on the final label and the Oracle Impurity Score (OIS)
\citep{zarlenga2023towards}. The OIS  measures how much information of each
concept is contained in other concepts. For concept interpretability, it is
desirable that the OIS should be as small as possible.
We report an extended discussion of our results in Appendix~\ref{app:addition_results}, 
with additional metrics, e.g. the NIS \citep{zarlenga2022concept}, and computation times.

\paragraph{Toy Dataset.}
The synthetic data experiments consist of $4$ settings, each using a different
set of features to perform the regression: linear, splines, 
random Fourier features (RFF) and kernels. The $M$ variables are distributed according
to $\cN(0, (1 - \rho)I_{d\times d} + \rho \ind)$, where $\ind$ denotes a 
matrix filled with only $1$'s. The $\rho \in (0, 1)$ parameter controls the
amount of correlation between the marginal variables. We sample $n$ data points, 
on which we perform a $80/20$ train/test data split. In the wellspecified
case, we generate the $M$ variables using the features from that setting. 
For each dimension $j=1,\ldots, d$, we draw a random weight
vector $\bstar_j \in \mathbb{R}^{p}$, such that 
$\|\bstar_j\| \in [16\lambda_0, 32\lambda_0]$ uniformly. A permutation 
$\pi\colon \{1, \ldots, d\} \to \{1, \ldots, d\}$ is sampled uniformly 
from all possible permutations. Finally, with independent
$\epsilon_i \sim N(0, \sigma^2)$ noise variables, we get
$C_{i} = \varphi(M_{\pi(i)})^{\top}\bstar_{\pi(i)} +  \epsilon_i.$
In the misspecified case, the $M$ variables are still sampled the same as before, 
but the $C$ variables are generated by sampling a
diffeomorphism for each dimension. These outcomes are then permuted using a
random permutation again. 
We cover a large range of values of
$\rho$,$d$, $n$ and  $\lambda$ as shown in
App.~\ref{app:synth_experiments}. 

\paragraph{Action/Temporal Sparsity Datasets.} We use the two synthetic
datasets from \citep{lachapelle2022dms} that represent the action and temporal
sparsity settings in a time series. These settings have $10$ causal
variables $z_1, \ldots, z_{10} \in [-2, 2]$ with a causal structure. The mixing function is 
an invertible neural network with Gaussian random 
weights, after which the columns in the linear layers are orthogonalized 
to ensure injectivity. 
To recover the ground-truth permutation, we
follow \citet{lachapelle2022dms} and use the test set to
calculate a permutation based on  Pearson correlations.

\paragraph{Temporal Causal3DIdent.} We evaluate our methods on an image
benchmark, TemporalCausal3DIdent \citep{LippeMLACG22}. The dataset consists of
images of 3D objects, rendered in different positions, rotation and lighting.
The causal variables are  the position $(x, y, z) \in [-2, 2]^2$, the object
rotation with two dimensions $[\alpha, \beta] \in [0, 2\pi)^2$, the hue, the
background and spotlight in $[0,2\pi)$. The object shape is a categorical
variable. We use a pretrained CITRIS-VAE encoder, which outputs a $32$
dimensional latent space and a grouping of which dimensions relate to which
causal variables. Although CITRIS-VAE provides the correct permutation of the
groups, we ignore it and perform a random permutation on the variables. We
train an MLP for each of the $32$ dimensions that predicts all causal
variables. Based on the $R^2$ scores of these regressions, we learn the group
assignments. 

\begin{table}[t]
    \caption{Label Accuracy and the OIS-metric on our downstream task. 
        The $(n)$ indicates the number of train and test points used in each
        column. We averaged the results over $10$ seeds and report the mean and
        standard deviation. The best result for each $n$ is written in \textbf{bold}.}
    \label{tbl:main-table}
    \centerfloat 
    \setlength\tabcolsep{2pt}
    \notsotiny
    \begin{tabular}{m{2.6cm}m{0.8cm}cccc@{\hskip 0.5em}@{\hskip 0.5em}cccc}
        \toprule
        & &  
        \multicolumn{4}{c}{{\small Label Acc. $\uparrow$ $(n)$}} & 
        \multicolumn{4}{c}{{\small OIS $\downarrow$ $(n)$}}\\
        \hline
        Model & Method   & 20 & 100 & 1000 & 10000 & 20 & 100 & 1000 & 10000\\
\toprule
\bottomrule
\multicolumn{10}{c}{\rule{0pt}{0.3cm}\cellcolor{gray!30}\textbf{Action Sparsity Dataset}}\\
\multirow{3}*{DMS-VAE} & {Linear} & \textbf{0.77} {\scriptsize$\pm$ \textbf{0.03}} & 0.81 {\scriptsize$\pm$ 0.01} & 0.83 {\scriptsize$\pm$ 0.01} & 0.84 {\scriptsize$\pm$ 0.01} & \textbf{0.63} {\scriptsize$\pm$ \textbf{0.01}} & 0.40 {\scriptsize$\pm$ 0.00} & 0.15 {\scriptsize$\pm$ 0.00} & 0.12 {\scriptsize$\pm$ 0.00}\\

 & {Spline} & 0.72 {\scriptsize$\pm$ 0.03} & \textbf{0.83} {\scriptsize$\pm$ \textbf{0.02}} & \textbf{0.85} {\scriptsize$\pm$ \textbf{0.01}} & 0.86 {\scriptsize$\pm$ 0.01} & 0.63 {\scriptsize$\pm$ 0.01} & 0.40 {\scriptsize$\pm$ 0.00} & 0.15 {\scriptsize$\pm$ 0.00} & 0.12 {\scriptsize$\pm$ 0.00}\\

 & {RFF} & 0.77 {\scriptsize$\pm$ 0.02} & 0.82 {\scriptsize$\pm$ 0.02} & 0.84 {\scriptsize$\pm$ 0.01} & 0.85 {\scriptsize$\pm$ 0.01} & 0.64 {\scriptsize$\pm$ 0.01} & \textbf{0.39} {\scriptsize$\pm$ \textbf{0.00}} & 0.14 {\scriptsize$\pm$ 0.00} & 0.12 {\scriptsize$\pm$ 0.00}\\

\hline
\multirow{3}*{iVAE} & {Linear} & 0.73 {\scriptsize$\pm$ 0.03} & 0.79 {\scriptsize$\pm$ 0.02} & 0.81 {\scriptsize$\pm$ 0.01} & 0.83 {\scriptsize$\pm$ 0.01} & 0.63 {\scriptsize$\pm$ 0.01} & 0.40 {\scriptsize$\pm$ 0.00} & 0.29 {\scriptsize$\pm$ 0.00} & 0.26 {\scriptsize$\pm$ 0.00}\\

 & {Spline} & 0.69 {\scriptsize$\pm$ 0.02} & 0.76 {\scriptsize$\pm$ 0.02} & 0.81 {\scriptsize$\pm$ 0.01} & 0.83 {\scriptsize$\pm$ 0.01} & 0.63 {\scriptsize$\pm$ 0.01} & 0.40 {\scriptsize$\pm$ 0.00} & 0.28 {\scriptsize$\pm$ 0.00} & 0.26 {\scriptsize$\pm$ 0.00}\\

 & {RFF} & 0.59 {\scriptsize$\pm$ 0.04} & 0.70 {\scriptsize$\pm$ 0.02} & 0.78 {\scriptsize$\pm$ 0.01} & 0.81 {\scriptsize$\pm$ 0.01} & 0.64 {\scriptsize$\pm$ 0.01} & 0.39 {\scriptsize$\pm$ 0.00} & 0.28 {\scriptsize$\pm$ 0.00} & 0.26 {\scriptsize$\pm$ 0.00}\\

\hline
 \makebox[0pt][l]{CBM \citep{koh2020conceptbottleneck}} &   & 0.60 {\scriptsize$\pm$ 0.03} & 0.60 {\scriptsize$\pm$ 0.02} & 0.73 {\scriptsize$\pm$ 0.01} & 0.88 {\scriptsize$\pm$ 0.01} & 0.92 {\scriptsize$\pm$ 0.01} & 0.43 {\scriptsize$\pm$ 0.01} & \textbf{0.11} {\scriptsize$\pm$ \textbf{0.00}} & 0.07 {\scriptsize$\pm$ 0.00}\\

 \makebox[0pt][l]{CEM \citep{zarlenga2022concept}} &   & 0.66 {\scriptsize$\pm$ 0.02} & 0.69 {\scriptsize$\pm$ 0.03} & 0.82 {\scriptsize$\pm$ 0.01} & 0.88 {\scriptsize$\pm$ 0.01} & 0.90 {\scriptsize$\pm$ 0.03} & 0.47 {\scriptsize$\pm$ 0.01} & 0.40 {\scriptsize$\pm$ 0.01} & 0.57 {\scriptsize$\pm$ 0.01}\\

 \makebox[0pt][l]{HardCBM \citep{havasi2022addressing}} &   & 0.56 {\scriptsize$\pm$ 0.03} & 0.61 {\scriptsize$\pm$ 0.01} & 0.68 {\scriptsize$\pm$ 0.01} & \textbf{0.89} {\scriptsize$\pm$ \textbf{0.00}} & 0.92 {\scriptsize$\pm$ 0.02} & 0.46 {\scriptsize$\pm$ 0.01} & 0.12 {\scriptsize$\pm$ 0.00} & \textbf{0.06} {\scriptsize$\pm$ \textbf{0.00}}\\

\bottomrule\multicolumn{10}{c}{\rule{0pt}{0.3cm}\cellcolor{gray!30}\textbf{Temporal Causal3DIdent Dataset}}\\
\multirow{3}*{CITRISVAE} & {Linear} & 0.74 {\scriptsize$\pm$ 0.06} & 0.75 {\scriptsize$\pm$ 0.06} & 0.80 {\scriptsize$\pm$ 0.04} & 0.81 {\scriptsize$\pm$ 0.04} & 0.69 {\scriptsize$\pm$ 0.02} & 0.44 {\scriptsize$\pm$ 0.01} & 0.19 {\scriptsize$\pm$ 0.00} & 0.16 {\scriptsize$\pm$ 0.00}\\

 & {Spline} & 0.74 {\scriptsize$\pm$ 0.06} & \textbf{0.80} {\scriptsize$\pm$ \textbf{0.04}} & \textbf{0.82} {\scriptsize$\pm$ \textbf{0.03}} & \textbf{0.83} {\scriptsize$\pm$ \textbf{0.03}} & 0.69 {\scriptsize$\pm$ 0.02} & 0.44 {\scriptsize$\pm$ 0.01} & \textbf{0.13} {\scriptsize$\pm$ \textbf{0.00}} & \textbf{0.09} {\scriptsize$\pm$ \textbf{0.00}}\\

 & {RFF} & 0.70 {\scriptsize$\pm$ 0.06} & 0.76 {\scriptsize$\pm$ 0.05} & 0.79 {\scriptsize$\pm$ 0.04} & 0.81 {\scriptsize$\pm$ 0.04} & \textbf{0.65} {\scriptsize$\pm$ \textbf{0.01}} & 0.43 {\scriptsize$\pm$ 0.00} & 0.16 {\scriptsize$\pm$ 0.01} & 0.13 {\scriptsize$\pm$ 0.01}\\

\hline
\multirow{3}*{iVAE} & {Linear} & \textbf{0.74} {\scriptsize$\pm$ \textbf{0.06}} & 0.76 {\scriptsize$\pm$ 0.05} & 0.78 {\scriptsize$\pm$ 0.05} & 0.80 {\scriptsize$\pm$ 0.04} & 0.69 {\scriptsize$\pm$ 0.02} & 0.44 {\scriptsize$\pm$ 0.01} & 0.17 {\scriptsize$\pm$ 0.00} & 0.14 {\scriptsize$\pm$ 0.00}\\

 & {Spline} & 0.73 {\scriptsize$\pm$ 0.06} & 0.78 {\scriptsize$\pm$ 0.04} & 0.79 {\scriptsize$\pm$ 0.04} & 0.81 {\scriptsize$\pm$ 0.04} & 0.69 {\scriptsize$\pm$ 0.02} & 0.43 {\scriptsize$\pm$ 0.04} & 0.17 {\scriptsize$\pm$ 0.00} & 0.13 {\scriptsize$\pm$ 0.00}\\

 & {RFF} & 0.69 {\scriptsize$\pm$ 0.06} & 0.75 {\scriptsize$\pm$ 0.05} & 0.78 {\scriptsize$\pm$ 0.04} & 0.80 {\scriptsize$\pm$ 0.04} & 0.65 {\scriptsize$\pm$ 0.01} & \textbf{0.42} {\scriptsize$\pm$ \textbf{0.02}} & 0.16 {\scriptsize$\pm$ 0.00} & 0.13 {\scriptsize$\pm$ 0.00}\\

\hline
 \makebox[0pt][l]{CBM \citep{koh2020conceptbottleneck}} &   & 0.57 {\scriptsize$\pm$ 0.02} & 0.53 {\scriptsize$\pm$ 0.03} & 0.68 {\scriptsize$\pm$ 0.04} & 0.78 {\scriptsize$\pm$ 0.05} & 1.00 {\scriptsize$\pm$ 0.03} & 0.48 {\scriptsize$\pm$ 0.02} & 0.25 {\scriptsize$\pm$ 0.00} & 0.18 {\scriptsize$\pm$ 0.00}\\

 \makebox[0pt][l]{CEM \citep{zarlenga2022concept}} &   & 0.64 {\scriptsize$\pm$ 0.04} & 0.67 {\scriptsize$\pm$ 0.03} & 0.72 {\scriptsize$\pm$ 0.05} & 0.78 {\scriptsize$\pm$ 0.05} & 0.97 {\scriptsize$\pm$ 0.05} & 0.51 {\scriptsize$\pm$ 0.02} & 0.36 {\scriptsize$\pm$ 0.01} & 0.49 {\scriptsize$\pm$ 0.01}\\

 \makebox[0pt][l]{HardCBM \citep{havasi2022addressing}} &   & 0.57 {\scriptsize$\pm$ 0.05} & 0.53 {\scriptsize$\pm$ 0.02} & 0.63 {\scriptsize$\pm$ 0.03} & 0.77 {\scriptsize$\pm$ 0.05} & 0.96 {\scriptsize$\pm$ 0.03} & 0.47 {\scriptsize$\pm$ 0.02} & 0.24 {\scriptsize$\pm$ 0.00} & 0.16 {\scriptsize$\pm$ 0.00}\\

\bottomrule
    \end{tabular}
    \vspace{-2em}
\end{table}

\paragraph{Results.}
We show a set of representative results in Fig.~\ref{fig:synth_data} for the
Toy Dataset, where we compare as baselines to the permutations one can learn
with Pearson or Spearman correlation, and a linear sum assignment. Our
estimator is able to reconstruct the correct permutation perfectly with only a
small number of features. It performs well for a broad range of regularization
parameters. In the wellspecified case we calculate the $\lambda_0$ parameter
and see that the estimator performs well around this value. The dimension
dependence is almost negligible, which was predicted by the dimension appearing
only in the log factor in Thm.~\ref{thm:group_structure}. The estimator works
well with few data points and even better than theory predicts, as it has a low
error even with high correlation, as shown in Fig.~\ref{fig:synth_data}.
Spearman correlation also performs well in some simple settings, which suggests
that the relation is indeed monotonic in these settings. On the other hand,
there are no theoretical guarantees for this baseline and it cannot be used for
multi-dimensional representations, for which our estimator still provides a
principled approach. We report results with similar trends for RFF and the
kernelized approach in App.~\ref{app:additional-toy}. To test the robustness of
our estimation to the case in which CRL methods do not provide a complete, but
only partial disentanglement, we also show an ablation where we provide
mixtures of pairs of causal variables as input to our algorithm, showing that
it still recovers the correct permutation for each pair. Similar trends apply
also to the image dataset, as reported in App.~\ref{app:additional-dl}.

A selection of results for the downstream classification task for the action
sparsity dataset and temporal Causal3DIdent dataset are in
Tab.~\ref{tbl:main-table}. Our estimator consistently scores high in terms of
Label Acc. and OIS, while requiring less labels. For image data, our estimator
beats the baseline for every number of datapoints. In
App.~\ref{app:additional-dl} we report the complete results, which show that
the baselines perform well with regards to the mean concept accuracy in the
Action/Temporal sparsity datasets, but our estimator still performs better with
the smallest number of labels and with less computation power to calculate. Our
estimators consistently achieve a perfect NIS \citep{zarlenga2023towards}
score, which measures the impurities distributed across the concept
representations.

Finally, in Fig.~\ref{fig:times-both-main} we show the execution times of our
estimator compared to several baseline methods. We report the times in both the
continuous experiments and binary classification experiments conducted with the
Temporal Causal3DIdent dataset. The baseline in the continuous setting is given
by training a neural network for each causal variable. Each neural network
learns how to predict the causal variable from the learned encodings, and the
$R^2$-scores are used to construct a matching. The time that is reported is the
time that is needed to estimate the matching and the training, if needed. In
the binary classification experiments, the baselines are given by the CBM, CEM,
and HardCBM models. The time that is reported is the time needed to either
estimate the matching and learn the classification or to train one of the
concept-based models. We see that all versions of our estimator require
significantly less computation time than the baselines. 

\section{Conclusions and Discussion}
\label{sec:conclusion}
\begin{figure}[t]
    \def\incFigure#1#2{%
        \node{\includegraphics[
            scale=0.9,
            trim={#2cm 0 0 0},
            clip
        ]{./figs/citris/#1.pdf}};
        }
    \begin{tikzpicture}
        \matrix[row sep=0cm, column sep=0.2cm] (plotmatrix) {
            \incFigure{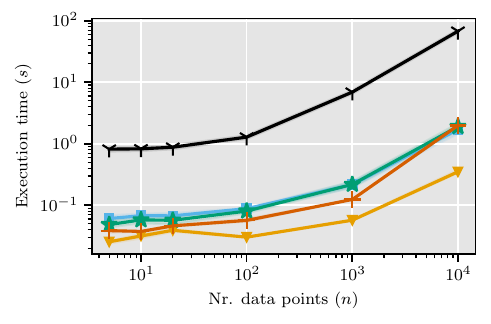}{0}&
            \incFigure{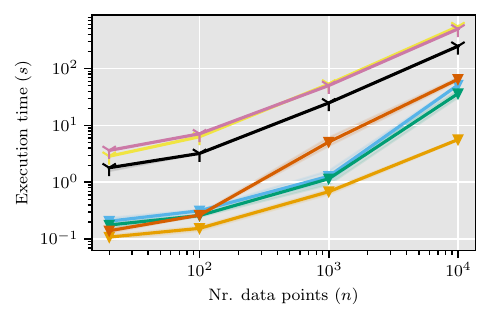}{0.7}&\\
        };

        \node[above right=0cm and 0.6cm of plotmatrix.north, anchor=east]
            {\includegraphics[%
                scale=0.75, 
                trim={0.2cm 0.2cm 0 0},
                clip
            ]{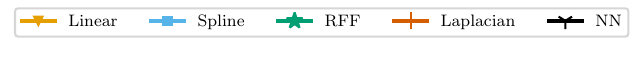}};

        \node[above right=0cm and 0.4cm of plotmatrix.north, anchor=west]
            {\includegraphics[%
                scale=0.75, 
                trim={0.2cm 0.2cm 0 0},
                clip
            ]{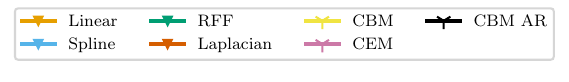}};
    \end{tikzpicture}
    \caption{Execution times for our estimators and several baseline models on
        the Temporal Causal3DIdent dataset. \textbf{Left:} The causal variables
        are continuous-valued. The baseline is given by training several neural
        networks and the matching is based on the $R^2$-scores. We report the
        time needed to estimate the matching and train the neural networks.
        \textbf{Right:} The causal variables are binarized and a downstream
        classification task is added. We report the time needed to estimate the
        matching and learn the classification task or the time needed to train
    the concept-based models.}
    \label{fig:times-both-main}
\end{figure}

We propose a framework that provides theoretical guarantees on learning of
concepts by leveraging causal representation learning (CRL). We provide two
estimators that are able to learn the alignment between the learned
representations and the concepts: a linear estimator with finite sample
guarantees and a non-parametric kernelized estimator with asymptotic
guarantees. We test our methods on CRL benchmarks and show they perform even
better than the theory predicted. 

While our work proposes a promising research direction, several limitations
remain. For example, we assume that human concepts coincide with the causal
variables recovered by the CRL methods, which limits the framework's
applicability. While our framework can also consider settings in which only
blocks of causal variables are identifiable, as opposed to each individual
variable, then the alignment can be only learned between these blocks and the
corresponding blocks of concepts. To generalize our framework to different
levels of coarse-grained human concepts, it would be interesting to incorporate
ideas from  \textit{causal abstraction} \citep{rubenstein2017abstraction,
geiger2021causal, beckers2020approximate} to extend this analysis to cases in
which human concepts are abstractions of the underlying causal system.
Moreover, while in the current framework we focus on CRL methods in order to
get theoretical guarantees on the identifiability of latent variables from
high-dimensional observations, we do not use any causal semantics, so in
principle our framework could be applied also to other types of methods that
provide similar theoretical guarantees.

A theoretical limitation of our analysis is the assumption of low correlation
between the learned representations. In our experiments we find that our
methods still work with significantly stronger correlations than assumed in our
theoretical results, so we believe there may be room here to strengthen the
theory. This might also allow us to relax our assumption that CRL methods
achieve perfect disentanglement of the causal variables, which is not realistic
in many domains. Finally, although we consider discrete concepts in our
evaluation, we do not provide a theoretical analysis for them, because current
CRL methods cannot identify discrete variables, or can only do so in
specific settings \citep{kong2024learning}. A possible way to extend our
theoretical results to this setting would be to consider previous
results for the Group Lasso for logistic regression \citep{meier2008group}.

\clearpage
\acks{We thank SURFsara for the support in using the Snellius Compute Cluster and
    Robert Jan Schlimbach for the technical support in parallelizing most of the
    experiments. Van Erven was supported by the Netherlands Organization for
    Scientific Research (NWO) under grant numbers VI.Vidi.192.095 and
\url{https://doi.org/10.61686/OWREN85146}.}

\IfFileExists{certifiedxai.bib}{
    \bibliography{certifiedxai.bib}
}{
    \bibliography{../certifiedxai.bib}
}


\newpage
\appendix

\section{Proofs for Linear Regression Alignment Learning}
\label{app:proofs_params}
In this section we provide proof details and the algorithm for
Section~\ref{sec:param_learning}. As discussed there, we extend
Theorem~\ref{thm:group_structure} to a more general result
Theorem~\ref{thm:group_structure_general} that allows for blocks of causal
variables corresponding to a single concept. To state the general result we
will redefine our model and introduce additional notation. The variables $C_i
\in \reals$ and $M_j \in \reals^{k_j}$ now live in potentially different
spaces. Let $C = (C_1,\ldots,C_d)^\top \in \reals^d$ and $M =
(M_1,\ldots,M_d)^\top \in \reals^k$, where $k = \sum_{j=1}^d k_j \geq d$. Let
$K_j = \{\sum_{a = 1}^{j-1} k_a + b | b \in \{1,\ldots,k_j\}\}$ denote the
subset of indices in $M$ that correspond to the block $M_j$. The permutation
that we want to recover is~$\pi$.

Each dimension in the $M$ variable can be transformed through a separate
feature map $\varphi_t\colon \mathbb{R} \to \mathbb{R}^{p_t}$ that can be
different for each $t=1,\ldots,k$. We will denote the total feature vector by
\begin{align*}
    \varphi(M) = \begin{bmatrix} 
        \varphi_1(M_1) \\
        \vdots \\ 
        \varphi_k(M_k) 
    \end{bmatrix} 
.\end{align*}

The grouped features will be denoted by 
$\varphi(M)^{j} = (\varphi_t(M_t)  \mid  t \in K_j)$.
Define the average feature set size 
as $\overline{p} = \frac{1}{k}\sum_{t=1}^{k} p_t$. The model is described by
\begin{align*}
    C_{i} 
    &= \varphi(M) \bstar_i + \epsilon_i, 
    \quad
    \bstar_i \in \mathbb{R}^{k \overline{p}}, 
    \epsilon_i \sim \mathcal{N}(0, \sigma^2)
.\end{align*}
For the actual regression task, we can define data matrices again. The matrix
$\Hvec$ will be defined as in the main text and $\Phi$ now becomes an $n \times
k\overline{p}$ matrix, in which all feature vectors $\varphi(M^{(\ell)})$ are
stacked. With $\bm{\epsilon}_i$ denoting $n$ independently draw $\cN(0,
\sigma^2)$ variables, this results in the relation
\begin{align*}
    \Hvec_i = \Phi \bstar_i + \bm{\epsilon_i}    
.\end{align*}
The $\bstar_i$ again has a sparse structure, because only the
parameters corresponding to $\varphi(M)^{j}$ should be non-zero.  Let
the indices of these parameters be denoted by $G_{j}$. Alternatively, 
this $G_j$ is defined through $\varphi(M)_{G_j} = \varphi(M)^{j}$. 
Thus, in this setting
we again have $d$ groups all denoted by $G_j$. 
To ease notation we set again $\beta^{j}_i = ((\beta_i)_t  \mid t \in G_j)$. 
The definitions of the norm $\|\cdot\|_{2, \infty}$ 
and covariance matrices, 
$\hS_{jj'} = \tfrac{1}{n}\varphi(M)_{G_j}^{\top}\varphi(M)_{G_{j'}}
= \tfrac{1}{n}\Phi_j^{\top}\Phi_{j'}$,
are altered in accordance with these groups. As the groups 
can now be of different size, we have to change the  
definition of the $\|\cdot\|_{2, 1}$-norm to take the different group sizes
into account,
\begin{align*}
    \|\beta\|_{2, 1} = \sum_{j=1}^{d} \|\beta^{j}\|\sqrt{p^{j}}  
,\end{align*}
where $p^{j} = \sum_{t \in K_j}p_t$. 
The loss function that we want to optimize to
estimate $\bstar_i$ has the same form as before
\begin{align}\label{eq:general_group_obj}
    \hbeta_{i}
        &= \argmin_{\beta \in \mathbb{R}^{k \overline{p}}}
        \tfrac{1}{n}\|\Hvec_i - \Phi\beta \|^2 + 
        \lambda \|\beta\|_{2, 1}
        =\argmin_{\beta \in \mathbb{R}^{k \overline{p}}}\tfrac{1}{n}\|\Hvec_i - \Phi\beta \|^2 + 
        \lambda \sum_{j=1}^{d} \|\beta^{j}\|\sqrt{p^{j}}  
.\end{align}
The optimality conditions for any solution, $\hbeta_i$, for this convex 
optimization problem are given by
\begin{align}
    \tfrac{1}{n}(\Phi^{\top}(\Hvec_i - \Phi\hbeta_i)^{j} 
    &= \frac{\lambda\sqrt{p^{j}}}{2} \frac{\hbeta^{j}}{\|\hbeta^{j}\|} 
    & \text{ if } \hbeta_i^{j} \neq 0,\label{eq:opt_cond_1}\\
    \tfrac{1}{n} \|\Phi^{\top}(\Hvec_i - \Phi\hbeta_i)\|^2 
    &\le  \frac{\lambda \sqrt{p^{j}}}{2}
    & \text{ if } \hbeta_i^{j} = 0\label{eq:opt_cond_2}
.\end{align}

To ensure that our results hold even in the case where $n < pd$, we
introduce a standard assumption on the data from the high-dimensional
statistics literature. Intuitively, this assumption ensures that the
data is ``variable enough'' in the directions that matter. 
\begin{asump}\label{asump:re_1}
    The Restricted Eigen Value (RE($1$)) 
    is satisfied by the data matrix $\Phi \in \mathbb{R}^{n \times  k \overline{p}}$
    if there exists a $\kappa > 0$
    such that for all $\Delta \in \mathbb{R}^{k \overline{p}} \setminus \{0\} $  
    and $j=1, \ldots, d$
    with $\sum_{i\neq j}\|\Delta^{i}\| \sqrt{p^i}\le 3\|\Delta^j\|\sqrt{p^j}$
    it holds that
    \begin{align*}
        \frac{\|\Phi \Delta\|}{\sqrt{n}\|\Delta^{j}\| }  \mid 
         \ge \kappa
    .\end{align*}
\end{asump}
This property is satisfied for any $\Delta \in \mathbb{R}^{k \overline{p}} \setminus \{0\} $
if $\hS = \tfrac{1}{n} \Phi^{\top} \Phi$ has a positive minimal eigenvalue. Let
$\lambda_{\text{min}} >0$ be the minimal eigenvalue of $\hS$, then
\begin{align*}
    \|\Phi\Delta\|^2 
    = \Delta^{\top} \Phi^{T}\Phi \Delta
    = n\Delta^{\top} \hS \Delta 
    \ge n\lambda_{\text{min}} \Delta^{\top}\Delta
    = n\lambda_{\text{min}}\|\Delta\|^2
.\end{align*}
Now divide by $n$ and take the square root on both sides. 
This gives us
\begin{align*}
    \frac{\|\Phi\Delta\|}{\sqrt{n}} 
    &\ge \sqrt{\lambda_{\text{min}}}\|\Delta\|
    = \sqrt{\lambda_{\text{min}}} \sqrt{
    \sum_{i=1}^{d} \|\Delta^{i}\|^2} 
    \ge \sqrt{\frac{\lambda_{\text{min}}}{d}}  
    \sum_{i=1}^{d} \|\Delta^{i}\|
    \ge \sqrt{\frac{\lambda_{\text{min}}}{d}}  \|\Delta^{j}\|
.\end{align*}
The second inequality follows from 
an application of Jensen's inequality. 
Dividing both sides by $\|\Delta^{j}\|$ gives the desired
result.

The matrix $\hS$ is the empirical covariance matrix and
will be positive definite almost surely whenever $n \ge k\overline{p}$ and hence RE(1) will be
satisfied if  $n \ge k\overline{p}$. 

Finally, define 
$\pmin = \min_{j=1, \ldots d} p^{j}$ and $\pmax = \max_{j=1, \ldots, d}p^{j}$.
The following theorems and proofs are adapted from Chapter~$8$ in 
\citet{buhlmannVG11} and sections $3$ and $5$ in \citet{lounici2011oracle}.
\begin{theorem}\label{thm:group_error_1}
    Assume that for all $\ell=1, \ldots, n$, $\epsilon^{(\ell)}_i \sim \mathcal{N}(0, \sigma^2)$ independently,  $\sigma^2 > 0$, 
    the RE(1) condition is satisfied with $\kappa > 0$
    and consider the Group Lasso estimator
    \begin{align*}
        \hbeta_i = \argmin_{\beta \in \mathbb{R}^{k \overline{p}}} 
            \tfrac{1}{n} \|\Hvec_i - \Phi\beta\|^2 + 
            \lambda \|\beta\|_{2, 1}
    ,\end{align*}
    where $\lambda \ge 4 \lambda_0$
    with
    \begin{align*}
        \lambda_0 = \frac{2\sigma}{\sqrt{n} }
        \sqrt{1 + 
            \sqrt{\frac{8\log(d / \delta) }{\pmin}} 
            +
            \frac{8\log(d / \delta)}{\pmin}
        }.
    \end{align*}
    Then, for any $\delta \in (0,1)$, with probability at least $1 - \frac{\delta}{d}$,
    \begin{align}
        \label{eq:group_error_1}
        \tfrac{1}{n}\|\Hvec_i - \Phi \hbeta_i\|^2
            + 
            \lambda \|\hbeta_i - \bstar_i\|_{2, 1}
        &\le 
        \frac{24 \lambda^2 p^{\pi(i)}}{\kappa^2} \\
        \|(\hS(\hbeta - \bstar))^{j}\|
        &\le \lambda \sqrt{p^{j}} \qquad \text{ for all } j=1,\ldots,d 
            \label{eq:tech_result_2}\\
        \|\hbeta - \bstar\|_{2, 1}
        &\le
        \frac{24\lambda p^{\pi(i)}}{\kappa^2} \label{eq:group_error_3}
    .\end{align}
\end{theorem}
This theorem offers us a several things. Equation~\ref{eq:group_error_1} gives
us a bound on the true prediction error. The last two equations,
(\ref{eq:tech_result_2},~ \ref{eq:group_error_3}), are needed to prove that we
find accurate parameter values using the Group Lasso approach. The fact that
the last equation gives a bound in the $(2, 1)$-norm, allows us to use a
duality argument later on to provide a bound on the $(2, \infty)$-norm of the
difference between the learned and true parameter. Knowing that only one of the
groups has to be non-zero combined with this uniform bound enables us to
conclude that the correct group has been identified in the proof of
Theorem~\ref{thm:group_structure_general}.
\begin{proof}
    First let us define for every $j=1,\ldots,d$ the random events
    $\mathcal{A}_j = \{\tfrac{1}{n}\|(\Phi^{\top}\bm{\epsilon}_i)^{j}\| 
    \le \tfrac{\lambda\sqrt{p^{j}}}{2}\}$
    and their intersection $\mathcal{A} = \bigcap_{j=1}^{d}\mathcal{A}_j$.
    Most importantly, we see from Lemma~\ref{lem:feature_concentration}
    that this event has probability at least $1 - \frac{\delta}{d}$.
    We get the $1 / d$ factor by using  $\widetilde{\delta} = \frac{\delta}{d}$ 
    in Lemma~\ref{lem:feature_concentration} and noticing that 
    this only adds a factor of $2$ in the log terms.
    The first assertion~\eqref{eq:group_error_1} is true on 
    the event $\mathcal{A}$ and follows from 
    the proof of Theorem~8.1 in \citep{buhlmannVG11} and noting
    that in our setting their oracle parameter is given by our $\bstar_i$
    and that $\textup{\textbf{f}}_0 = \Phi \bstar_i$. 
    
    Moving on towards \eqref{eq:tech_result_2}, by the optimality
    condition \eqref{eq:opt_cond_1} and  \eqref{eq:opt_cond_2} we have
    for each $j=1,\ldots, d$
    \begin{align*}
        \tfrac{1}{n} \|(\Phi( \Hvec_i - \Phi\hbeta)^{j}\| 
        \le  
        \frac{\lambda \sqrt{p^{j}}}{2}
    .\end{align*}
    Let us rewrite the expression in \eqref{eq:tech_result_2} into
    \begin{align*}
        \|(\hS(\hbeta_i - \bstar_i))^{j}\| 
        = 
        \tfrac{1}{n}\|(\Phi^{\top}(\Phi\hbeta - \Phi\bstar))^{j}\|
    .\end{align*}
    Substituting $\Phi\bstar_i = \Hvec_i - \bm{\epsilon}_i$ into this expression
    gives 
    \begin{align*}
        \|(\hS(\hbeta_i - \bstar_i))^{j}\|
        &\le 
            \tfrac{1}{n} \|(\Phi^{\top}(\Phi\hbeta_i - \Hvec_i))^{j}\|
            +
            \tfrac{1}{n}\|(\Phi^{\top}\bm{\epsilon}_i)^{j}\|\\
        &\le \frac{\lambda \sqrt{p^{j}}}{2}   + \frac{\lambda \sqrt{p^{j}} }{2}
        = \lambda\sqrt{p^{j}} 
    .\end{align*}
    Note that this inequality only holds on $\mathcal{A}$.
    
    The final assertion is a direct consequence of the first, 
    \begin{align*}
        \lambda \|\hbeta_i - \bstar_i\|_{2, 1} 
        &\le
          \tfrac{1}{n}\|\Phi(\hbeta_i - \bstar_i)\|^2
          + 
          \lambda \|\hbeta - \bstar\|_{2, 1} 
        \le 
        \frac{24 \lambda^2p^{\pi(i)}}{\kappa^2}\\
        \|\hbeta_i - \bstar_i\|_{2, 1} 
        &\le 
        \frac{24 \lambda p^{\pi(i)}}{\kappa^2} 
    .\end{align*}
\end{proof}
To state and prove the general 
version of Theorem~\ref{thm:group_structure} we also need to generalize
Assumption~\ref{asump:structure}. 
\begin{asump}\label{asump:structure_general}
    There exists some constant $a > 1$ such that for any $j\neq j'$, 
    it holds that 
    \begin{align}
        \max_{1 \le t\le \min(p^{j}, p^{j'})}|(\hS_{jj'})_{t t}| 
        \le \frac{1}{14a}\sqrt{\frac{\pmin}{\pmax}}
    \end{align}
    and
    \begin{align}
        \max_{1 \le t\le p^{j}, 1 \le t' \le p^{j'}, t\neq t'}|(\hS_{jj'})_{t t'}| 
        \le \frac{1}{14a}\sqrt{\frac{\pmin}{\pmax}}
        \frac{1}{\sqrt{p^{j} p^{j'}} }.
    \end{align}
\end{asump}
The previous assumption is stronger than the RE($1$ ) property, as shown
by the
following lemma:
\begin{lemma}\label{lem:kappa_alpha}
    Let Assumption~\ref{asump:structure_general} be satisfied. Then RE($1$) is satisfied
    with $\kappa = \sqrt{1 - 1 / a} $.
\end{lemma}
\begin{proof}
    This is Lemma~B.3 in \citep{lounici2011oracle}.
\end{proof}
The following theorem is a modification of Theorem~$5.1$ by \citet{lounici2011oracle}, 
where some adaptations are made to adjust the result to our setting.
\begin{theorem}\label{thm:group_structure_general}
    Let Assumption~\eqref{asump:structure_general} hold, for $\ell=1,\ldots,d$,
    $\epsilon^{\ell}_i \sim
    \mathcal{N}(0,\sigma^2)$ independently, $\sigma^2 > 0$, and 
    with $\delta \in (0, 1)$  set $\lambda \ge 4\lambda_0$, where
    \begin{align*}
        \lambda_0 = \frac{2\sigma}{\sqrt{n} }
        \sqrt{1 + 
            \sqrt{\frac{8\log(d / \delta)}{\pmin}} 
            +
            \frac{8\log(d / \delta)}{\pmin}
        }.
    \end{align*}
    Furthermore, set $c = \left( 1 + \tfrac{24}{7(a - 1)} \right) $. 
    Then, for any $\delta \in (0, 1)$, with probability at least $1 - \frac{\delta}{d}$, any solution 
    $\hbeta_i$ of \eqref{eq:general_group_obj} satisfies
    \begin{align}
        \|\hbeta_i - \bstar_i\|_{2, \infty} \le c\lambda \sqrt{\pmax} 
        \label{eq:param_max_bound}
    .\end{align}
    If, in addition, $\|(\bstar_i)^{\pi(i)}\| > 2c \lambda\sqrt{\pmax}$, 
    then \eqref{eq:param_max_bound} implies that
    \begin{align*}
        \hJ_i = \argmax_{j=1, \ldots, d} \|\hbeta_i^{j}\|
    \end{align*}
    estimates $\pi(i)$ correctly.
\end{theorem}

\begin{proof}
    Most of the proof is similar to the proof of Theorem~$5.1$ in \citet{lounici2011oracle}. 
    We supply a full proof for completeness and because our setting is slightly different. 
    We will need more notation to prove this statement. 
    Set $p_{\infty} = \max_{1\le d} p^{j}$ and define the extended covariance matrices
    $\widetilde{\Sigma}_{jj'}$ of size $p_{\infty}\times p_{\infty}$ as
    \begin{align*}
        \widetilde{\Sigma}_{jj'} 
        = 
        \begin{bmatrix} 
        \begin{array}{c|c}
            \hS_{jj'} & 0 \\
            \hline
            0 & 0
        \end{array}
        \end{bmatrix} & \text{ if } j\neq j'
        \text{ and }
        \widetilde{\Sigma}_{jj} 
        = 
        \begin{bmatrix} 
        \begin{array}{c|c}
            \hS_{jj} - I_{p^{j}\times p^{j}} & 0 \\
            \hline
            0 & 0
        \end{array}
        \end{bmatrix} \text{ if } j=j'
    .\end{align*}
    We also define for any $j=1,\ldots, d$ and $\Delta \in \mathbb{R}^{k\overline{p}}$
    the vector $\widetilde{\Delta}^{j}\in \mathbb{R}^{p_{\infty}}$ 
    such that 
    \begin{align*}
        \widetilde{\Delta}^{j} = 
        \begin{bmatrix} \Delta^{j} \\ 0 \end{bmatrix} 
    .\end{align*}
    Now set $\Delta = \hbeta_i - \bstar_i$ and bound
    \begin{align*}
        \|\Delta\|_{2, \infty}
        = 
        \|\hS\Delta - (\hS - I_{k\overline{p} \times k \overline{p}})\Delta\|_{2, \infty}
        \le 
        \|\hS \Delta\|_{2, \infty}
        + \|(\hS - I_{k\overline{p} \times k \overline{p}}) \Delta\|_{2, \infty}
    .\end{align*}
    The first term is controlled by \eqref{eq:tech_result_2} from
    Lemma~\ref{thm:group_error_1}.
    The latter term can be bounded by noticing that only
    the off-diagonal elements will contribute to the norm. We 
    can bound it using Cauchy-Schwarz:
    \begin{align*}
        \|(\hS &- I_{k\overline{p} \times k \overline{p}}) \Delta\|_{2, \infty}
        =\max_{j=1, \ldots, d}\|((\hS - I_{k\overline{p} \times k \overline{p}}) \Delta)^{j}\|\\
        &=\max_{j=1, \ldots, d}\left[ 
            \sum_{t=1}^{p^{j}} \left( 
                \sum_{j'=1}^{d}\sum_{t'=1}^{p^{j'}} 
                (\widetilde{\Sigma}_{j j'})_{t t'}\widetilde{\Delta}_{t'}^{j'}
            \right)^2
        \right]^{1 / 2}\\
        &\le \max_{j=1, \ldots, d}\left[ 
            \sum_{t=1}^{p^{j}} \left( 
                \sum_{j'=1}^{d}
                (\widetilde{\Sigma}_{j j'})_{tt}\widetilde{\Delta}_{t}^{j'}
            \right)^2
        \right]^{1 / 2}
        +
\max_{j=1, \ldots, d}\left[ 
            \sum_{t=1}^{p^{j}} \left( 
                \sum_{j'=1}^{d}\sum_{t'=1, t' \neq t}^{p^{j'}} 
                (\widetilde{\Sigma}_{j j'})_{t t'}\widetilde{\Delta}_{t'}^{j'}
            \right)^2
        \right]^{1 / 2}
    .\end{align*}
    We now bound both terms separately. The first term can be bounded using 
    an application of 
    Assumption~\ref{asump:structure_general} and then Minkowski's
    inequality.
    The Minkowski's inequality is true for $L^{p}$ norms and tells us
    \begin{align*}
        \|x + y\|_p \le \|x\|_p + \|y\|_p
    .\end{align*}
    In our case this generalises to 
    \begin{align*}
        \left[ \sum_{t=1}^{p_{\infty}} \left( \sum_{j'=1}^{d} |\widetilde{\Delta}^{j'}_t| \right)^2  \right]^{1 / 2}
        =
        \|\sum_{j=1}^{d} \widetilde{\Delta}^{j}\|  
        \le 
        \sum_{j=1}^{d} \|\widetilde{\Delta}^{j}\|
        \le  
        \frac{1}{\sqrt{\pmin} }\sum_{j=1}^{d} \sqrt{p^{j}} \|\widetilde{\Delta}^{j}\|
        =
        \frac{1}{\sqrt{\pmin}}\|\widetilde{\Delta}\|_{2, 1}
    .\end{align*}
    Combining Assumption~\ref{asump:structure_general} with the above inequality
    gives us
    \begin{align*}
        \max_{j=1, \ldots, d}\left[ 
            \sum_{t=1}^{p_{j}} \left( 
                \sum_{j'=1}^{d}
                (\widetilde{\Sigma}_{j j'})_{tt}\widetilde{\Delta}_{t}^{j'}
            \right)^2
        \right]^{1 / 2}
        &\le 
        \frac{1}{14a }\sqrt{\frac{\pmin}{\pmax}}  \left[ 
            \sum_{t=1}^{p_{\infty}} \left( \sum_{j'=1}^{d} |\widetilde{\Delta}_t^{j'}| \right)^2 
        \right]^{ 1/2}\\
        &\le \frac{1}{14a }\sqrt{\frac{\pmin}{\pmax}} \frac{1}{\sqrt{\pmin} }
            \|\widetilde{\Delta}\|_{2, 1}\\
        &\le \frac{1}{14a }\sqrt{\frac{1}{\pmax}} \|\Delta\|_{2, 1}
    .\end{align*}
    The second term can now be bounded by another application of Cauchy-Schwarz: 
    \begin{align*}
        \max_{j=1, \ldots, d}&\left[ 
        \sum_{t=1}^{p^{j}} \left( 
            \sum_{j'=1}^{d}\sum_{t'=1, t' \neq t}^{p^{j'}} 
            (\widetilde{\Sigma}_{j j'})_{t t'}\widetilde{\Delta}_{t'}^{j'}
            \right)^2
        \right]^{1 / 2} \\
        &\le \frac{1}{14a }\sqrt{\frac{\pmin}{\pmax}}
        \max_{j=1,\ldots, d} \left[ \frac{1}{p^{j}}
            \sum_{t=1}^{p^{j}} \left( 
                \sum_{j'=1}^{d} \sum_{t'=1}^{p^{j'}} 
                \frac{|\widetilde{\Delta}_{t'}^{j'}|}{\sqrt{p^{j'}} } 
            \right)^2
        \right] ^{ 1/2}\\
        &\le \frac{1}{14a }\sqrt{\frac{\pmin}{\pmax}}
        \sum_{j'=1}^{d} \sum_{t'=1}^{p^{j'}} 
            \frac{|\widetilde{\Delta}_{t'}^{j'}|}{\sqrt{p^{j'}} }\\
        &\le \frac{1}{14a }\sqrt{\frac{\pmin}{\pmax}} 
            \frac{1}{\sqrt{\pmin}}\|\widetilde{\Delta}\|_{2,1} \\
        &\le \frac{1}{14a}\sqrt{\frac{1}{\pmax}} \|\Delta\|_{2,1}
    .\end{align*}
    The $(2, 1)$-norm term is now bounded using \eqref{eq:group_error_3}. 
    Putting everything together we get 
    \begin{align*}
        \|\hbeta_i - \bstar_i\|_{2, \infty} 
        &\le 
        \|\hS (\hbeta_i - \bstar_i)\|_{2, \infty}
        + \|(\hS - I_{pd \times pd}) (\hbeta_i - \bstar_i)\|_{2, \infty}\\
        &\le \lambda\sqrt{\pmax} 
            + \frac{2}{14a}\sqrt{\frac{1}{\pmax}}
            \left(
                \frac{24 \lambda p^{\pi(i)}}{\kappa^2} 
            \right)\\
        &\le \left(
            1
            +
            \frac{24}{7\kappa^2 a}
        \right)\lambda \sqrt{\pmax} 
    .\end{align*}
    To satisfy both assumptions~(\ref{asump:re_1}, \ref{asump:structure}), we
    need to set $a \kappa^2 = (a - 1)$ as per Lemma~\ref{lem:kappa_alpha}.

    Finally, to prove the final claim, note that \eqref{eq:param_max_bound} combined
    with our sparsity assumption on the true parameters implies
    that for all $j' \neq \pi(i)$ 
    it must be that 
    $\|\hbeta_i^{j'}\| = \|\hbeta_i^{j'} - (\bstar_i)^{j'}\| < c\lambda \sqrt{\pmax} $. 
    We will show that for $\pi(i)$ it must be that $\|\hbeta_i^{\pi(i)} \| > c\lambda \sqrt{\pmax} $. 
    Hence, the estimator gets the correct index with high probability. Indeed, if
    $\|(\bstar_i)^{\pi(i)}\| > 2c \lambda \sqrt{\pmax} $ we get
    \begin{align*}
        \|\hbeta_i^{\pi(i)}\| 
        &= 
        \|(\bstar_i)^{\pi(i)} - ((\bstar_i)^{\pi(i)} - \hbeta_i^{\pi(i)}) \|\\
        &\ge \left|\|(\bstar_i)^{\pi(i)}\| - \|((\bstar_i)^{\pi(i)} - \hbeta_i^{\pi(i)}) \|\right|\\
        &\ge 2c\lambda \sqrt{\pmax} - c \lambda  \sqrt{\pmax}\\
        &= c\lambda \sqrt{\pmax}
    .\end{align*}
\end{proof}

Let us restate the specific version of Theorem~\ref{thm:group_structure} 
again for clarity. This theorem is now a corollary of 
Theorem~\ref{thm:group_structure_general}.
\GroupStruct*
\begin{proof}
    The result follows from Theorem~\ref{thm:group_structure_general}, where in this case $k=d$, 
    and $p^j = p$ for all $j=1,\ldots, d$. 
\end{proof}
\TotalProb*
\begin{proof}
Consider the following estimators
\begin{align*}
    \hbeta_i
    &= \argmin_{\beta \in \mathbb{R}^{dp}}
        \|\Hvec_{i} - \Phi \beta\|^2 + \lambda  \sqrt{p} \|\beta\|_{2,1},\\
    \hJ_{i} 
    &= \argmax_{j=1,\ldots d} \|\hbeta^{j}_i\|,\\
    \tpi
    &: [d] \to [d], i \mapsto \hJ_i
.\end{align*}
We will first show that $\tpi$ estimates $\pi$ with probability at
least $1 - \delta$.  Afterwards, we will show that the event on which $\tpi$ is correct, is
contained in the event that $\hpi$ estimates $\pi$ correctly, implying a lower
bound on the requested probability. 

We apply a union bound
\begin{align*}
    \mathbb{P}(\tpi = \pi) 
    &= 
    \mathbb{P}(\forall i=1,\ldots, d  \mid \hJ_i = \pi(i)) \\
    &= 
    1 - \mathbb{P}(\exists  i=1,\ldots, d  \mid \hJ_i \neq  \pi(i))\\
    &> 1 - \sum_{i=1}^{d} \frac{\delta}{d} \\
    &= 1 - \delta
.\end{align*}

We proceed to the second step. If $\tpi$ estimates $\pi$ correctly, then $\tpi$
is already a valid permutation and $\tpi=\hpi$. 
Indeed, if $\tpi$ is correct then that means that
$\|\hbeta_{i}^{\tpi(i)}\|= \|\hbeta_{i}^{\pi(i)}\|$ is the maximum norm 
for each $i$. Coincidentally, by $\pi$ being a correct permutation, $\tpi$ 
describes a correct matching with largest values, which means that 
$\tpi(i) = \hpi(i)$  for each $i=1,\ldots, d$ and 
\begin{align*}
    \mathbb{P}(\hpi = \pi) 
    \ge \mathbb{P}(\tpi = \pi)
    \ge 1 -\delta
.\end{align*}

\end{proof}

\section{Proofs and Implementation Details for Kernelized Alignment Learning}
\label{app:proofs_non_params}
In this section we provide proof details and the algorithm for
Section~\ref{sec:non_param_learning}.
We will make one adjustment to the Group Lasso
regularization in the optimization problem compared to \eqref{eq:feature_optim},
which is that we square the regularization term. 
This form is theoretically more appealing, but is still
equivalent to the standard formulation: as \citet{bach2008consistency} argues,
the two versions of the optimization problem have the same sets of
solutions when varying the regularization parameters. For $\mu > 0$, the
objective is given by
\begin{align}\label{eq:feature_optim_squared}
    \inf_{\beta^{1}, \ldots, \beta^d} 
        \tfrac{1}{n}
        \|\Hvec_i - \beta( \Mvec ) \|^2
        + 
        \mu \left( \sum_{j=1}^{d} \|\beta^{j}\|_{\mathcal{H}_j} \right)^2
.\end{align}
Let $\hbeta^{1}_i, \ldots, \hbeta^{d}_i$ be the solutions of the above
optimization problem. The translation between regularization parameters that
give the same solutions for \eqref{eq:feature_optim} and \eqref{eq:feature_optim_squared} is given by 
$\lambda = \mu\left(\sum_{j=1}^{d}\|\hbeta_i^{j}\|_{\mathcal{H}_j}\right)$.
\subsection{Representer Theorem}
The squared version of the optimization problem allows us to prove the Representer theorem 
from the main text:
\FeatKernEquiv*
\begin{proof}
    First we state the following result about a variational
    equality for positive numbers
    \begin{align}\label{eq:var_norm}
        \left( \sum_{j=1}^{d} \|\beta^{j}\| \right)^2 
        =
        \inf_{\eta \in \Delta_d}\sum_{j=1}^{d} \frac{\|\beta^{j}\|^2}{\eta^{j}}
    .\end{align}
    A proof of this statement can be found in section 1.5 of \citep{BachJMO12}. 
    Using~\eqref{eq:var_norm} and switching to the squared version of \eqref{eq:feature_optim}
    we rewrite \eqref{eq:feature_optim_squared} as 
    \begin{align*}
        &\inf_{\beta_1, \ldots, \beta_d} 
        \frac{1}{n}
        \|\Hvec_i - \beta(\Mvec)\|^2
        + 
        \mu \left( \sum_{j=1}^{d} \|\beta^{j}\|_{\mathcal{H}_j} \right)^2\\
        &= 
        \inf_{\eta \in \Delta_d}
        \inf_{\beta_1, \ldots, \beta_d} 
        \frac{1}{n}
        \|\Hvec_i - \beta(\Mvec)\|^2
        + 
        \mu \sum_{j=1}^{d} \frac{\|\beta^{j}\|^2_{\mathcal{H}_j}}{\eta_j}\\
        &= (\text{OPT}_1).
    \end{align*}
    We can rewrite this expression further, using the reproducing property of
    the RKHSs $\mathcal{H}_j$, which gives $\beta^{j}(M_j)=\left<\beta^{j}, \varphi_j(M_j) \right>$. 
    Furthermore, defining $\widetilde{\beta^{j}} = \frac{\beta^{j}}{\sqrt{\eta_j} }$ 
    and $\widetilde{\varphi_j} = \sqrt{\eta_j} \varphi_j$
    we rewrite
    \begin{align*}
        (\text{OPT}_1)
        &= 
        \inf_{\eta \in \Delta_d}
        \inf_{\beta_1, \ldots, \beta_d} 
        \frac{1}{n}\sum_{\ell=1}^{n} 
        \left(C^{\ell}_i - \sum_{j=1}^{d} \left< \beta^{j}, \varphi_j(M^{(\ell)}_i)\right>\right)^2+ 
        \mu \sum_{j=1}^{d} \frac{\|\beta^{j}\|_{\mathcal{H}_j}^2}{\eta_j}\\
        &= 
        \inf_{\eta \in \Delta_d}
        \inf_{\widetilde{\beta_1}, \ldots, \widetilde{\beta_d}} 
        \frac{1}{n}\sum_{\ell=1}^{n} 
        \left(C^{(\ell)}_i - \sum_{j=1}^{d} 
            \left< \sqrt{\eta_j} \widetilde{\beta^{j}}, \varphi_j(M^{(\ell)}_j)\right>
            \right)^2
            +
        \mu \sum_{j=1}^{d} \|\widetilde{\beta^{j}}\|_{\mathcal{H}_j}^2\\
        &= 
        \inf_{\eta \in \Delta_d}
        \inf_{\widetilde{\beta_1}, \ldots, \widetilde{\beta_d}} 
        \frac{1}{n}\sum_{i=1}^{n} 
        \left(C^{(\ell)}_i - \sum_{j=1}^{d} 
            \left< \widetilde{\beta^{j}},  \widetilde{\varphi_j}(M^{(\ell)}_j)\right>
            \right)^2+ 
        \mu \sum_{j=1}^{d} \|\widetilde{\beta^{j}}\|_{\mathcal{H}_j}^2\\
        &=(\text{OPT}_2)
    .\end{align*}
    This final expression should be recognized as the feature representation of
    the Representer theorem \citep{SHS2001rep} applied to the kernel described
    by
    \begin{align*}
        \kappa(\eta)(M, M')
        &= \sum_{j=1}^{d} \left<\widetilde{\varphi_j}(M_j), \widetilde{\varphi_j}(M_j') \right>\\
        &= \sum_{j=1}^{d} \eta_j\left<\varphi_j(M_j), \varphi_j(M_j') \right>\\
        &= \sum_{j=1}^{d} \eta_j\kappa_j(M_j, M_j')
    .\end{align*}
    The Representer theorem then gives us that the solution of the inner
    optimization problem in $(\text{OPT}_2)$ can be described by 
    \begin{align*}
        \widetilde{\beta^{j}} 
        = \sum_{\ell=1}^{n} \widetilde{\varphi_j}(M_j^{\ell})(c)_{\ell}
        \iff
        \frac{\beta^{j}}{\sqrt{\eta_j}}  
        = \sqrt{\eta_j}\sum_{\ell=1}^{n} \varphi_j(M^{(\ell)}_j)(c)_{\ell}
        \iff
        \beta^{j}
        = \sum_{\ell=1}^{n} \varphi_j(M^{(\ell)}_j)\eta_j(c)_{\ell}
    \end{align*}
    with $c \in \mathbb{R}^{n}$ and $\eta \in \Delta_{d}$. Alternatively
    interpreted this says that there exist $c^{1}, \ldots, c^{d} \in \mathbb{R}^{n}$ 
    such that $\beta^{j} = \sum_{\ell=1}^{n} \varphi_j(M_j^{(\ell)})(c^{j})_{\ell}$.
    We invoke again the equivalence between the squared and un-squared 
    versions of the optimization problem using the translation of regularization 
    parameters $\lambda = \mu\left( \sum_{j=1}^{d} \|\widehat{\beta}_i^{j}\|_{K_j} \right) $
    and conclude that the solutions 
    of \eqref{eq:feature_optim} are of the same form.
\end{proof}

To get the finite-dimensional optimization problem as stated in
\eqref{eq:kernel_optim_unsquared} we substitute the correct forms of
$\hbeta^{j}_i$ back into the original optimization problem. Define the Gramm
matrices $(K_j)_{\ell k} = \kappa_j(M^{(\ell)}_j, M^{k}_j) =
\left<\varphi_j(M_j^{(\ell)}, \varphi_j(M_j^{k}) \right>$ and observe
\begin{align*}
    &\inf_{\beta_1, \ldots, \beta_d} 
    \frac{1}{n}\sum_{\ell=1}^{n} 
    \left(C^{(\ell)}_i - \sum_{j=1}^{d} \left< \beta^{j}, \varphi_j(M^{(\ell)}_j\right>\right)^2
    + 
    \lambda \sum_{j=1}^{d} \|\beta^{j}\|_{\mathcal{H}_j} \\
    &=\inf_{c^{1}, \ldots, c^{d}\in \mathbb{R}^{n}}
    \frac{1}{n}\sum_{\ell=1}^{n} 
    \left( C^{(\ell)}_i - 
        \sum_{j=1}^{d} (K_jc^{j})_i \right)^2
        + \lambda \sum_{j=1}^{d} \sqrt{(c^{j})^{\top} K_j c^{j}}\\
    &=\inf_{c^{1}, \ldots, c^{d}\in \mathbb{R}^{n}}
    \tfrac{1}{n} 
    \| \Hvec_i - \sum_{j=1}^{d} K_jc^{j} \|^2
        + \lambda\sum_{j=1}^{d} \|c^{j}\|_{K_j}.
\end{align*}
\subsection{Estimator consistency}
\label{app:kernel_consistency}

As stated in the main text, the assumptions in Theorem~\ref{thm:kernel-prob} are explained in this section. 
The assumptions stated in (A-D) ensure that the RKHSs that we work with are nice enough 
and that the function we want to estimate is not too miss specified. 
For a more complete discussion on the assumptions, we refer to \citet{bach2008consistency}. 
To remind ourselves, 
we are given $d$ random variables $M = (M_1,\ldots, M_d)$, where each random variable
lives in $\mathcal{Z}_j$, and $d$ RKHSs $\mathcal{H}_1, \ldots, \mathcal{H}_d$ associated with
$d$ kernels $\kappa_1,\ldots, \kappa_j$. The cross-covariance operator, 
$\Sigma_{ij}$ for $\mathcal{H}_j$ to $\mathcal{H}_i$ is
defined such that for all $(\beta^i, \beta^{j}) \in \mathcal{H}_i \times \mathcal{H}_j$, 
\begin{align}
    \left<\beta^{j}, \Sigma_{ij}\beta^{j} \right> = 
    \mathbb{E}[\beta^{i}(M_i)\beta^{j}(M_j)] - \mathbb{E}[\beta^{i}(M_i)]\mathbb{E}[\beta^j(M_j)]
.\end{align}
The bounded correlation operators $\rho_{ij}$ are defined 
through the decomposition 
$\Sigma_{ij} = \Sigma_{ii}^{1 / 2} \rho_{ij} \Sigma_{jj}^{1 / 2}$ \citep{baker1973joint}.
\begin{enumerate}[label=(\Alph*)]
    \item For each $j=1,\ldots, d$, the Hilbert space $\mathcal{H}_j$ 
        is a separable reproducing kernel Hilbert space associated 
        with kernel $\kappa_j$ and the random variables $\kappa_j(\cdot, M_j)$ 
        are not constant and have finite fourth-order moments. 
    \item For all $i,j=1,\ldots, d$, the cross correlation operators are
        compact $\rho_{ij}$ and the joint correlation operator is invertible.
    \item For each $i=1,\ldots, d$, there exist functions 
        ${\bstar_i}^{1}, \ldots, {\bstar_i}^{d}\in \mathcal{H}_1,\ldots,\mathcal{H}_d$, 
        $b_i \in \mathbb{R}$ and a function $f_i$ of $M$ such that
        \begin{align*}
            C_i = \sum_{j=1}^{d} {\bstar_i}^{j}(M_j) + b_i + f_i(M) + \epsilon_i
        ,\end{align*}
        where $\mathbb{E}[\epsilon_i  \mid M] = 0$ and 
        $\sigma_{\text{min}}^2<\mathbb{E}[\epsilon_i^2  \mid M] < \sigma_{\text{max}}^2$
        with $\mathbb{E}[f_i(M)^2] < \infty$, $\mathbb{E}[f_i(M)] = 0$ 
        and $\mathbb{E}[f_i(M){\bstar_i}^j(M_j) = 0] $ for all $j=1\ldots, d$. 
        We define $\pi(i)$ to be the one index for which 
        ${\bstar_i}^{\pi(i)} \neq 0$.
    \item For all $i, j=1, \ldots, d$, there exists 
        $g_i^j \in \mathcal{H}_j$ such 
        that ${\bstar_i}^{j} = \Sigma^{1 / 2}_{jj} g_i^{j}$. 
\end{enumerate}
For each function, ${\bstar_i}^{\pi(i)}$, hat is non-zero we will require the following condition
\begin{align}\label{eq:kernel_condition}
    \max_{j \neq \pi(i)}
    \left\|
        \Sigma^{1 / 2}_{jj} \rho_{i\pi(i)} \rho^{-1}_{\pi(i)\pi(i)}
        Dg_{\pi(i)}
    \right\|_{\mathcal{H}_i} 
    < 1
,\end{align}
Where $D$ is a block diagonal operator where each block consists of
the operators $\tfrac{1}{\|{\bstar_i}^{j}\|_{\mathcal{H}_j}} I_{\mathcal{H}_j}$.
Condition (B) can be seen as an analogue to the correlation assumption
in Assumptions~\ref{asump:structure} and \ref{asump:structure_general}, as
it ensures that the variables are not too dependent. 

Before we prove Theorem~\ref{thm:kernel-prob}, we will first prove that each
individual index $\pi(i)$ can be estimated consistently. This follows from an
asymptotic result by \citet{bach2008consistency}.
\KernelProb*
\begin{proof}
    To prove this result we first define the estimator of each individual 
    index $\pi(i)$ as  
    \begin{align}
        \hJ_i = \argmax_{j=1,\ldots, d}\|\hbeta_i^{j}\|_{\mathcal{H}_j}
    .\end{align}
    Theorem~11 in \citep{bach2008consistency} gives consistency 
    for the estimated parameters $\hbeta^{j}_i$ and estimated index $\hJ_i$. 
    However, his result is stated for the squared version of the Group Lasso and 
    has as assumption that the $\mu_n$ regularization parameters have the
    property that $\mu_n \to \infty$ and $\sqrt{n}\mu_n \to + \infty$ 
    as the number of data points $n \to \infty$. 
    The translation factor between regularization parameters 
    $\lambda_n=\mu_n(\sum_{j=1}^{d} \|\hbeta^{j}_i\|_{K_j})$
    convergence to a constant in probability, by the consistancy
    of the estimated parameters. 
    This shows that the scalings for $\lambda_n$ and $\mu_n$ are the
    same asymptotically. We conclude that $\hJ_i$ estimates 
    $\pi(i)$ with probability tending to $1$. 

    Remember that the norms of $\hbeta^{j}_i$ and $\hat{c}^{j}_i$ are the
    same through $\|\hbeta^{j}_i\|_{\mathcal{H}_j} =
    \|\hat{c}_i^{j}\|_{K_j}$.  This means that we have consistency for
    estimators of $\pi(i)$ that are based on $\|\hat{c}_i^{j}\|_{K_j}$ as
    well. For each $i=1,\ldots, d$ we can repeat the above argumentation
    to get consistency for each $\hJ_i$ separately.
    To combine the conclusions, we apply the same argument as in the finite dimensional case
    and a a union bound that finishes our proof, 
    \begin{align*}
        \mathbb{P}(\hpi= \pi) 
        &\ge  1 - \mathbb{P}(\exists i=1,\ldots,d  \mid \hJ_i \neq \pi(i))\\
        &\ge 1 - \sum_{i=1}^{d}\mathbb{P}(\hJ_i \neq \pi(i)) 
        \to 1
    .\end{align*}
\end{proof}

\subsection{Implementation}

\KernelImp*
\begin{proof}
Applying the Cholesky decomposition $K_j = L_jL_j^{\top}$ 
for each $j=1,\ldots,d$  and substituting this into \eqref{eq:kernel_optim_lstsq}
gives
\begin{align*}
    \inf_{c^{1}, \ldots, c^{d}\in \mathbb{R}^{n}}
        \tfrac{1}{n} 
        \| \Hvec_i - \sum_{j=1}^{d} K_jc^{j} \|^2
            + \lambda \sum_{j=1}^{d} \|c^{j}\|_{K_j} 
    &= \inf_{c^{1}, \ldots, c^{d}\in \mathbb{R}^{n}}
        \tfrac{1}{n} 
        \| \Hvec_i - \sum_{j=1}^{d} L_jL_j^{\top}c^{j} \|
            + \lambda \sum_{j=1}^{d} \|L_j^{\top}c^{j}\|\\
    &= \inf_{\gamma^{1}, \ldots, \gamma^{d}\in \mathbb{R}^{n}}
        \tfrac{1}{n} 
        \| \Hvec_i - \sum_{j=1}^{d} L_j\gamma^{j} \|^2
            + \lambda \sum_{j=1}^{d} \|\gamma^{j}\|
.\end{align*}
\end{proof}

The number of parameters now scales with the number of data samples.
Computationally, this quickly becomes unwieldy.  We apply a Nystr\"om style
approximation by sub-sampling $m \ll n$ columns of each $K_j$ Gramm
matrix and using those to approximate the full Gramm matrix \citep{WS2000nystrom}. 

\section{Probability Results}
\label{app:probability results}
\begin{lemma}\label{lem:chi_concentration_bound}
For $j=1,\ldots, d$ and $\sigma^2> 0$ let $\frac{\chi_j^2}{\sigma^2}$ be 
independent chi-square distributed random variables with $p^{j}$ 
degrees of freedom for $j=1,\ldots, d$. Then, with $\delta \in (0, 1)$  and for 
\begin{align*}
    \lambda_0 = \frac{2\sigma}{\sqrt{n} }
        \sqrt{1 + 
            \sqrt{\frac{4\log(d / \delta)}{\pmin}} 
            +
            \frac{4\log(d /\delta)}{\pmin}
        },
\end{align*}
we have 
\begin{align*}
    \mathbb{P}\left(\max_{1\le j \le d} 
        \frac{\chi_j}{\sqrt{np^{j}}} \le \frac{\lambda_0 }{2}
    \right)
    \ge 1 - \delta
.\end{align*}
\end{lemma}
\begin{proof}
    This is Lemma 8.1 in \citep{buhlmannVG11} and substituting $x = \log(1 / \delta)$. 
\end{proof}
The previous lemma is the general version of a concentration inequality that
is needed in the proof of Theorem~\ref{thm:group_structure_general}. 
The concentration inequality that we want to use is the following.

\begin{lemma}\label{lem:feature_concentration}
    Let $\sigma^2 > 0$ and assume that $\epsilon^{(1)}, \ldots \epsilon^{(n)}$ are independently 
    $\cN(0, \sigma^2)$ distributed, 
    $\Phi$ as in Appendix~\ref{app:proofs_params}, 
    and with $\delta \in (0,1)$ set $\lambda \ge 4\lambda_0$ for
     \begin{align*}
        \lambda_0 = \frac{2\sigma}{\sqrt{n} }
            \sqrt{1 + 
                \sqrt{\frac{4\log(d / \delta)}{\pmin}} 
                +
                \frac{4 \log(d / \delta)}{\pmin}
            } 
    \end{align*}   
    Then, $\mathbb{P}(\mathcal{A}) \ge 1-\delta$, where 
    $\mathcal{A} = \bigcap_{j=1}^{d}\mathcal{A}_j$ with the events
    $\mathcal{A}_j = \{\frac{1}{n}\|(\Phi^{\top}\bm{\epsilon})^{j}\| 
    \le \frac{\lambda \sqrt{p^{j}} }{2}\} $ for all $j=1,\ldots d$
    and $\bm{\epsilon} = [\epsilon^{(1)}, \ldots\epsilon^{(n)}]^{\top}$.
\end{lemma}
\begin{proof}
    By assumption, 
    $I_{p^{j}\times p^{j}}=\hS_{j} =\tfrac{1}{n}(\Phi_{j})^{\top}\Phi_{j}$ 
    and the fact that 
    $\bm{\epsilon} \sim \cN(0, \sigma^2I_{n \times n})$, 
    we first see
    \begin{align*}
        \tfrac{1}{\sigma\sqrt{n} }(\Phi\bm{\epsilon})^{j}
       &= \tfrac{1}{\sigma\sqrt{n} }\Phi_j^{\top}\bm{\epsilon}\\
       &\sim \cN(0, \tfrac{1}{n}\Phi_j^{\top}I_{n\times n}\Phi_j)\\
       &\sim \cN(0, I_{p^{j}\times p^{j}})
    .\end{align*}
    This shows us that $\tfrac{1}{\sigma^2n}\|(\Phi^{\top}\bm{\epsilon})^{j}\|^2$
    has a  chi-squared distribution. We can now apply 
    Lemma~\ref{lem:chi_concentration_bound} by noticing that it holds 
    also  holds for $\lambda \ge 4\lambda_0 \ge \lambda_0$ and that 
    \begin{align*}
        \bigcap_{j=1}^{d}\mathcal{A}_j = 
        \left\{\max_{1\le j\le d} \frac{1}{\sqrt{np^{j}} } 
        \frac{1}{\sqrt{n} }\|(\Phi\bm{\epsilon})^{j}\| \le \frac{\lambda}{2}\right\} 
    .\end{align*}
\end{proof}

\section{Experimental Details}
\label{app:experiments}
All the code to reproduce the experiments and figures in this paper
is provided as a GitHub repository at \url{https://github.com/HiddeFok/sample-efficient-learning-of-concepts}.
We perform synthetic experiments on a $192$-core CPU node 
(AMD Genoa) with $336$GB of RAM. We perform the image benchmark experiments 
on single $18$-core GPU node (NVIDIA A100 GPU and Intel XEON CPU)
with $120$GB of RAM.

\subsection{Toy Dataset Experiments}
\label{app:synth_experiments}

\begin{figure}[t]
    \centering
    \includegraphics{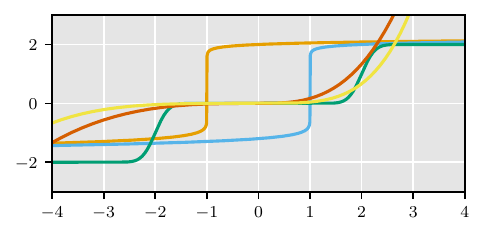}
    \caption{The different types of diffeomorphisms used in the misspecified case.}
    \label{fig:diffeomorphisms}
\end{figure}

The synthetic experiments can be subdivided into $4$ sets, based 
on which features mapping is used. These features mappings are 
linear features, spline features , random Fourier features and kernels. 
Here, we describe the data generatation and 
the hyperparameters of the features and kernels. 

\subsubsection{Data Generation}

For the synthetic experiment we sample the $C \in \mathbb{R}^{d}$ variables
from a $\cN(0, (1 - \rho)I_{d\times d} + \rho\ind)$ distribution, 
where $\ind$ denotes a matrix filled with only $1$'s. The  $\rho \in (0,1)$ 
parameter controls the amount of correlation between the variables. 
For the experiments with continuous concepts, 
we sample $n$ data points, on which we perform a $80/20$ train/test
data split. For the binary concept experiments, we perform a $50/50$ train/test split, to ensure
that both labels of each concept appear in the test set. We also keep the
correlation coefficient in the test set at $\rho=0$, while varying the correlation
in the  train set.

We now simulate the $M$ variables directly from the $C$ variables, instead of using a causal representation learning method that would learn them from the observations $X$.
This allows us to control the different sources of noise more carefully.
There are $2$ settings
in which we generate the $M \in \mathbb{R}^{d}$ variables, well-specified (following all of our assumptions) and misspecified (settings in which we consider more general mixing functions beyond our theoretical assumptions.)

\paragraph{Wellspecified.} In the wellspecified setting we generate
the $M$ variables by applying a map consisting of the features and kernels
used to estimate the permutation. This setting is a sanity check to see
if our estimator works in a setting that satisfies all the required
assumptions. For each dimension $j=1,\ldots, d$, a random weight
vector $\bstar_j\in \mathbb{R}^{p}$ is sampled, such that 
$\|\bstar_j\| \in [16\lambda_0, 32\lambda_0]$ uniformly. A permutation 
$\pi\colon \{1, \ldots, d\} \to \{1, \ldots, d\}$ is uniformly sampled
from all possible permutations. Finally, with independent
$\epsilon_i \sim \cN(0, \sigma^2)$ noise variables we get 
\begin{align*}
    C_i = \varphi(M_{\pi(i)})^{\top}\bstar_{\pi(i)} + \epsilon_i
.\end{align*}
\paragraph{Misspecified.} In the misspecified setting we
generate the $M$ variables by first sampling $d$ 
diffeomorphisms uniformly, $\{f_i \colon \mathbb{R} \to \mathbb{R}\}_{i=1}^{d}$,
from a set of pre-specified diffeomorphisms. The set of possible functions is plotted
in Figure~\ref{fig:diffeomorphisms}. Each function gets a random scaling 
$w_i$ uniformly in $[-2,  2]$. Finally, we get
\begin{align*}
    C_{i} = w_{\pi(i)} f_{\pi(i)}(M_{\pi(i)}) + \epsilon_i,
\end{align*}
with $\epsilon_i \sim \cN(0, \sigma^2)$ again independent noise
variables.

\paragraph{Target Label.} In the experiment where we need binary concepts, the concept 
labels were created by thresholding each variable at the midpoint. The underlying 
continuous score was generated using the misspecified setting from above. In this case the
midpoint is at $0$
and we get 
\begin{align*}
    \overline{C_i} = \ind[C_i \le 0]
.\end{align*}
The $M$ variables are kept continuous and a logistic regression is used to map $M$ 
to $\overline{C}$. To create the target labels, a subset of size
$k = \min(3, \left\lfloor  \frac{d}{5}\right\rfloor)$ is sampled
from $\{1, \ldots, d\} $, without replacement. To determine if the
target
label is $0$ or $1$, we sum the binary concepts in the $k$ selected variables
and check if the sum is at least  $\max(1, \left\lfloor \frac{d}{10} \right\rfloor)$.

\textbf{Metrics.}
For each of the experiments with continuous variables we report the the mean error in the learned permutation $\pi$ of the variables w.r.t to the true permutation $\tpi$ of the variables, defined as
$\text{MPE} = \frac{1}{d}\sum_{i=1}^{d} \ind\{\tpi(i) \neq \pi(i)\}$.
Moreover, we report the $R^2$ score and execution
time. To evaluate different settings of our estimator, 
we compare with the MPE, $R^2$ and execution time of using only
the purely linear version of our estimator on that particular
data setting. We also report the MPE and execution times of assigning the variables based on
Pearson or Spearman correlations as naive baselines. 
In the binary experiments, we report the MPE, the mean concept accuracy,
the label accuracy, 
the OIS-metric and the NIS-metric by \citet{zarlenga2023towards}.

\textbf{Parameters of the experiments.}
We vary the following parameters: the regularization parameter $\lambda$, 
the dimension $d$, the correlation  $\rho$ and the number of data points
$n$. As stated before, in each of the continuous experiment we look at the wellspecified
and misspecified case. The settings in the continuous case are:
\begin{itemize}
    \item The regularization parameter varies in 
        $\lambda \in \{0.001, 0.005, 0.01, 0.05, 0.1, 0.5, 1\} $.
        The other settings are set to $d\in \{20, 60, 100\} $, 
        $\rho=0$ and $n=1250$.
    \item The dimension is varied in $d \in \{5, 30, 60, 80, 100\} $.
        The other settings are set to $\lambda \in \{0.001, 0.01, 0.1\} $,
        $\rho=0$ and $n=1250$.
    \item The correlation parameter varies in 
        $\rho\in \{0, 0.2, 0.4, 0.6, 0.8, 0.95, 0.99\} $.
        The other settings are set to $\lambda \in\{0.001, 0.01, 0.1\} $,
        $d=60$ and $n=1250$.
    \item The total number of data points
        is varied in $n \in \{65, 125, 1250, 2500, 5000\} $.
        The other settings are set to $\lambda \in \{0.001, 0.01, 0.1\} $,
        $d=60$ and $\rho=0$.
\end{itemize}

In the case of the binary concepts, we have fewer combinations, because
we focus more on the downstream task and comparing the estimator
to baselines:
\begin{itemize}
    \item The dimension is varied in $d \in \{5, 10, 15, 20, 30\} $.
        The other settings are set to $\lambda \in \{0.001, 0.01, 0.1\} $,
        $\rho=0.5$ and $n=2000$.
    \item The total number of data points 
        is varied in $n \in \{100, 200, 2000, 4000, 10000\} $.
        The other settings are set to $\lambda \in \{0.001, 0.01, 0.1\} $,
        $d=20 $ and $\rho=0.5$.
\end{itemize}

\subsubsection{Feature and Kernel setting}

We now describe the types of models we consider in our evaluation on the synthetic data.

\paragraph{Linear features} In the linear case no transformation is applied
to the $M$ variables. 

\paragraph{Spline features} In the spline features case we perform the 
regression using a spline basis transformation, either piecewise linear or cubic splines. 
We expect this method to work especially well, because the cubic splines 
form a dense subset in the space of twice-differentiable functions, of which
the diffeomorphisms are a subset. To calculate these features we use the
\verb|SplineTransformer| class of the \verb|scikit-learn| package. The total 
number of feature parameters is calculated as $p=n_k + n_d - 1$, where $n_k$ 
is the number of knots and $n_d$ is the degree of each spline. In each of 
the toy dataset experiments the number of knots was $n_l \in \{4, 8\} $ 
and the degrees were $n_d \in \{1, 3\} $. 

\paragraph{Random Fourier features} For the random Fourier features we use a 
varying number of random features. We sample random features that approximate
the RBF kernel. To sample these features we use the \verb|RBFSampler| class
of the \verb|scikit-learn| package. The total number of feature parameters 
in this case is the number of random Fourier features. The number
of features in the toy dataset experiments is $p \in \{2, 4, 6, 8\} $.

\paragraph{Kernels} For the kernel experiments we perform the experiments for 
several kernels: the polynomial kernel, the RBF kernel, the Brownian kernel
and a Sobolev kernel. These kernels are given by
\label{p:kernel_examples}
\begin{align*}
    \kappa_{\text{pol}}(x, y) &= (1 + \left<x, y \right>)^3\\
    \kappa_{\text{RBF}}(x, y) &= \e^{-(x - y)^2}\\
    \kappa_{\text{Lap}}(x, y) &= \e^{-|x - y|} \\
    \kappa_{\text{cos}}(x, y) &= \cos(\left<x, y \right>)
.\end{align*}

\subsection{Action/Temporal Datasets Experiments}
\label{app:at_experimets}
We consider a synthetic data benchmark from the causal representation learning literature  \citep{lachapelle2022dms}.
The data generation settings, model architectures and training hyperparameters
are taken from the original paper \citep{lachapelle2022dms}. The implementation
can be found at 
\url{https://github.com/slachapelle/disentanglement_via_mechanism_sparsity/tree/main}, 
and is available under the Apache License 2.0. 

\subsubsection{Dataset Generation Details}
The benchmark consists of temporal data sequences, $\{(X^{t}, z^{t},
a^{t})\}_{t=1}^{T}$, where $X^{t}\in \mathbb{R}^{20}$ is the observed data,
$a^{t} \in \mathbb{R}^{10}$ is an action, which is seen as an auxiliary
variable in the ICA framework developed in \citep{khemakhem2020vaeica}, and
$z^{t}\in \mathbb{R}^{10}$ is the latent causal variable. 
The ground-truth mixing function $f$ is a random neural network with three
hidden layers of $20$ units with Leaky-ReLU activations with negative slope of
$0.2$. The weight matrices are sampled independently 
according to $\cN(0, 1)$ and the weight matrices are then orthogonalized to
ensure invertability of the mixing function. The observational noise $\epsilon$
in each dimension is sampled according to $\cN(0, 10^{-4})$ and is added
to $f(z^{t})$. The transitions
from $(z^{t-1}, a^{t-1})$ to $z^{t}$ are sampled according
to $\cN(\mu(z^{t-1}, a^{t-1}), 10^{-4}I_{10 \times 10})$. The mean function
$\mu$ will be different between the Action Sparsity dataset
and the Temporal Sparsity datasets.

\paragraph{Action Sparsity Dataset} In this case, the sequences have length
$T=1$ and the mean function is given by
\begin{align*}
    \mu(z^{t-1}, a^{t-1})_i \coloneqq 
        \sin(\tfrac{2 + i}{\pi}a_i^{t-1} +(i-1))
        +
        \sin(\tfrac{2 + i}{\pi}a_{i-1}^{t-1} +(i-1) )
,\end{align*}
where the index $i=-1$ is periodically identified with $i=10$. 

\paragraph{Temporal Sparsity Dataset} In this case, the sequences have length $T=2$ and the mean
function is given by
\begin{align*}
    \mu(z^{t-1}, a^{t-1})_i \coloneqq 
    z_i^{t-1} + 0.5 \sum_{j=1}^{i} \sin(\tfrac{2 + i}{\pi}z_j^{t-1} + (i-1))
.\end{align*}

\paragraph{Target Labels} We employ the same tactic as before to create 
the target labels. The true latent variables $z^{t}_i$ are binarized 
by looking at the empirical range over the whole dataset, and then thresholded at the midpoint. 
The target label is then determined by sampling $3$ dimensions 
in $\{1, \ldots, 10\} $ and setting the label to $1$ if at least $2$ of
the $3$ latent variables labels are $1$. 

In both datasets we sample $10^{6}$ points and split the data $80/20$ for the
train/test split.
\subsubsection{Model Architectures} 
\begin{table}[t]
  \caption{Architecture details for the encoder and decoder used in the 
    temporal and action sparsity dataset experiments.}
    \label{tbl:dms_architecture}
    \centering 
    \small
    \begin{tabular}{m{1cm}m{1cm}m{2.5cm}m{2.5cm}}
        \toprule
        & Layer & Hidden Size & Activation Function\\
        \bottomrule
        \multirow{7}*{Encoder} 
        & Linear & 512 & LeakyReLU(0.2)\\   
        & Linear & 512 & LeakyReLU(0.2)\\   
        & Linear & 512 & LeakyReLU(0.2)\\   
        & Linear & 512 & LeakyReLU(0.2)\\   
        & Linear & 512 & LeakyReLU(0.2)\\   
        & Linear & 512 & LeakyReLU(0.2)\\   
        & Linear & $2\cdot 10$ & -  \\
        \hline
        \multirow{7}*{Decoder}
        & Linear & 512 & LeakyReLU(0.2)\\   
        & Linear & 512 & LeakyReLU(0.2)\\   
        & Linear & 512 & LeakyReLU(0.2)\\   
        & Linear & 512 & LeakyReLU(0.2)\\
        & Linear & 512 & LeakyReLU(0.2)\\   
        & Linear & 512 & LeakyReLU(0.2)\\
        & Linear & 20 & -  \\
        \bottomrule
    \end{tabular}
\end{table}

As the first step of our pipeline we compare three models: TCVAE \citep{chen2018isolating}, iVAE \citep{khemakhem2020vaeica} and DMS-VAE \citep{lachapelle2022dms}. The same encoder
and decoder architecture is used for all models: an MLP with $6$ layers of $512$ units
with LeakyReLU activations with negative slope $0.2$.  
The encoder $f_{\text{enc}}(x; \theta)$ outputs the
mean and standard deviation of $q_{\theta}(z^{t} \mid x^{t})$, which are 
the densities of normal distributions. The latent 
transition distribution $\hat{p}_{\lambda}(z_i^{t} \mid z^{<t}, a^{<t})$, 
where $z^{<t}=(z^{t'})_{t'=1}^{t-1}$ and $a^{<t}=(a^{t'})_{t'=1}^{t-1}$, 
is also learned by a fully connected neural network.
The decoder $f_{\text{dec}}(z; \psi)$ tries to reconstruct the original
data from the learned encodings. 
A minibatch size of $1024$ is used for training. See Table~\ref{tbl:dms_architecture}
for a detailed description.

The differences between the three methods come from 
the loss function that is optimized. The common term in each of the optimizations
is the Evidence Lower Bound (ELBO) objective, which is given by 
\begin{align}
    \label{eq:elbo_obj_dms}
    \text{ELBO}(\theta, \psi, \lambda)
    &=
    \sum_{t=1}^{T} \bE_{z^{t}\sim q_{\theta}(\cdot  \mid x^{t})}
    \left[ \log p_{\psi}(x^{t}  \mid z^{t}) \right] \\
    &\phantom{=}- 
    \bE_{z^{<t} \sim q_{\theta}(\cdot  \mid x^{<t})}\left[
        \infdivKL{q_{\theta}(z^{t} \mid x^{t})}{\hat{p}_{\lambda}(z^{t}\mid z^{<t}, a^{<t})}
    \right]
,\end{align}
where $\infdivKL{\cdot}{\cdot}$ is the Kullback-Leibler divergence.

\paragraph{TCVAE} We use the implementation of TCVAE \citep{chen2018isolating} by
\citet{lachapelle2022dms}.
The loss function consists
of the same components as the ELBO objective, but they decompose
it into $3$ terms and add a weight parameter to each of the terms.
The hyperparameters for the training 
procedure can be found in Table~\ref{tbl:dms-training}.

\paragraph{iVAE} We use the implementation iVAE
\citep{khemakhem2020vaeica} by
\citet{lachapelle2022dms}. The loss function here is 
similar to the ELBO objective, but it adds one parameter $\beta$ to the KL-term
in the objective. The hyperparameters for the training 
procedure can be found in Table~\ref{tbl:dms-training}.

\paragraph{DMS-VAE} \citet{lachapelle2022dms} introduce a sparsity regularization in the ELBO objective to prove the identifiability in temporal and action settings. 
The new objective is given by
\begin{align*}
    \text{ELBO}(\theta, \psi, \lambda) 
    + \alpha_z \|\hat{G}_z \|_{0}
    + \alpha_a \|\hat{G}_a \|_{0}
.\end{align*}
The variable $\hat{G}_z$ is a learned matrix, representing the relations between 
the latent variables between two time steps. The variable $\hat{G}_a$ 
is a learned matrix representing the relations between 
the actions and the latent variables. The norm $\|\cdot\|_{0}$ counts
the number of non-zero terms. This is a discrete objective and 
can transformed into a continuous objective using the Gumbel-Softmax
trick \citep{MaddisonMT17, JangGP17}.

Instead of using the regularized objective,
the authors also propose a constraint-based
optimization procedure on the ELBO objective, where the constraint is determined by
the number of edges in the learned graph.
We use their constrained
optimization method, which provides an optimization schedule that we set by enabling the \verb|--constraint_schedule| flag.
The other hyperparameters for the training 
procedure can be found in Table~\ref{tbl:dms-training}.

\begin{table}[b]
    \caption{The hyperparameters used for the training of the DMS-VAE, TCVAE and 
    iVAE for the action and temporal sparsity datasets}
    \label{tbl:dms-training}
    \centering 
    \small
    \begin{tabular}{m{4cm}m{8cm}}
        \toprule
        \textbf{Hyperparameter} & \textbf{Value}\\
        \bottomrule
        Batch Size & 1024 \\
        \multirow{2}*{Optimizer} & Adam \citep{KingmaB14} and \\
                  & Cooper \citep{gallego2021flexible}\\
        Learning rate & 5e-4 (DMS-VAE), 1e-4 (iVAE), 1e-3 (TCVAE)\\
        KL divergence factor $\beta$ & 1.0\\
        Number of latents & 20\\
        Number of epochs & 500\\
        Gumbel Softmax temperature & 1.0\\ 
        \bottomrule
    \end{tabular}
\end{table}
\subsection{Temporal Causal3DIdent}
\label{app:temp_causal_experimets}

\begin{figure}[t]
    \centering
    \begin{tikzpicture}
        \centering
        \node (a) {\includegraphics{./figs/causal_ident/teapot.png}};       
        \node[right = 0.1cm of a ] (b) {\includegraphics{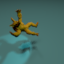}};       
        \node[right = 0.1cm of b ] (c) {\includegraphics{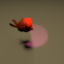}};       
        \node[right = 0.1cm of c ] (d) {\includegraphics{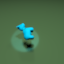}};       
        \node[right = 0.1cm of d ] (e) {\includegraphics{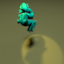}};       
        \node[right = 0.1cm of e ] (f) {\includegraphics{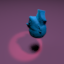}};       
        \node[right = 0.1cm of f ] (g) {\includegraphics{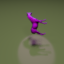}};       

        \node[below = 0.1cm of a] {(a)};
        \node[below = 0.1cm of b] {(b)};
        \node[below = 0.1cm of c] {(c)};
        \node[below = 0.1cm of d] {(d)};
        \node[below = 0.1cm of e] {(e)};
        \node[below = 0.1cm of f] {(f)};
        \node[below = 0.1cm of g] {(g)};
        
    \end{tikzpicture}
    \caption{Examples of the $7$ shapes in the Temporal Causal3DIdent dataset.
        From left to right: teapot, armadillow, bunny, cow, dragon, head and
        horse.}
\label{fig:causal-3d-ident}
\end{figure}

We consider an image data benchmark from the causal representation learning literature  \citep{LippeMLACG22}.
The dataset, model architectures and training hyper parameters are taken 
from \citep{LippeMLACG22}. The implementation 
can be found at 
\url{https://github.com/phlippe/CITRIS/tree/main}, which is available
under the BSD 3-Clause Clear License.

\subsubsection{Dataset Details}
The data comes from a setting
which is referred to as Temporal Intervened Sequences. The assumption is that 
there are $d$ causal variables $(G_1, \ldots, G_d)$ and a corresponding
causal graph $\mathcal{G}=(V, E)$ where each node $i \in V$ represents a causal variable $G_i$.
The variables can be real-valued or vector-valued, and each edge $(i, j) \in E$ 
represents a relation between $G_i$ and $G_j$. 
We assume that there 
are $T$ time steps and for every $t=1,\ldots,T$ the causal variables
follow a stochastic process.
We therefore have a sequence $\{(G_1^{t}, \ldots, G_d^{t})\}_{t=1}^{T}$, 
where only the causal variables in $t-1$ are the parents of the causal variables
in time step $t$. 
The observations $X^t$ at each time step are created through a mixing function 
$X^{t}=f(G_1^{t}, \ldots G^{t}_d, \epsilon^{t})$, where $ \epsilon^{t}$ is an i.i.d. noise variable.
Additionally, we have access to a $d$-dimensional binary vector at each time step $I^{t}\in \{0,1\}^{d}$ 
that tells us which causal variables have been intervened on, but not with which value.

The causal variables that are used in the data generating process are the following:
\begin{itemize}
    \item The \textbf{object position} (\verb|pos_o|) is modelled in $3$ dimensions
        $(x, y, z) \in [-2, 2]^{3}$. The values are forced to be in this interval to ensure
        that the object does not disappear from the image, becomes too small or
        covers the whole image.

    \item The \textbf{object rotation} (\verb|rot_o|) is modelled in $2$ dimensions
        $(\alpha, \beta) \in [0, 2\pi)^2$. Distances for angles are calculated in 
        a periodic fashion, ensuring that angles close to $0$ and $2\pi$ are 
        close together.

    \item The \textbf{spotlight rotation} (\verb|rot_s|) is the positioning of the spotlight
        that shines on the object. The value range is $[0, 2\pi)$, where distances
        are again calculated in a periodic fashion.

    \item The \textbf{spotlight hue} (\verb|hue_s|) is the color of the spotlight. The 
        range of the value is $[0, 2\pi)$, where $0$ corresponds to red. 

    \item The \textbf{background hue} (\verb|hue_b|) is the color of the background.
        The value range is $[0, 2\pi)$ with $0$ corresponding to red again.

    \item The \textbf{object hue} (\verb|hue_o|) is the color of the object, 
        with value range is $[0,2\pi)$ and again with $0$ representing red.
\end{itemize}
\citet{LippeMLACG22} generate the data using Blender, a setup inspired
by \citet{KugelgenSGBSBL21} and using code provided by \citet{ZimmermannSSBB21}.
They generate the dataset by starting with an initial random set of causal variables. 
They then sample the causal variables in each subsequent time step by following a 
specific conditional distribution, which is given by the set of equations 
in \eqref{eq:set_equations}.
\begin{align}
    \begin{split}\label{eq:set_equations}
    f(a,b,c) 
    &= \frac{a - b}{2} + c\\
    \mathrm{pos\_x}^{t+1} 
    &= f(1.5 \cdot \sin(\mathrm{rot\_}\beta^{t}), \mathrm{pos\_x}^{t}, \epsilon_x^{t})\\
    \mathrm{pos\_y}^{t+1} 
    &= f(1.5 \cdot \sin(\mathrm{rot\_}\alpha^{t}), \mathrm{pos\_y}^{t}, \epsilon_y^{t})\\
    \mathrm{pos\_z}^{t+1} 
    &= f(1.5 \cdot \sin(\mathrm{rot\_}\alpha^{t}), \mathrm{pos\_z}^{t}, \epsilon_z^{t})\\
    \mathrm{rot\_}\alpha^{t+1} 
    &= f(\mathrm{hue\_}b^{t}, \mathrm{rot\_}\alpha^{t}, \epsilon_{\alpha}^{t})\\
    \mathrm{rot\_}\beta^{t+1} 
    &= f(\mathrm{hue\_}o^{t}, \mathrm{rot\_}\beta^{t}, \epsilon_{\beta}^{t})\\
    \mathrm{rot\_}s^{t+1} 
    &= f(\mathrm{atan}2(\mathrm{pos\_x}^{t}, \mathrm{pos\_y}^{t}, \mathrm{rot\_}s^{t}, \epsilon_{rs}^{t})\\
    \mathrm{hue\_s}^{t+1} 
    &= f(2\pi - \mathrm{hue\_b}^{t}, \mathrm{hue\_s}^{t}, \epsilon_{hs}^{t})\\
    \mathrm{hue\_b}^{t+1} 
    &= \mathrm{hue\_b}^{t}+ \epsilon_{b}^{t}\\ 
    \mathrm{hue\_b}^{t+1} 
    &= f(g(i), \mathrm{hue\_o}^{t}, \epsilon_{ho}^{t})
    \end{split}
\end{align}

\begin{table}[b]
    \caption{Output of the $g$ function for each object shape. The $\text{avg}$ function for angles
        is defined as $\text{avg}(\alpha, \beta)
        = 
        \text{atan}2\left( \frac{\sin(\alpha) + \sin(\beta)}{2},\frac{\cos(\alpha) + \cos(\beta)}{2}\right)$}
    \label{tbl:shape-values}
    \centering 
    \begin{tabular}{m{4cm}m{7.25cm}}
        \toprule
        \textbf{Object shape} & \textbf{Object hue goal}\\
        \bottomrule
        Teapot Size & $0$ \\
        Armadillo & $\frac{2\pi}{5}$\\
        Hare & $\text{avg}(\mathrm{hue\_s}^{t}, \mathrm{hue\_b}^{t})$\\
        Cow & $\frac{4 \pi}{5}$\\
        Dragon & $ \pi + \text{avg}(\mathrm{hue\_s}^{t}, \mathrm{hue\_b}^{t})$\\
        Head & $\frac{6\pi}{5}$\\
        Horse & $\frac{8\pi}{5}$\\
        \bottomrule
    \end{tabular}
\end{table}

All the noise variables, $\epsilon$, are independently $\cN(0, 10^{-2})$ distributed for the position
and $\cN(0, (0.15)^2)$ distributed for the angles. The $g$ function in the final
line maps the object shapes to specific values detailed in Table~\ref{tbl:shape-values}.

The object shape is changed in each time step with a probability of $0.05$. If it is
changed, a new shape is sampled uniformly over the $7$ shapes.

They then sample for each time step the intervention targets $I^{t+1}_i \sim \text{Bernoulli}(0.1)$.
If a causal variable is intervened on, it is replaced with a random sample
from $U(-2, 2)$ for the position values or $U(0, 2\pi)$ for the angles. For the object
shape a uniform distribution over the $7$ shapes is used. They run this
simulation for $250,000$ steps, which is the full dataset.

We use the already generated dataset downloaded from 
\url{https://zenodo.org/records/6637749#.YqcWCnVBxCA}, which is available
under the Creative Commons Attribution 4.0 International license.

\paragraph{Target Labels} We employ the same tactic as before to create 
the target labels. The true latent variables $G_i$, are binarized 
by looking at the empirical range over the whole dataset, and then thresholded at the midpoint. 
The target label is then determined by sampling three dimensions 
in $\{1, \ldots, 7\} $ and setting the label to $1$ if at least two of
the three latent variables labels are $1$.

\subsubsection{Model Architectures}
\begin{table}[t]
    \caption{Architecture details for the encoder and decoder used in the Temporal Causal3DIdent
    experiments.}
    \label{tbl:citris_architecture}
    \centering 
    \small
    \begin{tabular}{m{1cm}m{2cm}m{3.5cm}m{0.75cm}m{0.75cm}m{3cm}}
        \toprule
        & Layer & $\begin{array}{c}\text{Feature Dimension} \\ (\text{H} \times \text{W} \times \text{C})\end{array}$ & Kernel & Stride & Activation Function\\
        \bottomrule
        \multirow{11}*{Encoder} 
        & Conv & 32 $\times$ 32 $\times$ 64 & 3 & 2 & BatchNorm+SiLU\\
        & Conv & 32 $\times$ 32 $\times$ 64 & 3 & 1 & BatchNorm+SiLU\\
        & Conv & 16 $\times$ 16 $\times$ 64 & 3 & 2 & BatchNorm+SiLU\\
        & Conv & 16 $\times$ 16 $\times$ 64 & 3 & 1 & BatchNorm+SiLU\\
        & Conv & 8 $\times$ 8 $\times$ 64 & 3 & 2 & BatchNorm+SiLU\\
        & Conv & 8 $\times$ 8 $\times$ 64 & 3 & 1 & BatchNorm+SiLU\\
        & Conv & 4 $\times$ 4 $\times$ 64 & 3 & 2 & BatchNorm+SiLU\\
        & Conv & 4 $\times$ 4 $\times$ 64 & 3 & 1 & BatchNorm+SiLU\\
        & Reshape & 1 $\times$ 1 $\times$ 1024 & - & - & -\\
        & Linear & 1 $\times$ 1 $\times$ 256 & - & - & LayerNorm+SiLU\\
        & Linear & 1 $\times$ 1 $\times$ 2 $\cdot$\verb|num_latents| & - & - & -\\
        \hline
        \multirow{14}*{Decoder}
        & Linear & 1 $\times$ 1 $\times$ 256 & - & - & LayerNorm+SiLU\\
        & Linear & 1 $\times$ 1 $\times$ 1024 & - & - & - \\
        & Reshape & 4 $\times$ 4 $\times$ 1024 & - & - & - \\
        & Upsample & 8 $\times$ 8 $\times$ 64 & - & - & - \\
        & ResidualBlock & 8 $\times$ 8 $\times$ 64 & 3 & 1 & - \\
        & Upsample & 16 $\times$ 16 $\times$ 64 & - & - & - \\
        & ResidualBlock & 16 $\times$ 16 $\times$ 64 & 3 & 1 & - \\
        & Upsample & 32 $\times$ 32 $\times$ 64 & - & - & - \\
        & ResidualBlock & 32 $\times$ 32 $\times$ 64 & 3 & 1 & - \\
        & Upsample & 64 $\times$ 64 $\times$ 64 & - & - & - \\
        & ResidualBlock & 64 $\times$ 64 $\times$ 64 & 3 & 1 & - \\
        & Pre-Activations & 64 $\times$ 64 $\times$ 64 & - & - & BatchNorm+SiLU \\
        & Conv & 64 $\times$ 64 $\times$ 64 & 1 & 1 & BatchNorm+SiLU \\
        & Conv & 64 $\times$ 64 $\times$ 3 & 1 & 1 & Tanh \\
        \bottomrule
    \end{tabular}
\end{table}

In this setting we use as the first part of our pipeline both CITRIS \citep{LippeMLACG22}, in particular CITRIS-VAE,
and iVAE \citep{khemakhem2020vaeica}. In both models, the encoder and decoder architecture
are set to be the same. The encoder is a convolutional neural network, which
outputs two parameters per latent variable. These will be the mean and the log
of the standard deviation for the normal distribution that models the latent
variable. The decoder uses bilinear upsampling and residual blocks to
reconstruct the image. The 
full architecture is described in Table~\ref{tbl:citris_architecture}.

As an additional optimization step, an autoencoder is pre-trained to
map the high-dimensional images to lower-dimensional feature vectors, but
without enforcing disentanglement. This is done separately from the main training
procedure, as \citet{LippeMLACG22} mention that this improves performance. 
During training a small amount of Gaussian
noise is added to the encodings to prevent a collapse of the encoding
distribution. No prior is enforced for this encoder. This autoencoder
has
$2$ ResidualBlocks instead of $1$ per resolution in the decoder part. The
training hyperparameters are described in Table~\ref{tbl:ae-training} and the
autoencoder is trained using the MSE reconstruction loss.

\paragraph{CITRIS-VAE} 
CITRIS \citep{LippeMLACG22} allows for multidimensional causal variables. The number of latent variables
is allowed to be bigger than the number of causal variables $d$ , but the model subdivides the latent variables into
$d$ possibly uneven blocks that get mapped to the causal variables. This means that an assignment $\psi\colon \{1, \ldots, k\} \to \{1, \ldots,d\} $ is learned between
the learned latent variables and the learned causal variables, so that
multiple latent variables can represent one causal variables. This is done by assuming that
each $\psi(i)$ follows a Gumbel-Softmax distribution and we learn the continuous parameters
that govern these distributions. CITRIS-VAE optimizes the following objective:
\begin{align*}
    \text{ELBO}(\theta, \phi, \gamma)=
    &-\bE_{z^{t+1}\sim q_{\theta}(\cdot  \mid x^{t+1})}\left[ \log p_{\theta}(x^{t+1}  \mid z^{t+1}) \right] \\
    &+
    \bE_{
        \substack{
            z^{t} \sim q_{\theta}(\cdot  \mid x^{t})\\
            \pi \sim \text{GS}(\gamma)
        }}\left[ 
            \sum_{i=1}^{d}
            \infdivKL{q_{\theta}(z_{\psi(i)}^{t+1}  \mid x^{t+1})}{p_{\phi}(z_{\psi(i)}^{t+1} \mid z^{t}, I^{t+1}_i)}
        \right] 
.\end{align*}
Here $p_{\theta}$ models the encoder, $q_{\theta}$ the decoder,
$p_{\varphi}(z^{t+1} \mid z^{t}, I^{t+1})$ the transition prior and GS
is the Gumbel-Softmax distribution
of the causal variables between time steps given the intervention targets.
Finally, $\pi$ is the target assignment between learned encoding variables
and the causal variables. 
During
training a latent-to-causal variable  assignment is sampled, while during
testing the argmax is used. 
The transition prior $p_{\phi}$ is learned by an autoregressive model, 
which for each $z_{\psi(i)}^{t+1}$ takes $z^{t}, I^{t+1}_{i} $and $z^{t+1}$ 
as inputs and outputs a Gaussian random variable. The autoregressive 
model follows a MADE architecture \citep{KhajenezhadMB21}, 
with $16$ neurons per layer for each encoding, and the input
to these neurons are the features of all previous encodings. 
The prior is $2$ layers deep, and uses SiLU activation functions.
Finally, a small network is trained to predict the intervention 
targets, given $z^{t}$ and $z^{t+1}_{\psi(i)}$ for each $i=1,\ldots k$.

\paragraph{iVAE} To adapt the iVAE model \citep{khemakhem2020vaeica} for this setting, the auxiliary 
variable $u$ is given by the previous observation $x^{t}$ and intervention
targets $I^{t+1}$. Another alteration that is made, is that the prior
with the iVAE model only conditions on $(x^{t}, I^{t+1})$. The main 
difference between iVAE and the CITRIS-VAE is the structure of the
prior $p(z^{t+1} \mid z^t, I^{t+1})$. Another difference is that
no target assignment is learned during the training, but only after. 
For  iVAE a $2$-layer MLP with hidden dimensionality of $128$ is used
for the transition prior.

All the hyperparameters for training CITRIS-VAE and iVAE are reported in Table~\ref{tbl:citris-training}.

\begin{table}
\parbox{0.45\linewidth}{
    \caption{The hyperparameters used for the training of autoencoder used as preprocessing by both
    the CITRIS-VAE and iVAE.}
    \label{tbl:ae-training}
    \centering 
    \small
    \begin{tabular}{m{3cm}m{2.2cm}}
        \toprule
        \textbf{Hyperparameter} & \textbf{Value}\\
        \bottomrule
        Batch Size & 512 \\
        Optimizer & Adam \citep{KingmaB14}\\
        Learning rate & 1e-3\\
        Learning rate scheduler & Cosine Warmup (100 steps)\\
        Number of latents & 32\\
        Gaussian noise $\sigma$ & 0.05\\
        Number of epochs & 1000\\
        \bottomrule
    \end{tabular}
        }
        \hfill
    \parbox{0.45\linewidth}{
    \caption{The hyperparameters used for the training of both the CITRIS-VAE and iVAE 
    models.}
    \label{tbl:citris-training}
    \centering 
    \small
    \begin{tabular}{m{4cm}m{2.2cm}}
        \toprule
        \textbf{Hyperparameter} & \textbf{Value}\\
        \bottomrule
        Batch Size & 512 \\
        Optimizer & Adam \citep{KingmaB14}\\
        Learning rate & 1e-3\\
        Learning rate scheduler & Cosine Warmup (100 steps)\\
        KL divergence factor $\beta$ & 1.0\\
        KL divergence factor $\psi_{0}(\lambda)$ & 0.01\\
        Number of latents & 32\\
        Number of epochs & 600\\
        Target classifier weight & 2.0\\
        Gumbe Softmax temperature & 2.0\\ 
        \bottomrule
    \end{tabular}
}
\end{table}

\subsection{Concept-Based Models}
\label{app:concept-based-models}
Concept-bottleneck models
were introduced and popularized by \citet{koh2020conceptbottleneck}.
They offer a deep learning architecture that is inherently interpretable, by letting the
input go through a concept layer. Each node in that layer indicates either
the value of a predefined concept when using continuous concepts if the concept is present or not
when the concepts are binary.
Then, a linear layer from these concepts to the final target layer is applied. 
Formally, given the observational data $\{(x_i, c_i, y_i)\}_{i=1}^{n}$,
where $x_i \in \mathcal{X}$, $c_i \in  \mathcal{G}_i$, $y_i \in \mathbb{R} $,
a concept encoder $p_\theta(c  \mid x)$ and a label predictor $p_{\phi}(y  \mid c)$
are learned. The total prediction is then performed by 
combining these $2$ probabilities. Different choices are possible for
this combination, which lead
to different concept-based models. We will discuss the Concept-Bottleneck
Model (CBM) 
\citep{koh2020conceptbottleneck}, the Concept-Embedding Model (CEM)
\citep{zarlenga2022concept} and the HardCBM \citep{havasi2022addressing}, which
are all three used in the binary versions of our experiments. 
For all three models, we use the same
encoders as in the VAE models, to make the comparison fair. 

The linear layer allows the user to construct explanations for each prediction. 
When a prediction is made, we can look at the concept activations,
$\hat{c}_i$, and the corresponding weights in the final layer $\psi$. 
An explanation is given by the set 
\begin{align*}
    \{(\psi_i, \hat{c}_i)  \mid  i \in \{1, \ldots, d\} \} 
.\end{align*}
The interpretation is that $\psi_i$ tells us how important that concept is in general, 
and $\hat{c}_i$ tells us how active that concept is for this particular prediction. 

\paragraph{CBM} In the Concept-Bottleneck Model (CBM) \citep{koh2020conceptbottleneck} the output of the concept layer will be denoted by $g$
and the function from the concepts to the labels will be $f$. The output of $g$ are the logits 
of that concept being $1$, meaning $\sigma(g(x)_{i}) = p_{\theta}(c_i = 1  \mid x)$ 
for $i=1, \ldots, d$, where $\sigma$ is the sigmoid function.
The function $f$ is a linear combination of the probabilities for all the concepts
so together $f(\sigma(g(x))) = \psi^{\top}\sigma(g(x)) = p_{\theta, \psi}(y  \mid x)$.

\paragraph{CEM} Concept-Embedding Model (CEM)
\citep{zarlenga2022concept} show that the performance of CBMs can be improved by letting the 
concepts be encoded in a higher dimensional space, i.e., $\hat{c}_i \in \mathbb{R}^{k}$ 
instead of $\hat{c}_i \in \mathbb{R}$. These embeddings are then concatenated
and fed to the label predictor. For each concept, 
a function $s_i$ is learned that predicts the concept label from 
these embeddings. These predictions can then be used for explanations, but
the embeddings are passed  to the label predictor.

\paragraph{HardCBM} CBMs can also be susceptible to concept leakage, when unrelated information is encoded in each concept. HardCBM \citep{havasi2022addressing} propose a way to address this issue. 
In particular, to discourage information from other sources leaking
through the concepts, the HardCBM label predictor only takes binary 
values from the concept layer as input. To make a prediction, 
we have to marginalize over the possible concepts,
$p_{\theta, \psi}(y  \mid  x) = \mathbb{E}_{c \sim p_{\theta}(c|x) }\left[ p_{\phi}(y  \mid  c) \right] $.
To model possible dependencies between the concepts in the distribution of $X  \mid C$, 
the conditional distribution is decomposed as
\begin{align*}
    p_{\theta}(c  \mid x) = \prod_{i=1}^{d}p(c_i  \mid x, c_1, \ldots c_{i-1})
.\end{align*}
In the HardCBM such a decomposition can be achieved by using an autoregressive architecture. 

In all the experiments, each model is 
trained for $100$ epochs, with batch sizes of $256$ and optimised by
Adam with a learning rate of $1$e-$3$.

\subsection{Performance metrics}
To assess the performance of our estimator we report the mean permutation error of the
estimated permutation (MPE), together with the execution time. 
Whenever the concepts are continuous, we also report the $R^2$-score. 
Some of the alignment learning baselines are
created by regressing every input variable onto every output
variable and using the individual $R^2$-scores to extract an
alignment. In those cases, we first match the encodings and ground truth causal variables
according to the highest $R^2$-scores. The $R^2$-score that is reported is
the average of the $R^2$-scores that are
chosen to be matched. 
This is also referred to as the $R^2$-score on the diagonal. 
This gives 
\begin{align*}
    R^2 = \frac{1}{d}\sum_{i=1}^{d} 
    \left(1 - \frac{
            \|\Hvec_i - \alpha_{i}\left(\Mvec_{\hpi(i)}\right)\|^2
        }{
            \|\Hvec_i - \overline{\Hvec_i}\|^2
        }\right),
    \qquad
    \overline{\Hvec_i} = \frac{1}{n}\sum_{\ell=1}^{n} C^{(\ell)}_i
.\end{align*}
Here, $\alpha$ and $\hpi$ are the estimated alignment map and 
estimated permutation. The alignment map is applied
to the whole vector  $\Mvec_{j}$. As in the main text, $\Hvec$ and $\Mvec$ are
matrices where the rows are the data samples and the columns the concepts and encodings
respectively. 

In the downstream classification experiments we also report various other
metrics. The mean accuracy of all the concepts and the accuracy of the
downstream label prediction is reported. Furthermore, two additional metrics
for concept-based models are reported, the OIS and NIS metric. The exact
definitions can be found in \citet{zarlenga2023towards}. The OIS metric measures
how well concept $i$ can be used to reconstruct concept $j$. The NIS metric
measures how much information about concept  $i$ is contained in the other
concepts jointly. A low OIS metric, but high NIS metric
indicates that it is not possible to
reconstruct one concept from any of the others individually, but combining the concept embeddings
would allow to reconstruct the concept from the others.

\paragraph{Estimator settings} We use various settings of our estimator, 
to assess if there are particular advantages for certain versions. The
versions that we use are
\begin{itemize}
    \item \textbf{Linear}, no feature map is applied. 
    \item \textbf{Random Fourier Features}, we sample $8$ random Fourier features
        from the RBF kernel. 
    \item \textbf{Spline}, we calculate cubic spline features with  $4$ knots. 
    \item \textbf{Laplacian}, we use the Laplacian kernel with 
        $\min\{n, 20\}$ components. If the number of components
        is less than the number of datapoints a Nystr\"om sampling procedure is used. 
    \item \textbf{Two stage} we apply a two stage approach, where 
        we use $20\%$ of the data to estimate the permutation
        using no additional features and then use the rest of the 
        data to perform ridge regression with cubic spline features 
        using $4$ knots.
\end{itemize}
We define an array of regularization parameters and report the results 
of the best choice for each $n$. The parameters that we consider are
$\lambda \in \{0.0001, 0.0005, 0.001, 0.005, 0.01, 0.05, 0.1, 0.2\} $.

\paragraph{Alignment Learning Baselines.} We estimate the permutations using the
Pearson and Spearman correlations. We also add another baseline for the experiments with the
encodings learned by a VAE. This baseline 
is given by training multiple neural networks. Each neural 
network takes an individual encoding as input and tries to predict
all the causal variables. The individual $R^2$-scores are used to 
construct a matching again. This neural network is a $2$-layer MLP
with $128$ hidden nodes in each layer and $\tanh$ 
is used as an activation function. The network is trained for $100$ epochs
with Adam and a learning rate of 4e-3. We considered using $32$ 
and $64$ hidden nodes per layer, but concluded that $128$ nodes per layer offers
the best baseline. 

In the Temporal Causal3DIdent dataset experiments, both VAE models
learn groups of latent variables that are matched with a
causal variable. To use the Pearson and Spearman correlations 
in this case, we first sum the latent variables in the groups and then 
calculate the correlation coefficients. 

The baselines in the binary versions of the experiment in the two 
DL datasets are given by the three concept-based models discussed in
Section~\ref{app:concept-based-models}.

\section{Additional Results}
\label{app:addition_results}
In this section, we show all the results obtained in the experiments with the 
Toy Dataset, Action/Temporal Sparsity datasets and Temporal Causal3DIdent
dataset. 

\subsection{Toy Dataset}
\label{app:additional-toy}

\paragraph{Continuous Concepts.} 
We provide plots for experiments performed with linear features, spline
features, random Fourier features and kernels. The concepts are
continuous in these plots. All results are depicted in 
Figures~\ref{fig:spline-perm_error}--\ref{fig:kernel-time}.

In each version of our estimator we see that the MPEs
are good, especially in the misspecified setting. It is interesting
to note that the estimator does not perform well in the wellspecified
case, when the regularization parameter is not tuned correctly. 
This can be explained by the non-invertability of the functions
in this setting. This makes the identification of each matching more 
noisy and difficult, which can also be noted by the fact that the 
Pearson and Spearman correlation approaches are not able to find
the correct permutation, while our estimator is able to do it correctly.

The non-invertibility of the functions in the wellspecified setting
also explains why the $R^2$ scores in that setting are 
worse than the $R^2$-scores in the misspecified setting. Another reason for that observation is that
the norms of the true parameters have to be quite large, to strictly 
adhere to the assumptions of our theoretical results. This increases
the variability of the output data by a large amount and makes
regression more difficult. 

Finally, we also see that our estimator does work well, even when the
correlation is high, but the regularization parameter has to be tuned
correctly. This does come at the cost of an increased computation time.

\paragraph{Binary Concepts Ablation.} We also perform an ablation study, with binary concepts. 
How the continuous scores are binarized is explained in Appendix~\ref{app:synth_experiments}
and sampling the diffeomorphisms from the functions depicted in Figure~\ref{fig:diffeomorphisms}. 
In these experiments we added the CBM \citep{koh2020conceptbottleneck}, CEM \citep{zarlenga2022concept} and HardCBM \citep{havasi2022addressing} as baselines. We report 
the concept accuracy, label accuracy, OIS and NIS scores \citep{zarlenga2023towards} and execution times. 
The results are depicted in Figures~\ref{fig:ablation-concept-acc}--\ref{fig:ablation-time}.
The ablation comes from the fact that the correlation in the training set differs
from the correlation in the test set; 
in the former it is $\rho=0.5$ and in the latter it is $\rho=0$. This is
done to mimic the situation where a classifier could pick up on spurious correlations in the train
set that are not present in the test set. We see that our estimator performs well in this
case in terms of concept accuracy and label accuracy, and better than the concept based-models whenever the regularization parameter is set
to be not too large. Our estimator performs similarly as the concept-based models in terms of OIS, but consistently
scores better when considering the NIS metric. 

It is interesting to note that CEM \citep{zarlenga2022concept} performs the worst in terms of the NIS
and OIS metrics. An explanation for this is that it has
the capacity to store more information about the other concepts in the embeddings. 
CBM \citep{koh2020conceptbottleneck} and HardCBM \citep{havasi2022addressing} suffer less from this problem, 
as the concepts are stored as one-dimensional objects. All concept-based models perform similarly
in terms of label and concept accuracy. 

One caveat is that the training times of our estimator do increase faster than the training times of the 
concept-based models for increasing number of dimensions.
We think this can be attributed to the fact that to perform the logistic
alignment learning procedure, a loop over all dimensions has to be done. In each step 
one logistic Group Lasso regression is performed.
In the continuous case, this was not necessary, as we could see the output as multi-dimensional, 
and all regressions could be done in parallel. In principle this
should be possible for the logistic Group Lasso as well, but implementing this parallel computation
was beyond the scope of this paper.
The implementation of the logistic Group Lasso we use, does not have this feature for 
multiple labels. 

\paragraph{Mixtures of Causal Variables.} We conducted an ablation with the
synthetic toy dataset to evaluate performance when we have mixtures of causal
variables as inputs to the alignment, which simulates the case in which CRL methods might not provide complete disentanglement, but only partial disentanglement. In these experiments, we set the dimension $d$ 
to an even
number and then randomly created $d/2$ pairs of $M_j$ variables. For
each pair $(M_i, M_j)$, we
then
created a new mixture variable $M_j' =
(1 - a) M_i + a M_j$,  where $a$ represents the mixing parameter. We
then used these $M_j'$ variables as input to the alignment estimation
step. We ran 10 random repetitions for each setting and show the 
permutation error for the original variables in Figure~\ref{fig:mixed-perm-error}. We additionally show a paired mean
permutation error, where we consider the estimated permutation correct if the
matching identifies the pairs correctly. These results are shown in
Figure~\ref{fig:mixed-paired-perm-error}.

What we see is that our estimator keeps performing perfectly up until 
slightly below $a = 1 / 2$, but afterwards does not get the permutation
right. However, we do see that the pairs are matched correctly for almost
all values of $a$. 

\subsection{Action/Temporal Sparsity Datasets and Temporal Causal3DIdent Dataset}
\label{app:additional-dl}

\paragraph{Continuous Concepts.}
Here, we report all results obtained in the Action/Temporal Sparsity datasets
and Temporal Causal3DIdent dataset. The MPEs are reported in Table~\ref{tbl:perm-error-dms},
the $R^2$-scores in Table~\ref{tbl:r2-dms} and the execution times
are plotted in Figure~\ref{fig:times-both}. 

In terms of MPEs, we always see that a version of our estimator performs the best
or as good as the best in each of the datasets for each number of training data samples
considered. For the $R^2$scores, 
we see that we perform well in some cases in the low data regime, but when
using all the data available, the neural network often performs the best. 
This does come at a computational cost, where the neural network approach
requires two to three orders of magnitude more computation time. 

It is interesting to note that the estimator typically works better 
for the more advanced models developed. This can be explained by the fact
that these models achieve a better disentanglement, which should make it
easier to find the correct matching between the encodings and the 
causal variables. 

\paragraph{Binary Concepts.} We create binary concepts and labels in a similar
fashion as detailed in Appendix~\ref{app:synth_experiments}. The number of active columns
selected is now fixed at $k=3$ and at least two of them have to be non-zero for the label to
be one. The permutation errors are also reported in this case and can be found
in Table~\ref{tbl:perm-error-bin}. The baselines are CBM \citep{koh2020conceptbottleneck}, CEM \citep{zarlenga2022concept} and HardCBM \citep{havasi2022addressing}.
The concept accuracy and label accuracy are reported in Table~\ref{tbl:concept-acc-roc}, 
and the OIS- and NIS-metrics in Table~\ref{tbl:ois-nis}. The execution times
are plotted in Figure~\ref{fig:times-both-bin}.

We see that the standard concept-based models perform worse with respect to all metrics with 
only a few labels. However, in the action and temporal datasets, they quickly 
perform well in terms of concept accuracy and OIS. Our estimator performs
especially well when using non-tabular data, such as in the Temporal Causal3Dident
dataset. This shows the added benefit of the CRL phase of our framework. 

\begin{figure}[h]
    \def\incFigure#1#2#3{%
        \node{\scalebox{1}{\includegraphics[
            trim={{#1} 0 {#2} 0.2cm},
            clip
        ]{./figs/spline/perm_error/#3.pdf}}};
    }
    \centerfloat
    \resizebox{1.2\textwidth}{!}{
    \begin{tikzpicture} 
        \node (legend) {\includegraphics[scale=2.5]{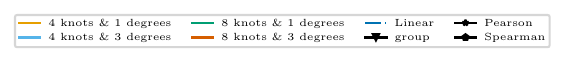}};
        \matrix[below=-0.2cm of legend.south, 
            row sep=-0.1cm, 
            column sep=-0.2cm,
            ampersand replacement=\&] (plotmatrix) {
            \incFigure{0.1cm}{0cm}{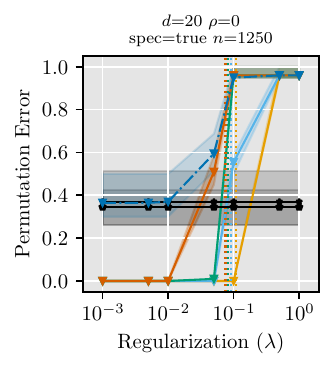}\&[-0.22cm]
            \incFigure{1.38cm}{0cm}{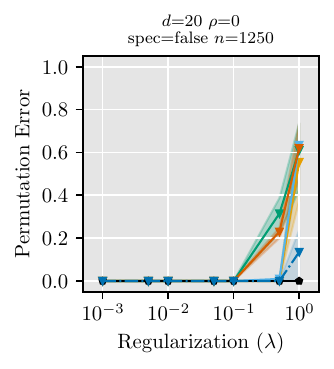}\&
            \incFigure{0.55cm}{0cm}{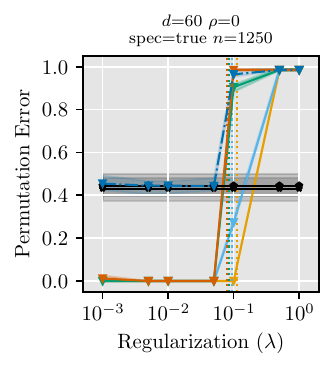}\&[-0.215cm]
            \incFigure{1.38cm}{0cm}{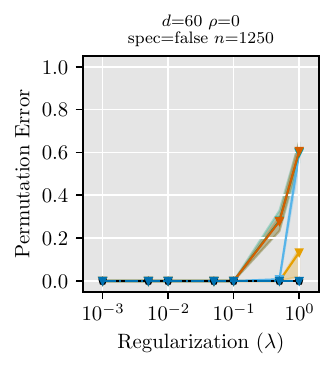}\&
            \incFigure{0.55cm}{0cm}{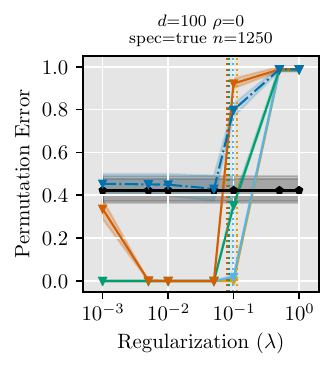}\&[-0.215cm]
            \incFigure{1.38cm}{0cm}{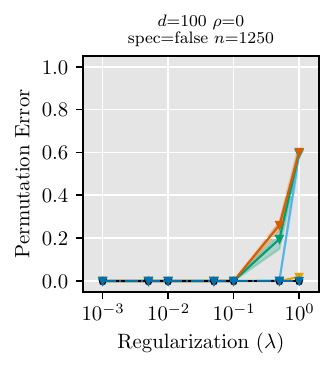}\&\\

            \incFigure{0.1cm}{0cm}{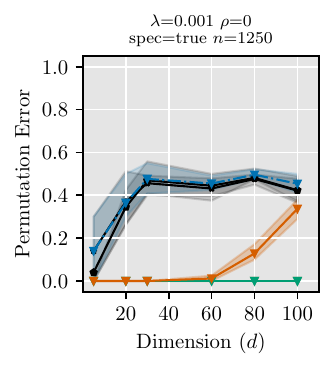}\&
            \incFigure{1.38cm}{0cm}{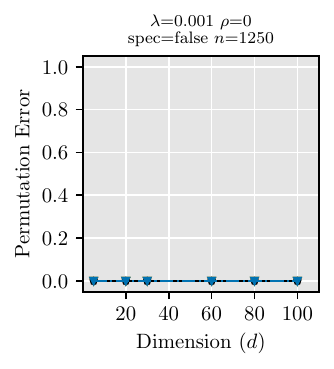}\&
            \incFigure{0.55cm}{0cm}{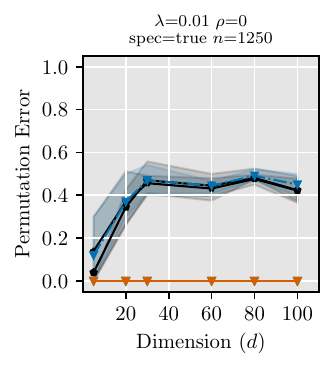}\&
            \incFigure{1.38cm}{0cm}{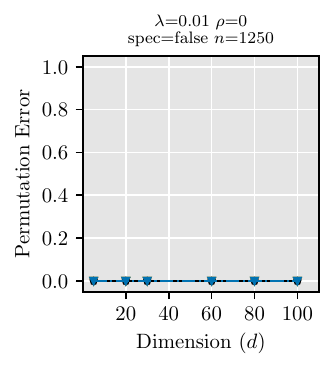}\&
            \incFigure{0.55cm}{0cm}{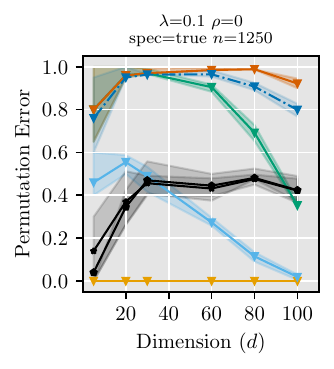}\&
            \incFigure{1.38cm}{0cm}{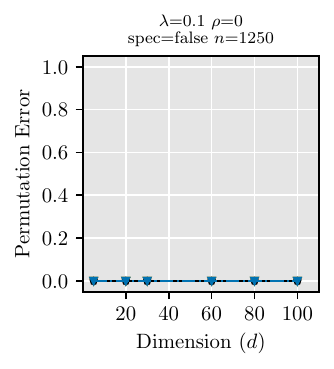}\&\\

            \incFigure{0.1cm}{0cm}{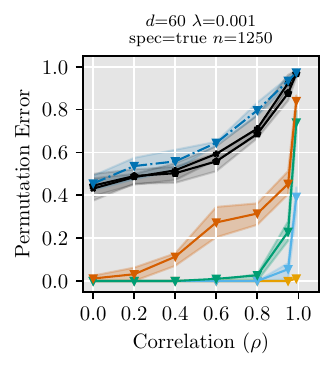}\&
            \incFigure{1.38cm}{0cm}{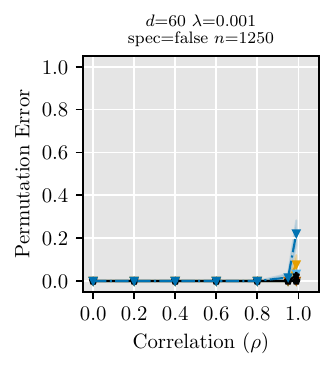}\&
            \incFigure{0.55cm}{0cm}{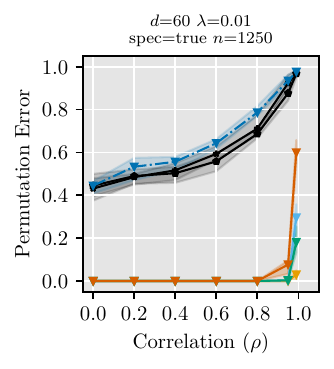}\&
            \incFigure{1.38cm}{0cm}{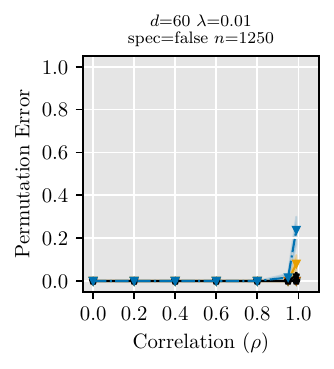}\&
            \incFigure{0.55cm}{0cm}{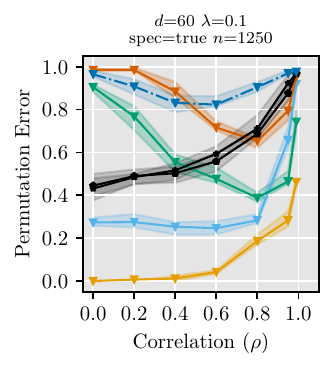}\&
            \incFigure{1.38cm}{0cm}{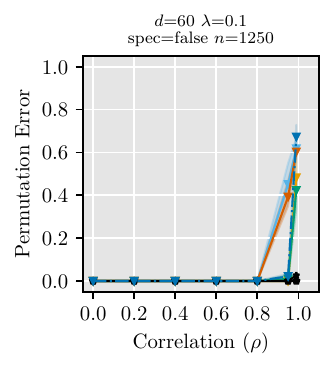}\&\\

            \incFigure{0.1cm}{0cm}{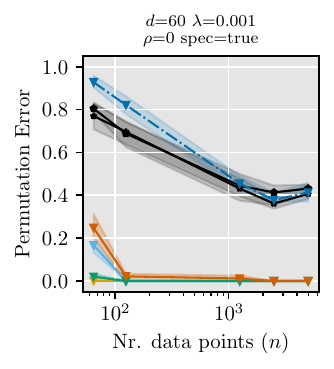}\&
            \incFigure{1.38cm}{0cm}{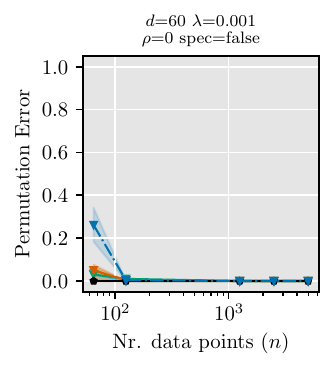}\&
            \incFigure{0.55cm}{0cm}{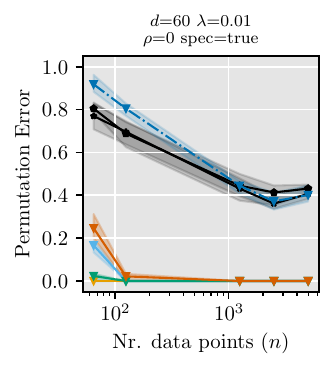}\&
            \incFigure{1.38cm}{0cm}{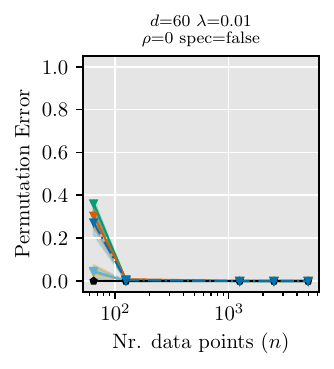}\&
            \incFigure{0.55cm}{0cm}{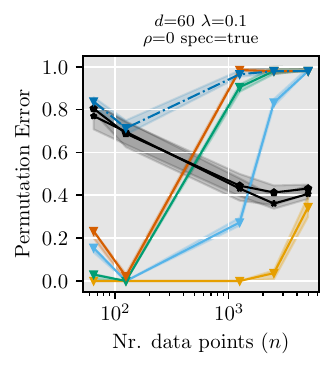}\&
            \incFigure{1.38cm}{0cm}{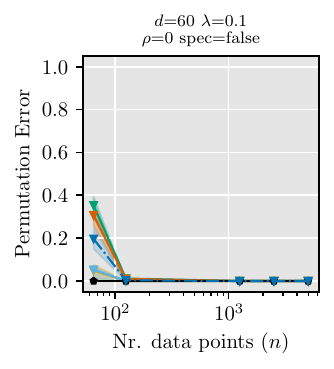}\&\\
        };
    \end{tikzpicture}
    }
    \caption{Permutation Error using Spline Features for all parameters
        considered in the experiments. Each plot is paired, where the left is
        wellspecified case and the right is the misspecified case.}
    \label{fig:spline-perm_error}
\end{figure}
\begin{figure}[h]
    \def\incFigure#1#2#3{%
        \node{\scalebox{1}{\includegraphics[
            trim={{#1} 0 {#2} 0.2cm},
            clip
        ]{./figs/spline/r2/#3.pdf}}};
    }
    \centerfloat
    \resizebox{1.2\textwidth}{!}{
    \begin{tikzpicture} 
        \node (legend) {\includegraphics[scale=2.5]{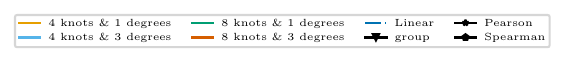}};
        \matrix[below=-0.2cm of legend.south, 
            row sep=-0.1cm, 
            column sep=-0.2cm,
            ampersand replacement=\&] (plotmatrix) {
            \incFigure{0.1cm}{0cm}{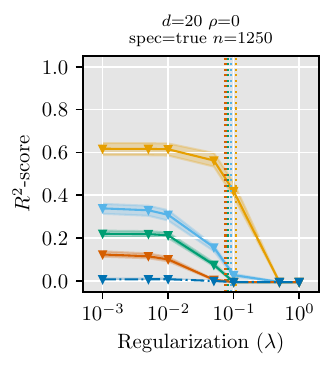}\&[-0.22cm]
            \incFigure{1.38cm}{0cm}{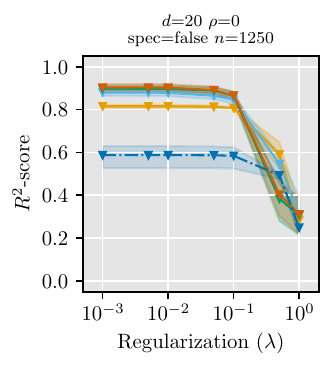}\&
            \incFigure{0.55cm}{0cm}{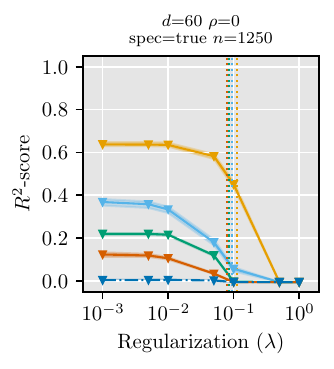}\&[-0.215cm]
            \incFigure{1.38cm}{0cm}{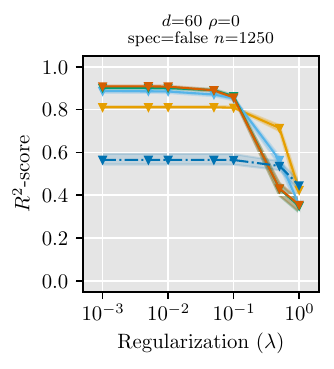}\&
            \incFigure{0.55cm}{0cm}{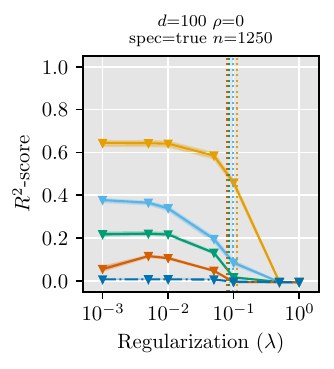}\&[-0.215cm]
            \incFigure{1.38cm}{0cm}{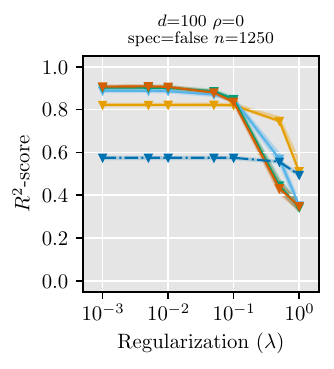}\&\\

            \incFigure{0.1cm}{0cm}{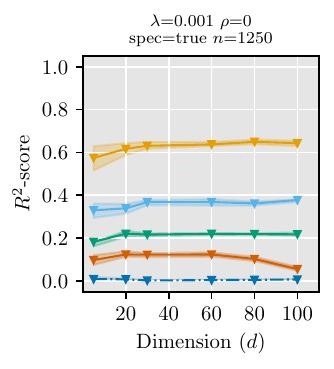}\&
            \incFigure{1.38cm}{0cm}{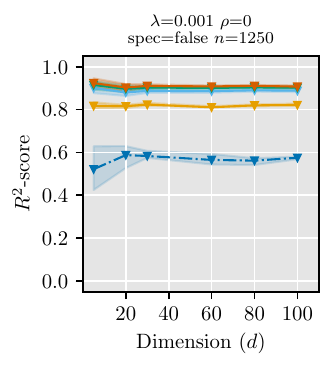}\&
            \incFigure{0.55cm}{0cm}{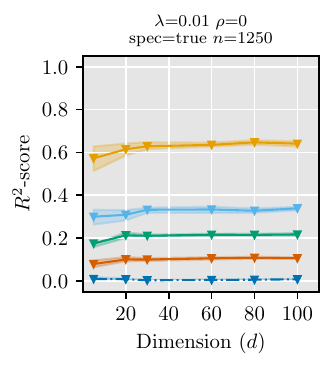}\&
            \incFigure{1.38cm}{0cm}{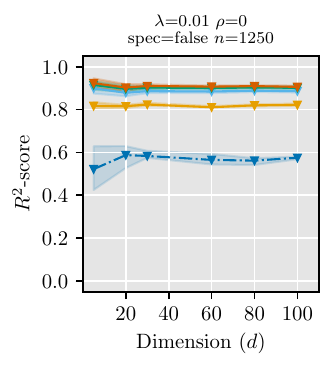}\&
            \incFigure{0.55cm}{0cm}{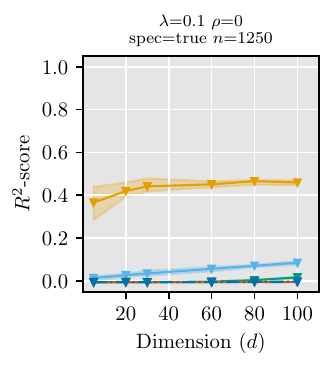}\&
            \incFigure{1.38cm}{0cm}{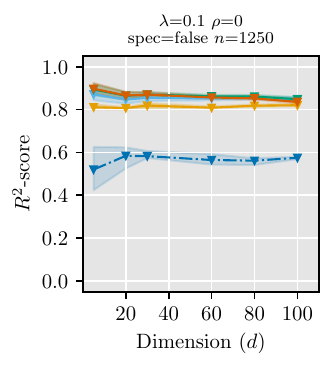}\&\\

            \incFigure{0.1cm}{0cm}{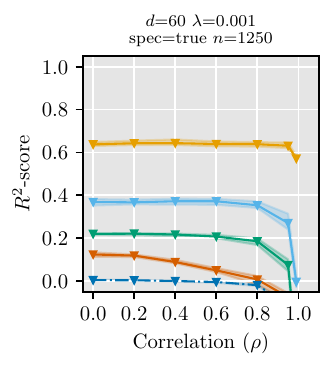}\&
            \incFigure{1.38cm}{0cm}{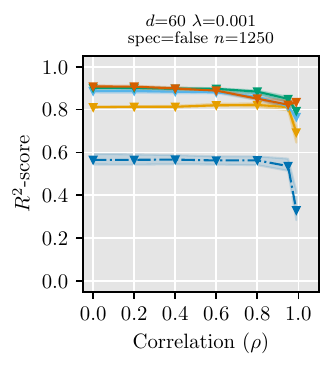}\&
            \incFigure{0.55cm}{0cm}{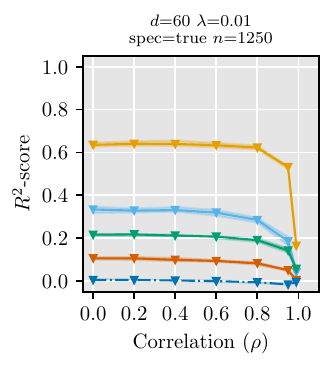}\&
            \incFigure{1.38cm}{0cm}{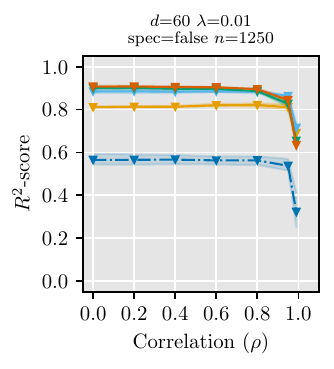}\&
            \incFigure{0.55cm}{0cm}{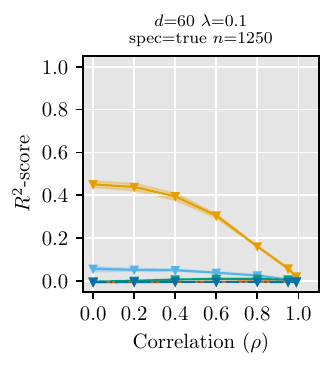}\&
            \incFigure{1.38cm}{0cm}{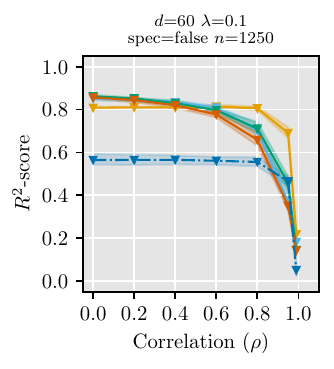}\&\\

            \incFigure{0.1cm}{0cm}{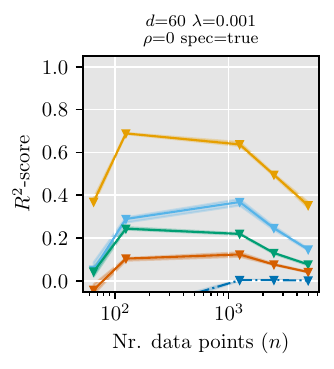}\&
            \incFigure{1.38cm}{0cm}{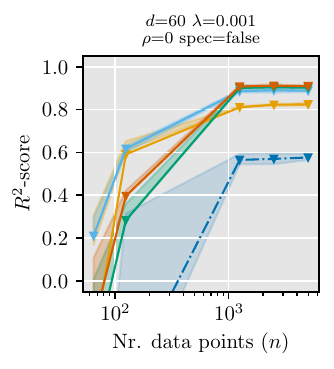}\&
            \incFigure{0.55cm}{0cm}{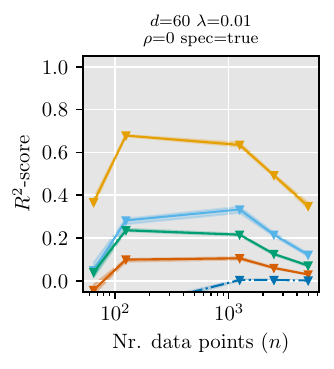}\&
            \incFigure{1.38cm}{0cm}{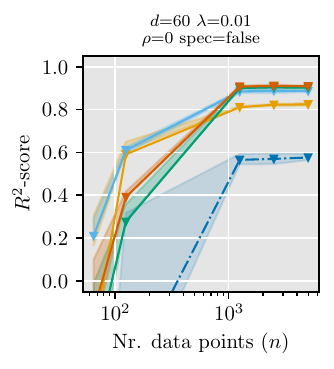}\&
            \incFigure{0.55cm}{0cm}{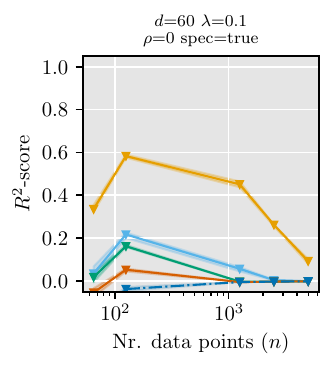}\&
            \incFigure{1.38cm}{0cm}{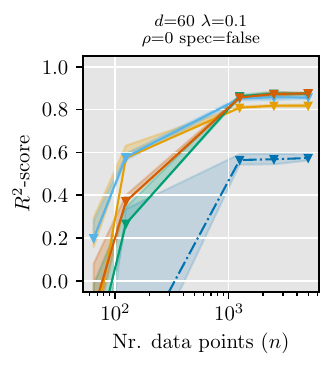}\&\\
        };
    \end{tikzpicture}
    }
    \caption{$R^2$-score on the diagonal using Spline Features for all parameters
        considered in the experiments. Each plot is paired, where the left is
        wellspecified case and the right is the misspecified case.}
    \label{fig:spline-r2}
\end{figure}
\begin{figure}[h]
\def\incFigure#1#2#3{%
    \node{\scalebox{1}{\includegraphics[
        trim={{#1} 0 {#2} 0.2cm},
        clip
    ]{./figs/spline/time/#3.pdf}}};
}
    \centerfloat
    \resizebox{1.2\textwidth}{!}{
\begin{tikzpicture} 
    \node (legend) {\includegraphics[scale=2.5]{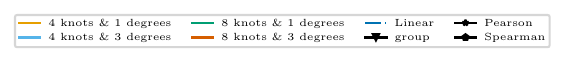}};
    \matrix[below=-0.2cm of legend.south, 
        row sep=-0.1cm, 
        column sep=-0.2cm,
        ampersand replacement=\&] (plotmatrix) {
            \incFigure{0.25cm}{0cm}{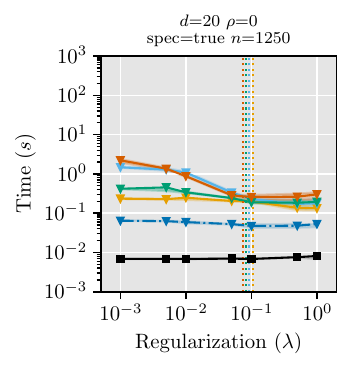}\&[-0.22cm]
            \incFigure{1.68cm}{0cm}{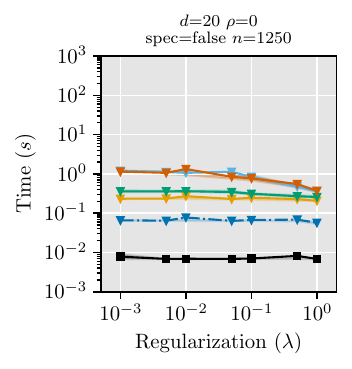}\&
            \incFigure{0.70cm}{0cm}{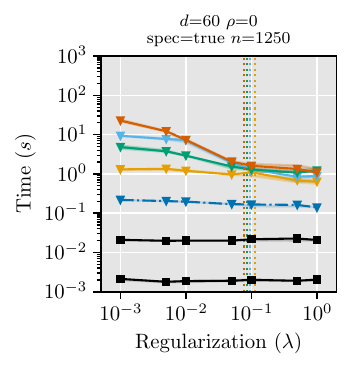}\&[-0.215cm]
            \incFigure{1.68cm}{0cm}{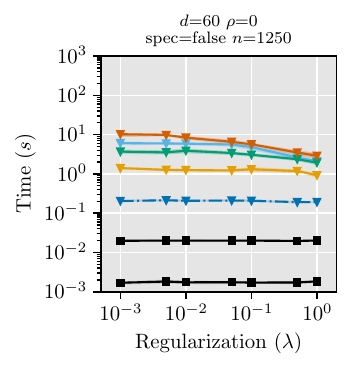}\&
            \incFigure{0.70cm}{0cm}{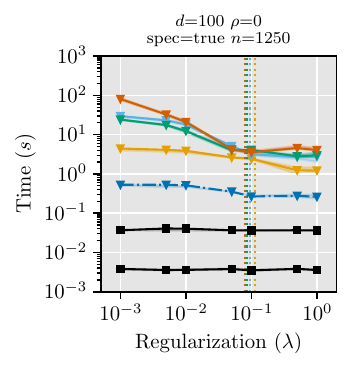}\&[-0.215cm]
            \incFigure{1.68cm}{0cm}{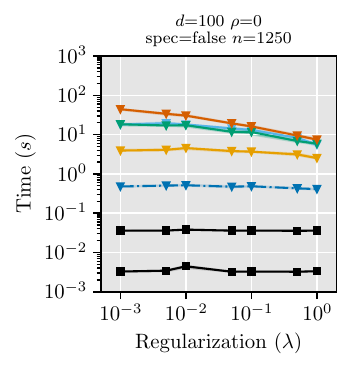}\&\\

            \incFigure{0.25cm}{0cm}{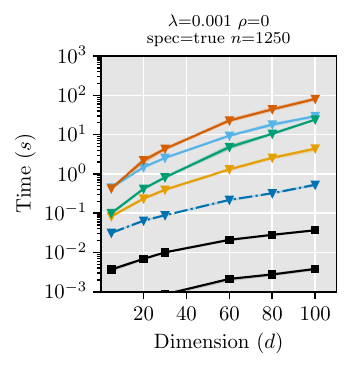}\&
            \incFigure{1.68cm}{0cm}{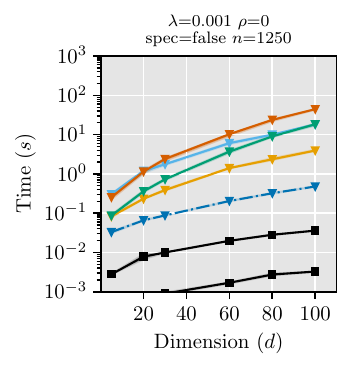}\&
            \incFigure{0.70cm}{0cm}{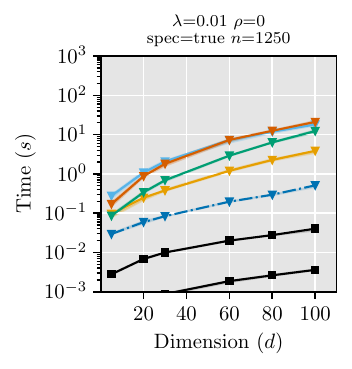}\&
            \incFigure{1.68cm}{0cm}{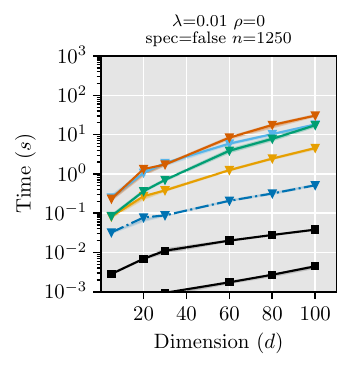}\&
            \incFigure{0.70cm}{0cm}{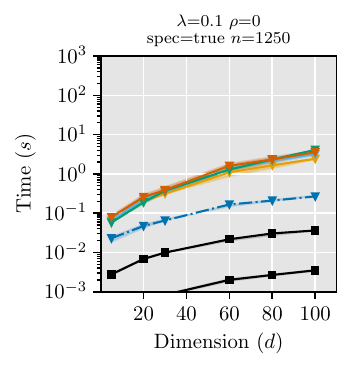}\&
            \incFigure{1.68cm}{0cm}{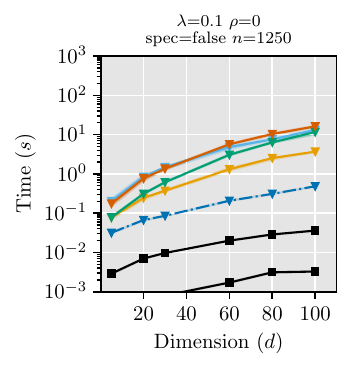}\&\\

            \incFigure{0.25cm}{0cm}{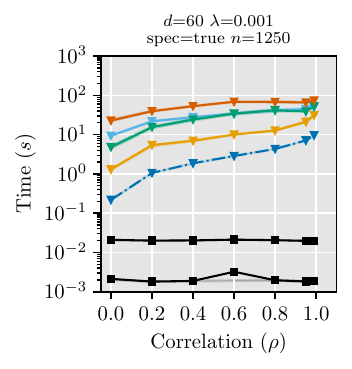}\&
            \incFigure{1.68cm}{0cm}{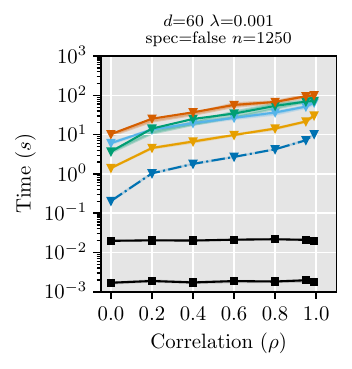}\&
            \incFigure{0.70cm}{0cm}{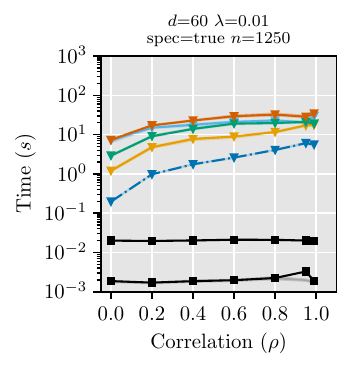}\&
            \incFigure{1.68cm}{0cm}{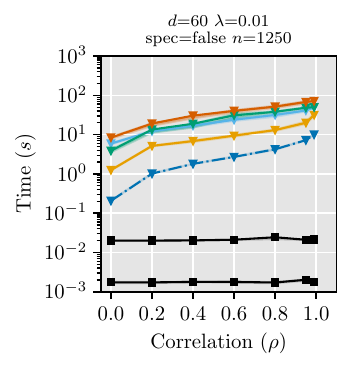}\&
            \incFigure{0.70cm}{0cm}{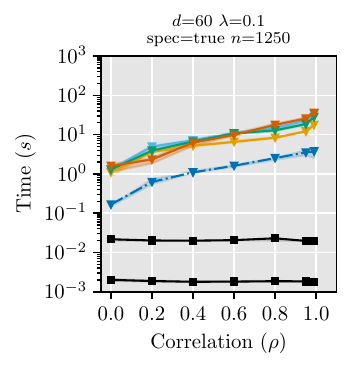}\&
            \incFigure{1.68cm}{0cm}{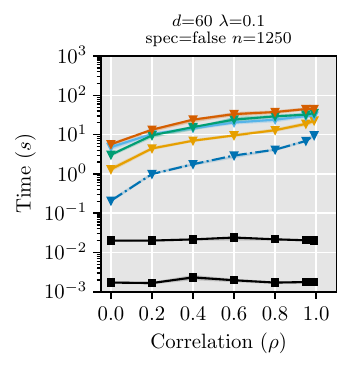}\&\\

            \incFigure{0.25cm}{0cm}{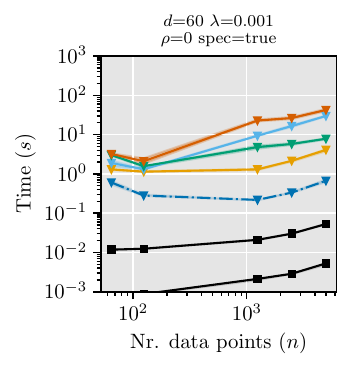}\&
            \incFigure{1.68cm}{0cm}{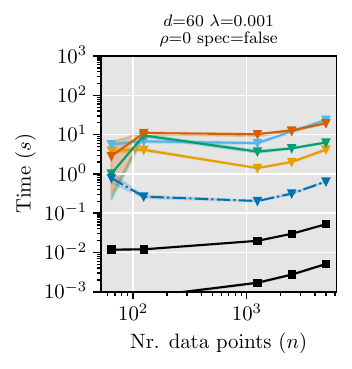}\&
            \incFigure{0.70cm}{0cm}{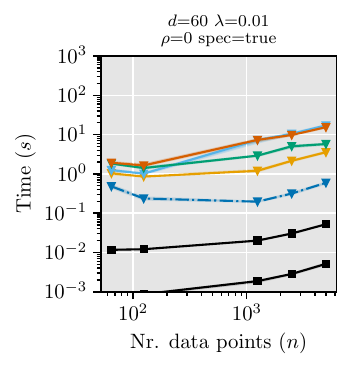}\&
            \incFigure{1.68cm}{0cm}{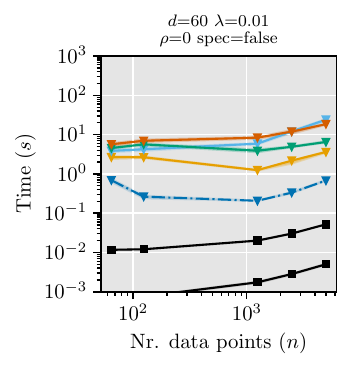}\&
            \incFigure{0.70cm}{0cm}{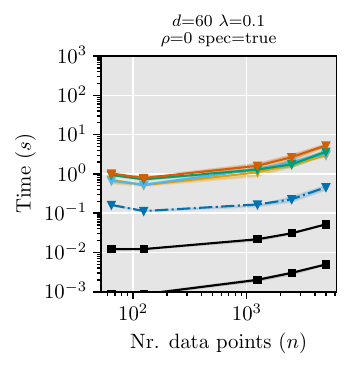}\&
            \incFigure{1.68cm}{0cm}{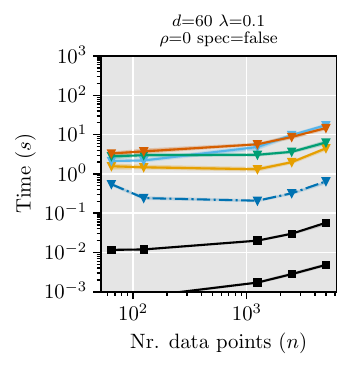}\&\\    
        };
\end{tikzpicture}
}
\caption{Execution times using Spline Features for all parameters
        considered in the experiments. Each plot is paired, where the left is
        wellspecified case and the right is the misspecified case.}
\label{fig:spline-time}
\end{figure}
\begin{figure}[h]
    \def\incFigure#1#2#3{%
        \node{\scalebox{1}{\includegraphics[
            trim={{#1} 0 {#2} 0.2cm},
            clip
        ]{./figs/rff/perm_error/#3.pdf}}};
    }
    \centerfloat
    \resizebox{1.2\textwidth}{!}{
    \begin{tikzpicture} 
        \node (legend) {\includegraphics[scale=2.5]{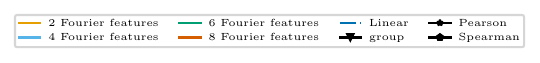}};
        \matrix[below=-0.2cm of legend.south, 
            row sep=-0.1cm, 
            column sep=-0.2cm,
            ampersand replacement=\&] (plotmatrix) {
            \incFigure{0.1cm}{0cm}{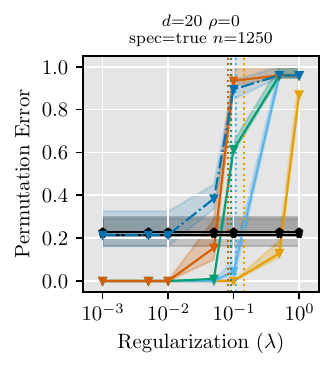}\&[-0.22cm]
            \incFigure{1.38cm}{0cm}{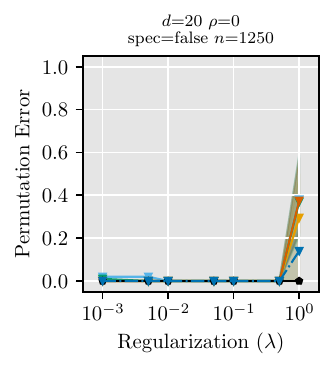}\&
            \incFigure{0.55cm}{0cm}{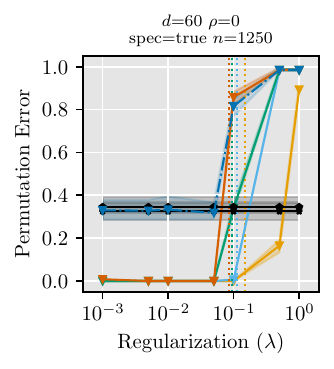}\&[-0.215cm]
            \incFigure{1.38cm}{0cm}{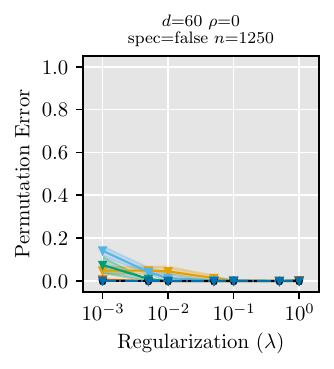}\&
            \incFigure{0.55cm}{0cm}{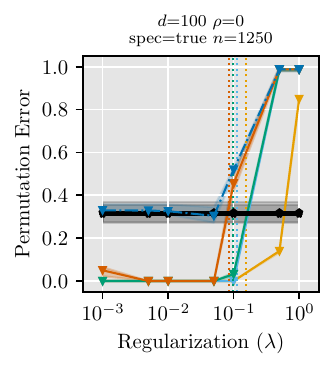}\&[-0.215cm]
            \incFigure{1.38cm}{0cm}{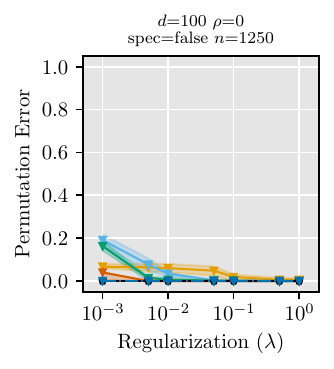}\&\\

            \incFigure{0.1cm}{0cm}{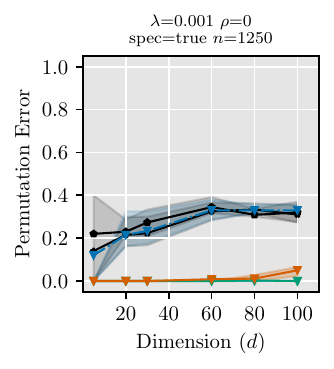}\&
            \incFigure{1.38cm}{0cm}{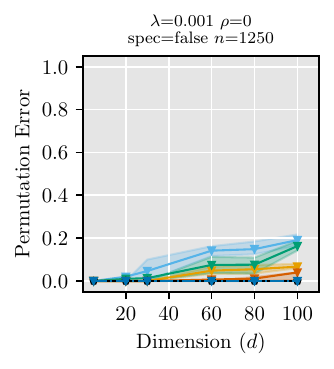}\&
            \incFigure{0.55cm}{0cm}{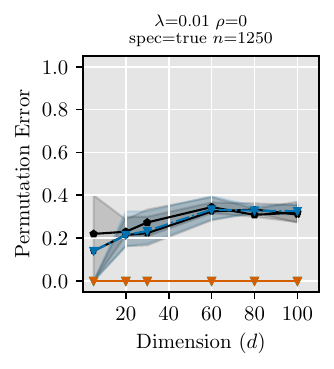}\&
            \incFigure{1.38cm}{0cm}{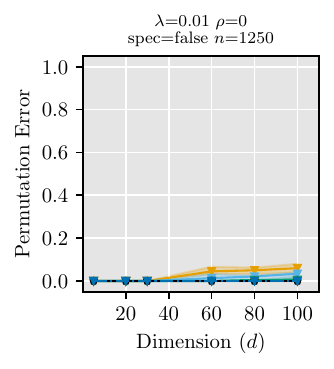}\&
            \incFigure{0.55cm}{0cm}{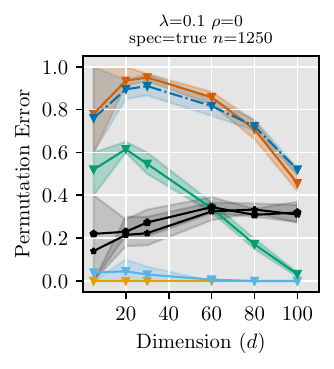}\&
            \incFigure{1.38cm}{0cm}{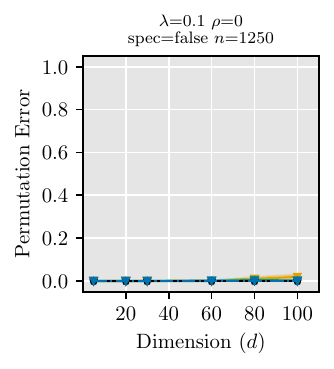}\&\\

            \incFigure{0.1cm}{0cm}{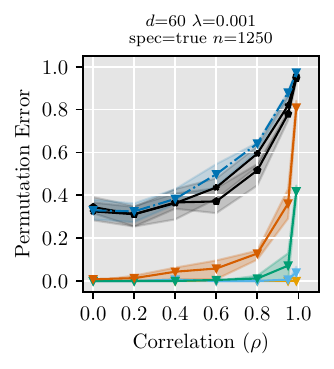}\&
            \incFigure{1.38cm}{0cm}{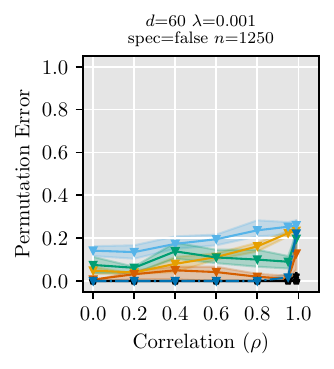}\&
            \incFigure{0.55cm}{0cm}{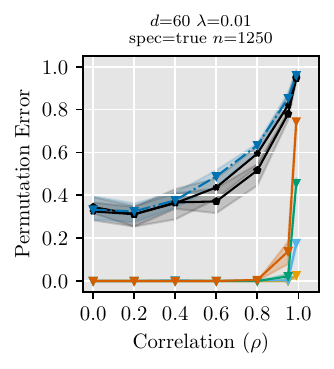}\&
            \incFigure{1.38cm}{0cm}{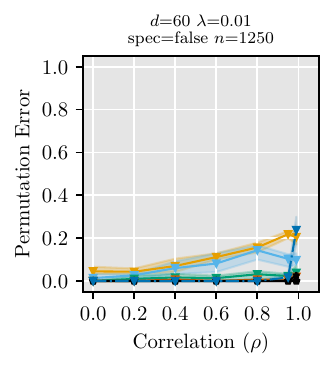}\&
            \incFigure{0.55cm}{0cm}{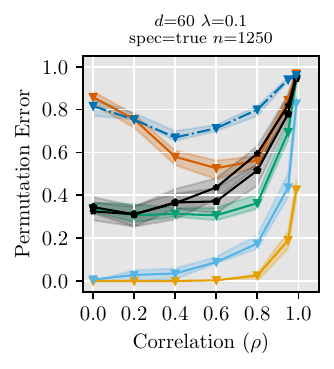}\&
            \incFigure{1.38cm}{0cm}{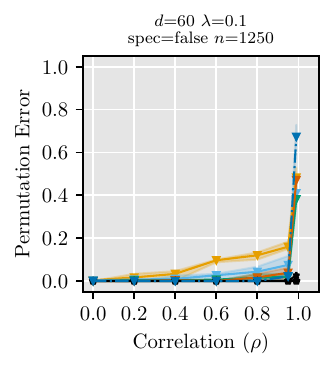}\&\\

            \incFigure{0.1cm}{0cm}{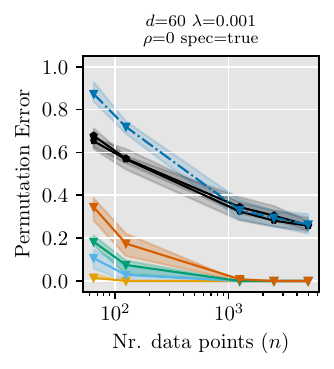}\&
            \incFigure{1.38cm}{0cm}{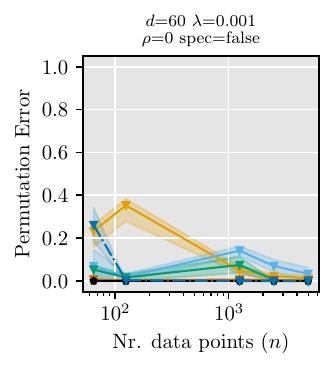}\&
            \incFigure{0.55cm}{0cm}{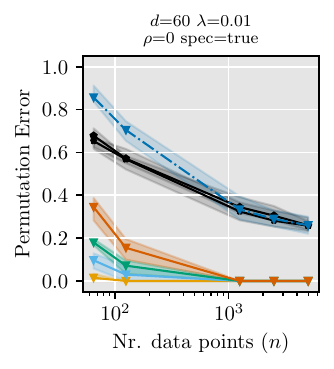}\&
            \incFigure{1.38cm}{0cm}{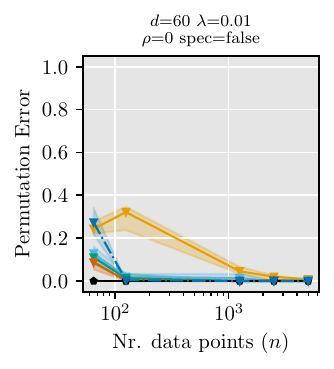}\&
            \incFigure{0.55cm}{0cm}{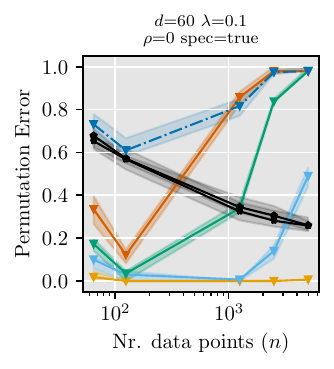}\&
            \incFigure{1.38cm}{0cm}{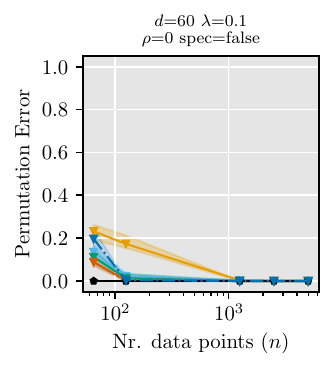}\&\\
        };
    \end{tikzpicture}
}
    \caption{Permutation Errors using Random Fourier Features for all parameters
        considered in the experiments. Each plot is paired, where the left is
        wellspecified case and the right is the misspecified case.}
    \label{fig:rff-perm-errors}
\end{figure}
\begin{figure}[h]
    \def\incFigure#1#2#3{%
        \node{\scalebox{1}{\includegraphics[
            trim={{#1} 0 {#2} 0.2cm},
            clip
        ]{./figs/rff/r2/#3.pdf}}};
    }
    \centerfloat
    \resizebox{1.2\textwidth}{!}{
    \begin{tikzpicture} 
        \node (legend) {\includegraphics[scale=2.5]{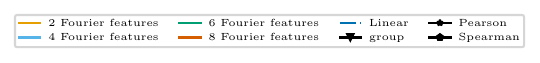}};
        \matrix[below=-0.2cm of legend.south, 
            row sep=-0.1cm, 
            column sep=-0.2cm,
            ampersand replacement=\&] (plotmatrix) {
            \incFigure{0.1cm}{0cm}{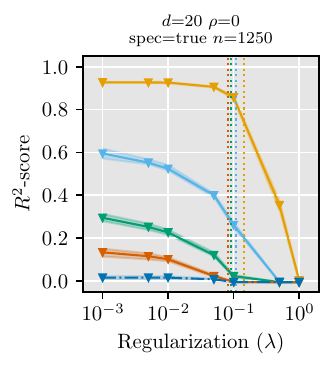}\&[-0.22cm]
            \incFigure{1.38cm}{0cm}{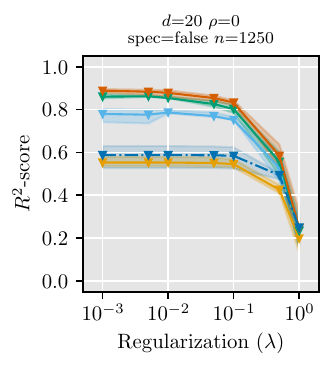}\&
            \incFigure{0.55cm}{0cm}{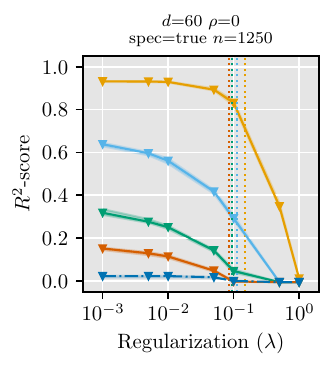}\&[-0.215cm]
            \incFigure{1.38cm}{0cm}{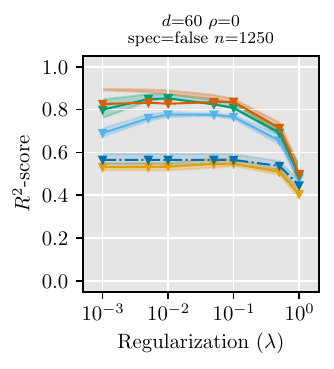}\&
            \incFigure{0.55cm}{0cm}{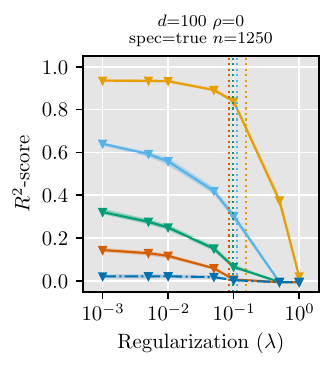}\&[-0.215cm]
            \incFigure{1.38cm}{0cm}{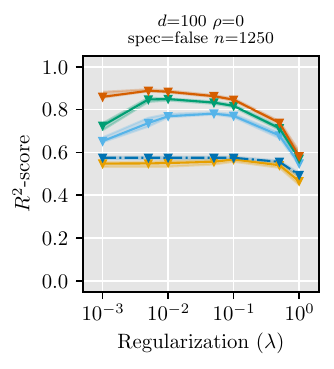}\&\\

            \incFigure{0.1cm}{0cm}{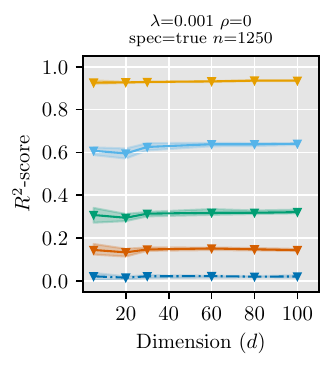}\&
            \incFigure{1.38cm}{0cm}{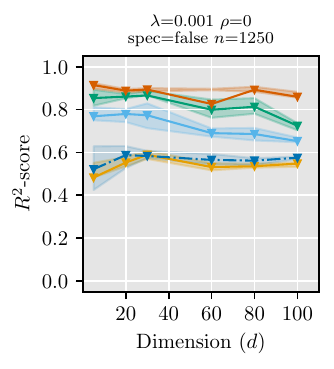}\&
            \incFigure{0.55cm}{0cm}{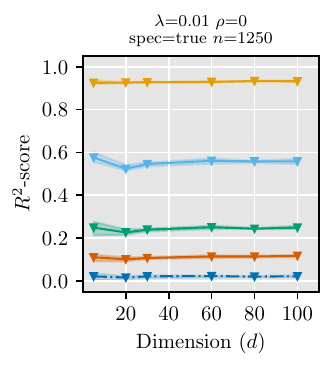}\&
            \incFigure{1.38cm}{0cm}{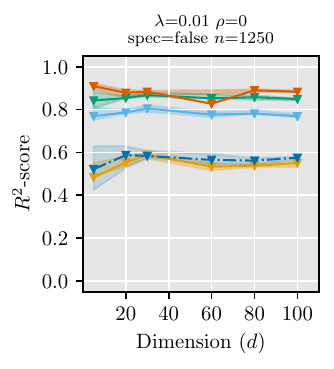}\&
            \incFigure{0.55cm}{0cm}{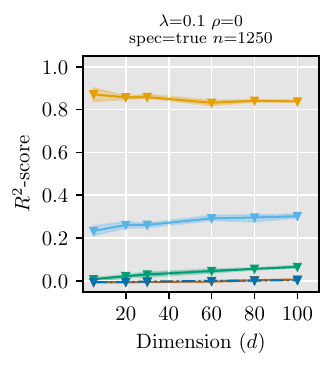}\&
            \incFigure{1.38cm}{0cm}{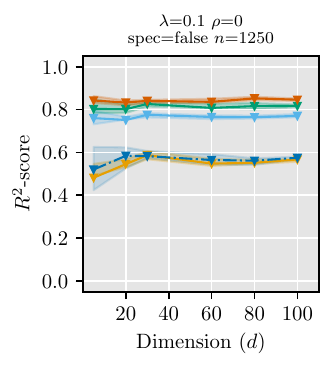}\&\\

            \incFigure{0.1cm}{0cm}{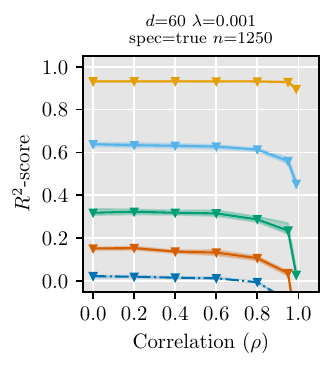}\&
            \incFigure{1.38cm}{0cm}{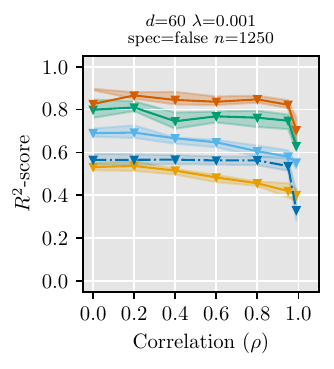}\&
            \incFigure{0.55cm}{0cm}{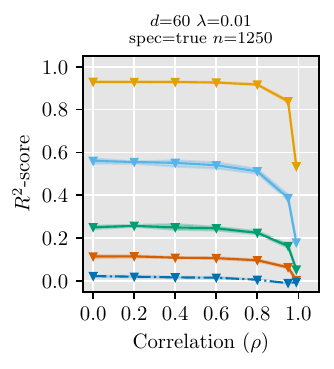}\&
            \incFigure{1.38cm}{0cm}{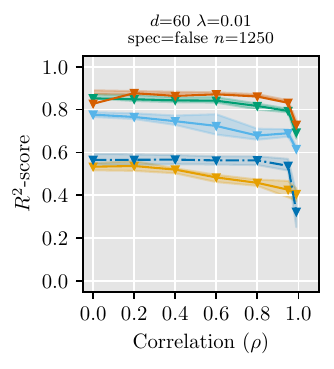}\&
            \incFigure{0.55cm}{0cm}{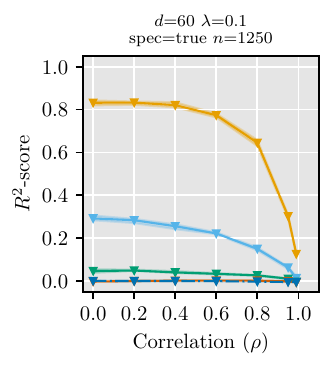}\&
            \incFigure{1.38cm}{0cm}{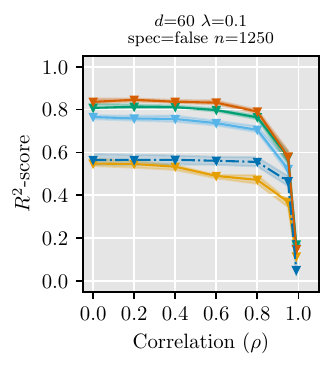}\&\\

            \incFigure{0.1cm}{0cm}{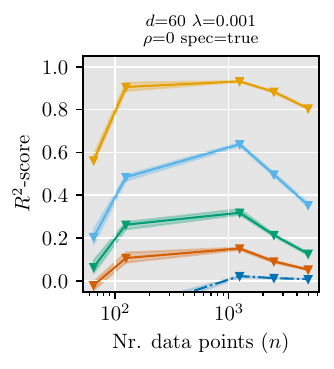}\&
            \incFigure{1.38cm}{0cm}{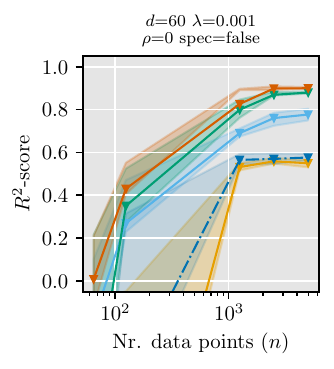}\&
            \incFigure{0.55cm}{0cm}{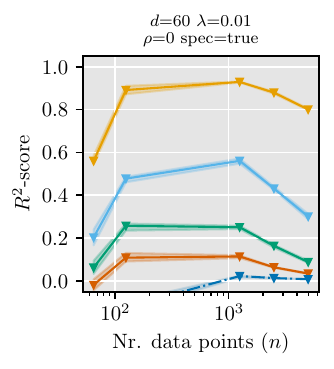}\&
            \incFigure{1.38cm}{0cm}{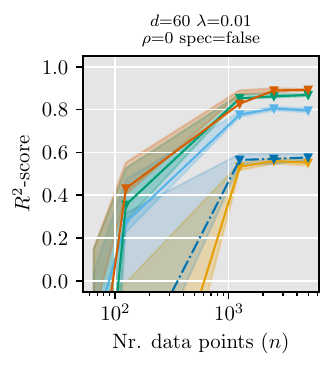}\&
            \incFigure{0.55cm}{0cm}{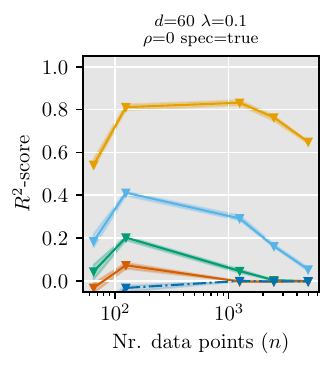}\&
            \incFigure{1.38cm}{0cm}{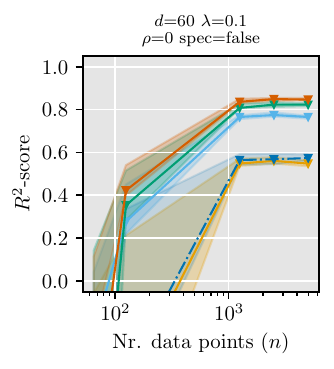}\&\\
        };
    \end{tikzpicture}
}
    \caption{$R^2$-score on the diagonal using Random Fourier Features for all parameters
        considered in the experiments. Each plot is paired, where the left is
        wellspecified case and the right is the misspecified case.}
    \label{fig:rff-r2}
\end{figure}
\begin{figure}[h]
    \def\incFigure#1#2#3{%
        \node{\scalebox{1}{\includegraphics[
            trim={{#1} 0 {#2} 0.2cm},
            clip
        ]{./figs/rff/time/#3.pdf}}};
    }
    \centerfloat
    \resizebox{1.2\textwidth}{!}{
    \begin{tikzpicture} 
        \node (legend) {\includegraphics[scale=2.5]{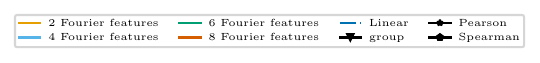}};
        \matrix[below=-0.2cm of legend.south, 
            row sep=-0.1cm, 
            column sep=-0.2cm,
            ampersand replacement=\&] (plotmatrix) {
            \incFigure{0.25cm}{0cm}{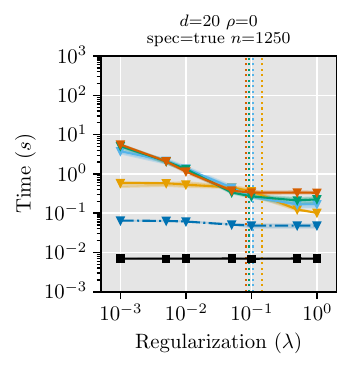}\&[-0.22cm]
            \incFigure{1.68cm}{0cm}{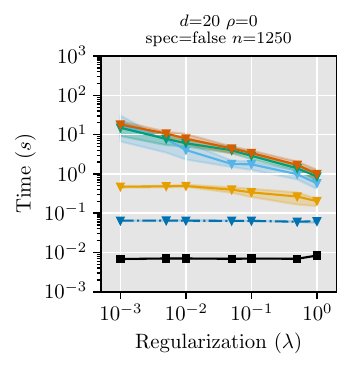}\&
            \incFigure{0.70cm}{0cm}{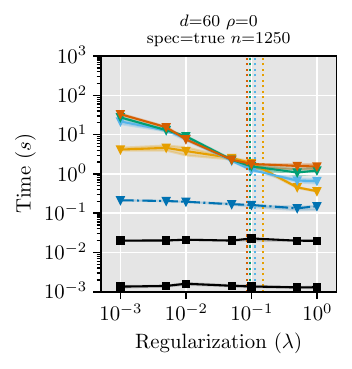}\&[-0.215cm]
            \incFigure{1.68cm}{0cm}{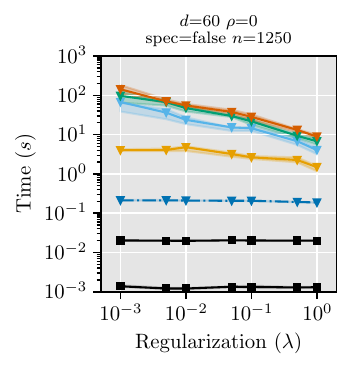}\&
            \incFigure{0.70cm}{0cm}{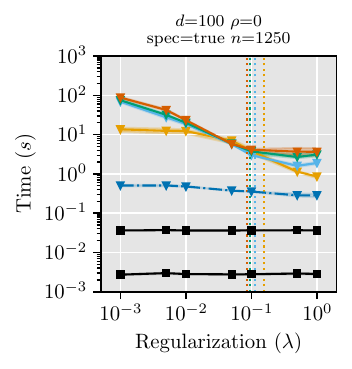}\&[-0.215cm]
            \incFigure{1.68cm}{0cm}{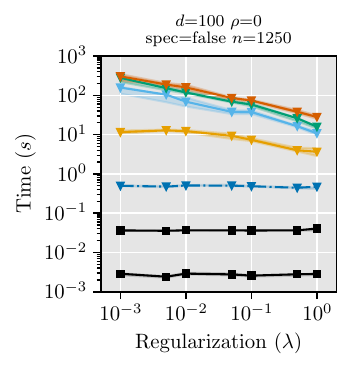}\&\\

            \incFigure{0.25cm}{0cm}{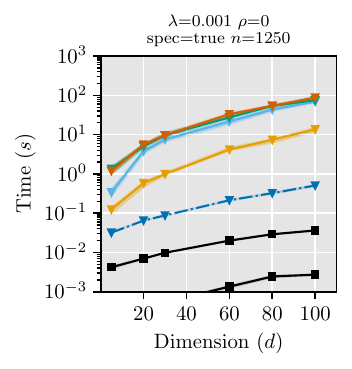}\&
            \incFigure{1.68cm}{0cm}{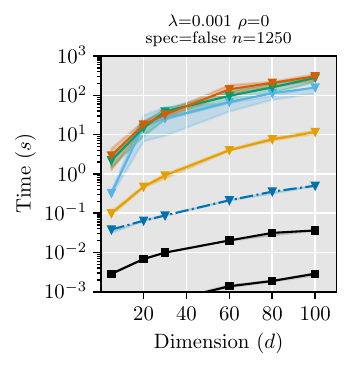}\&
            \incFigure{0.70cm}{0cm}{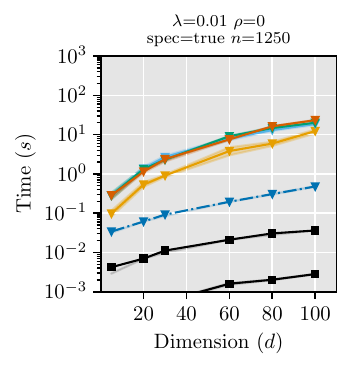}\&
            \incFigure{1.68cm}{0cm}{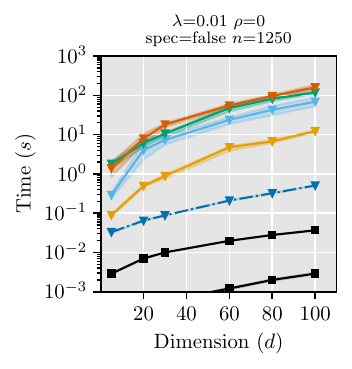}\&
            \incFigure{0.70cm}{0cm}{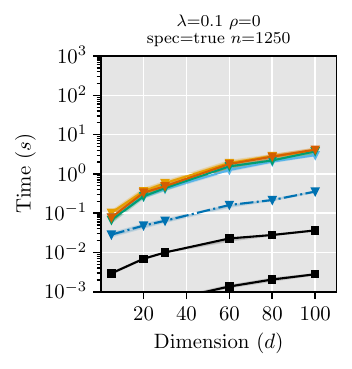}\&
            \incFigure{1.68cm}{0cm}{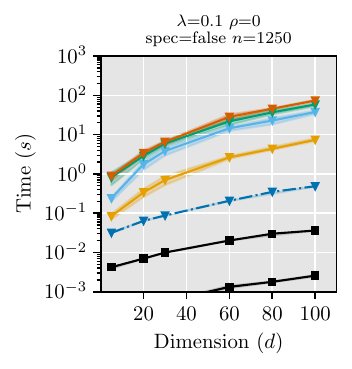}\&\\

            \incFigure{0.25cm}{0cm}{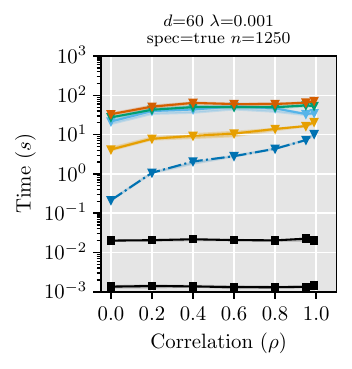}\&
            \incFigure{1.68cm}{0cm}{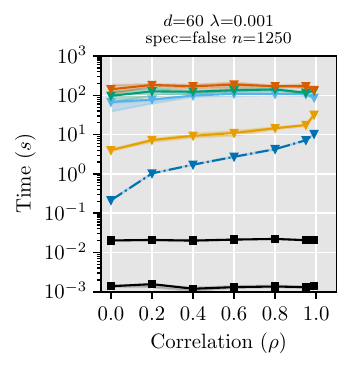}\&
            \incFigure{0.70cm}{0cm}{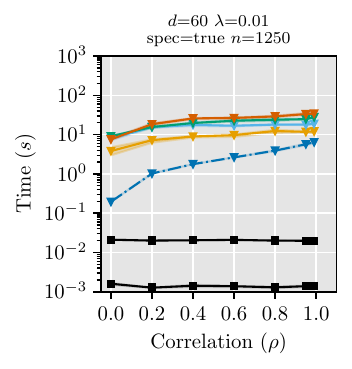}\&
            \incFigure{1.68cm}{0cm}{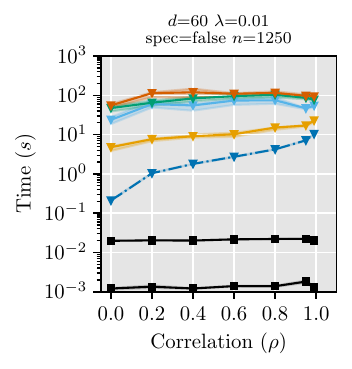}\&
            \incFigure{0.70cm}{0cm}{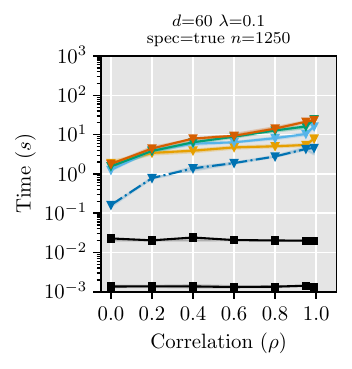}\&
            \incFigure{1.68cm}{0cm}{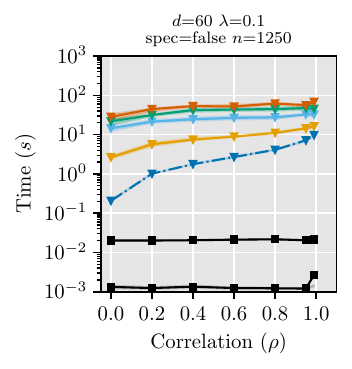}\&\\

            \incFigure{0.25cm}{0cm}{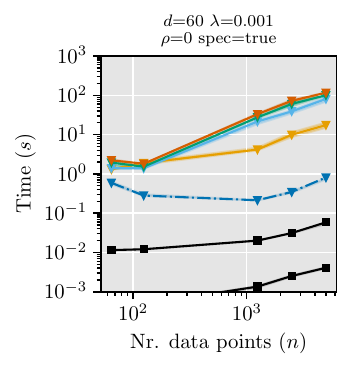}\&
            \incFigure{1.68cm}{0cm}{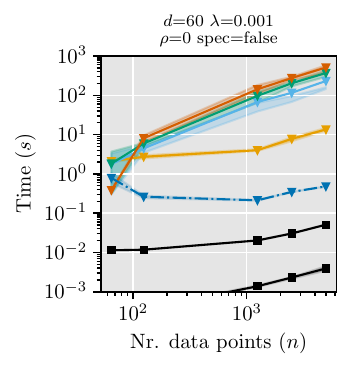}\&
            \incFigure{0.70cm}{0cm}{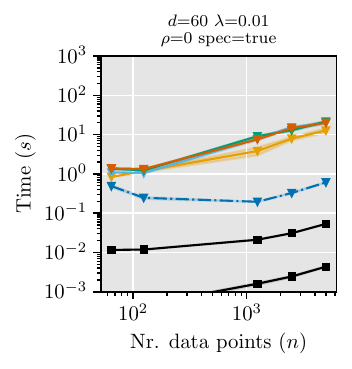}\&
            \incFigure{1.68cm}{0cm}{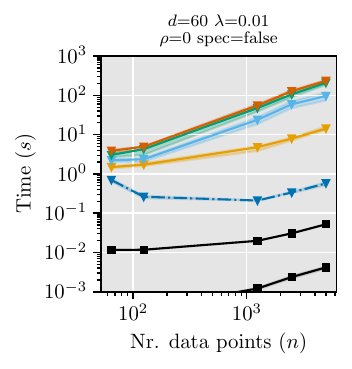}\&
            \incFigure{0.70cm}{0cm}{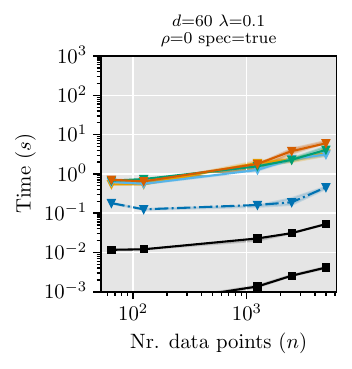}\&
            \incFigure{1.68cm}{0cm}{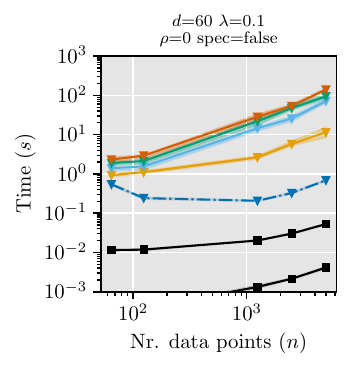}\&\\
        };
    \end{tikzpicture}
    }
    \caption{Execution times using Random Fourier Features for all parameters
        considered in the experiments. Each plot is paired, where the left is
        wellspecified case and the right is the misspecified case.}
    \label{fig:rff-time}
\end{figure}
\begin{figure}[h]
    \def\incFigure#1#2#3{%
        \node{\scalebox{1}{\includegraphics[
            trim={{#1} 0 {#2} 0.2cm},
            clip
        ]{./figs/kernel/perm_error/#3.pdf}}};
    }
    \centerfloat
    \resizebox{1.2\textwidth}{!}{
    \begin{tikzpicture} 
        \node (legend) {\includegraphics[scale=2.5]{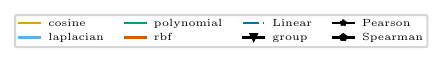}};
        \matrix[below=-0.2cm of legend.south, 
            row sep=-0.1cm, 
            column sep=-0.2cm,
            ampersand replacement=\&] (plotmatrix) {
            \incFigure{0.1cm}{0cm}{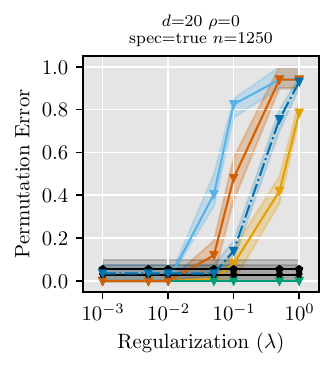}\&[-0.22cm]
            \incFigure{1.38cm}{0cm}{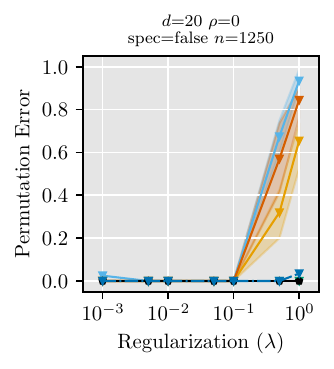}\&
            \incFigure{0.55cm}{0cm}{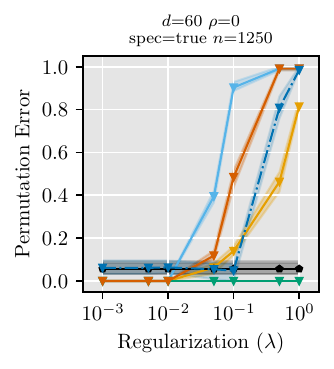}\&[-0.215cm]
            \incFigure{1.38cm}{0cm}{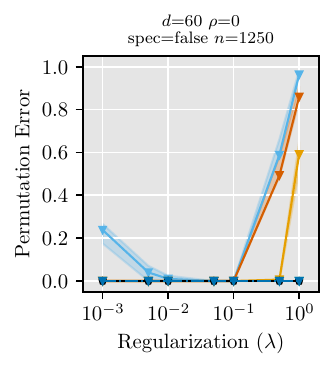}\&
            \incFigure{0.55cm}{0cm}{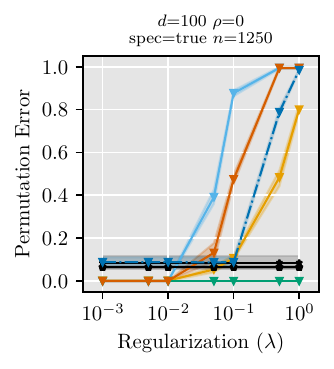}\&[-0.215cm]
            \incFigure{1.38cm}{0cm}{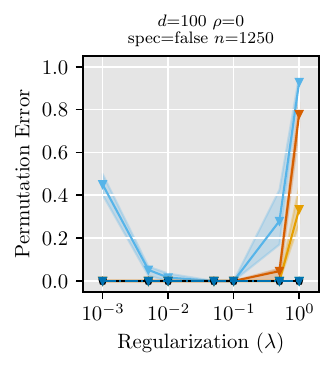}\&\\

            \incFigure{0.1cm}{0cm}{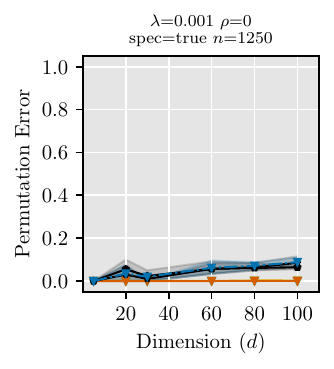}\&
            \incFigure{1.38cm}{0cm}{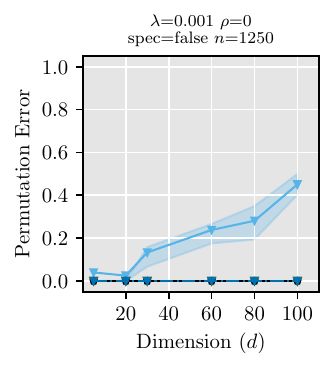}\&
            \incFigure{0.55cm}{0cm}{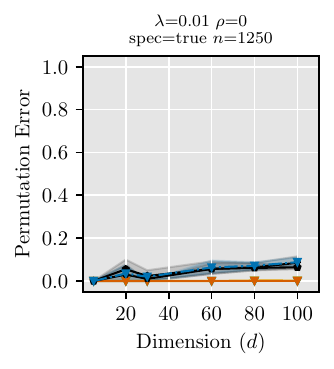}\&
            \incFigure{1.38cm}{0cm}{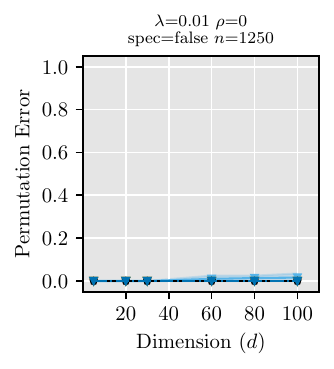}\&
            \incFigure{0.55cm}{0cm}{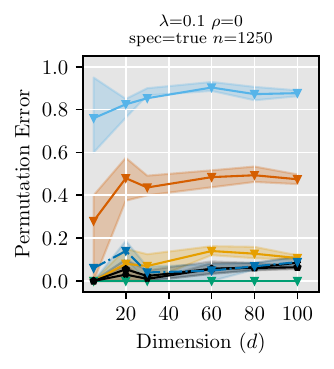}\&
            \incFigure{1.38cm}{0cm}{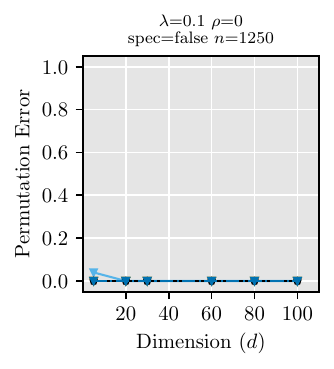}\&\\

            \incFigure{0.1cm}{0cm}{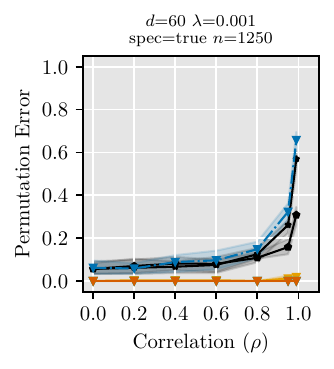}\&
            \incFigure{1.38cm}{0cm}{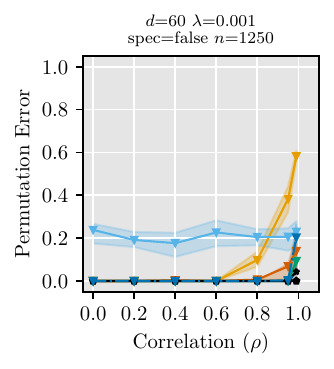}\&
            \incFigure{0.55cm}{0cm}{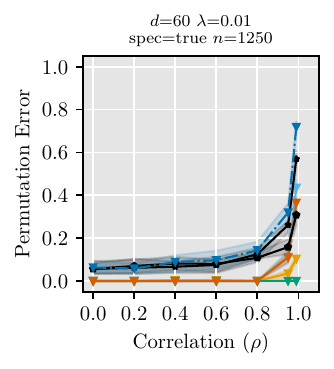}\&
            \incFigure{1.38cm}{0cm}{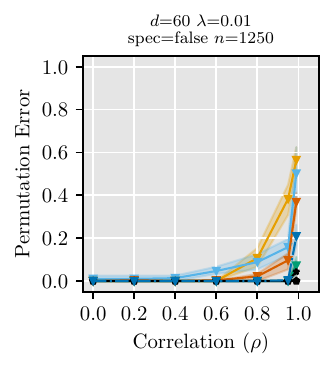}\&
            \incFigure{0.55cm}{0cm}{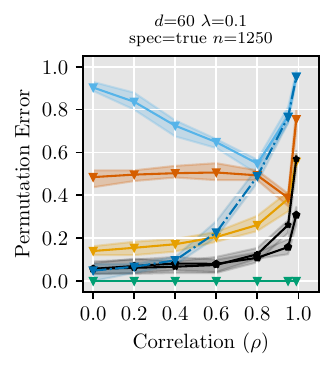}\&
            \incFigure{1.38cm}{0cm}{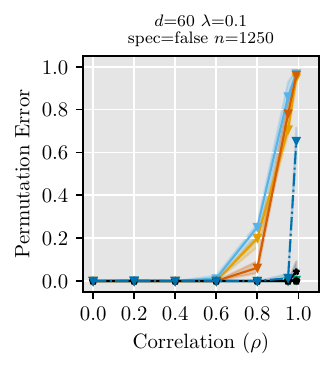}\&\\

            \incFigure{0.1cm}{0cm}{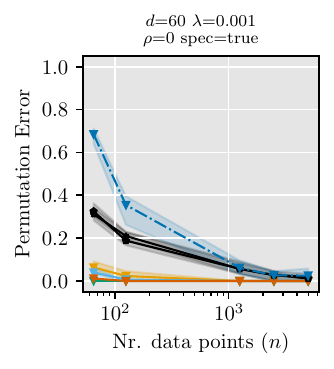}\&
            \incFigure{1.38cm}{0cm}{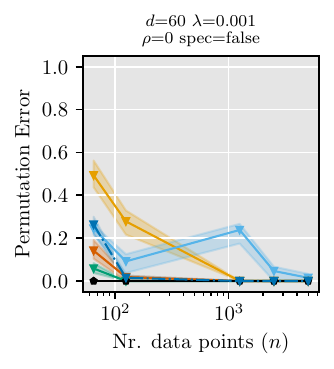}\&
            \incFigure{0.55cm}{0cm}{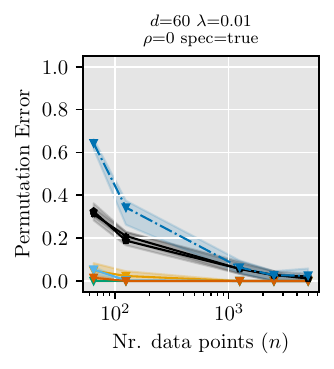}\&
            \incFigure{1.38cm}{0cm}{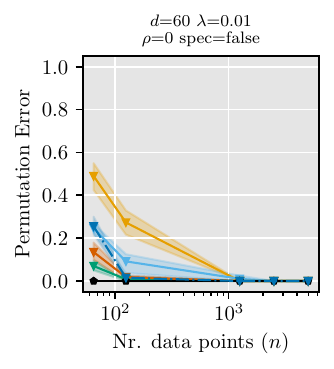}\&
            \incFigure{0.55cm}{0cm}{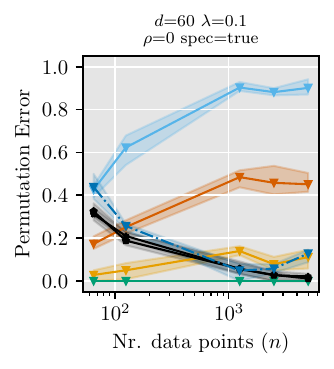}\&
            \incFigure{1.38cm}{0cm}{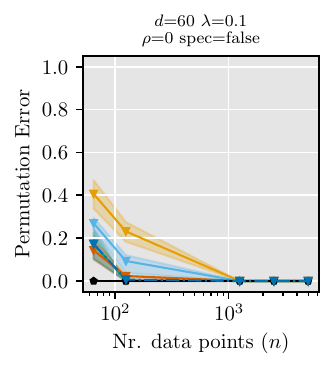}\&\\
        };
    \end{tikzpicture}
    }
    \caption{Permutation Errors using Kernels}
    \label{fig:kernel-perm-errors}
\end{figure}
\begin{figure}[h]
    \def\incFigure#1#2#3{%
        \node{\scalebox{1}{\includegraphics[
            trim={{#1} 0 {#2} 0.2cm},
            clip
        ]{./figs/kernel/r2/#3.pdf}}};
    }
    \centerfloat
    \resizebox{1.2\textwidth}{!}{
    \begin{tikzpicture} 
        \node (legend) {\includegraphics[scale=2.5]{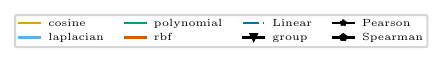}};
        \matrix[below=-0.2cm of legend.south, 
            row sep=-0.1cm, 
            column sep=-0.2cm,
            ampersand replacement=\&] (plotmatrix) {
            \incFigure{0.1cm}{0cm}{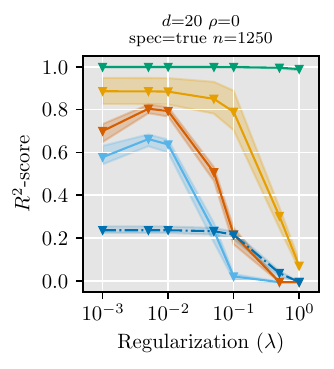}\&[-0.22cm]
            \incFigure{1.38cm}{0cm}{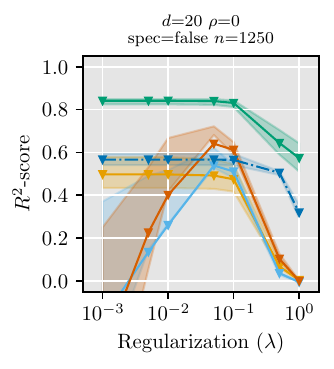}\&
            \incFigure{0.55cm}{0cm}{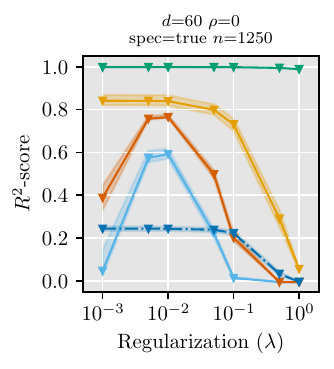}\&[-0.215cm]
            \incFigure{1.38cm}{0cm}{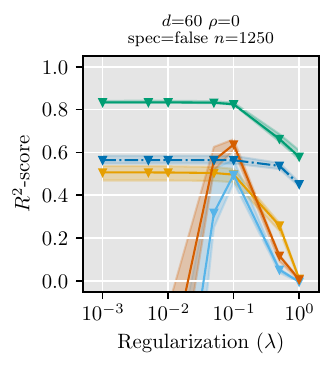}\&
            \incFigure{0.55cm}{0cm}{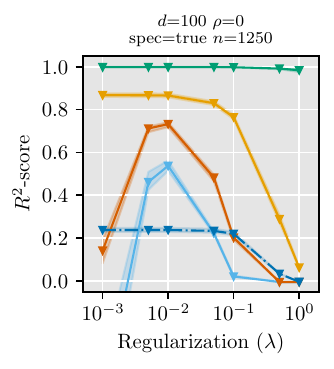}\&[-0.215cm]
            \incFigure{1.38cm}{0cm}{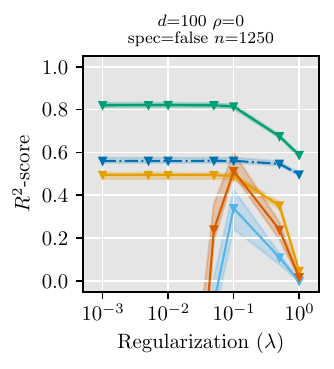}\&\\

            \incFigure{0.1cm}{0cm}{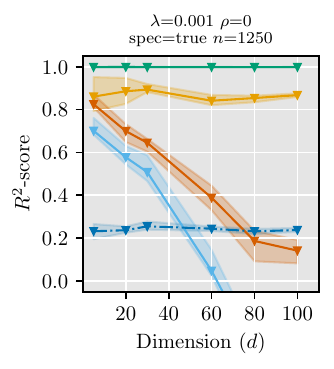}\&
            \incFigure{1.38cm}{0cm}{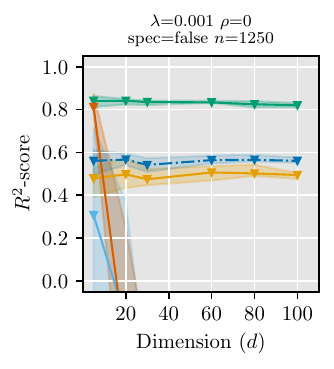}\&
            \incFigure{0.55cm}{0cm}{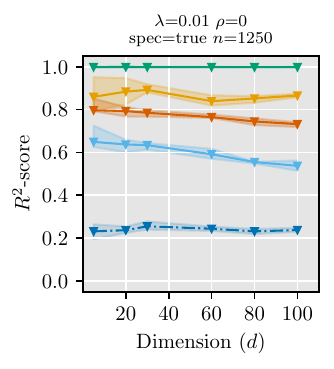}\&
            \incFigure{1.38cm}{0cm}{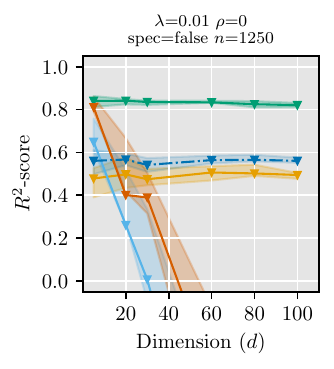}\&
            \incFigure{0.55cm}{0cm}{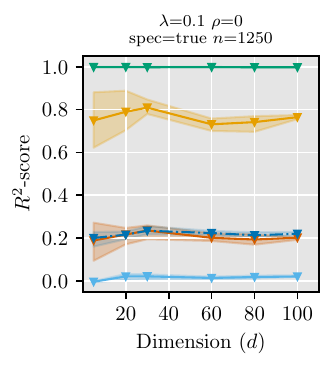}\&
            \incFigure{1.38cm}{0cm}{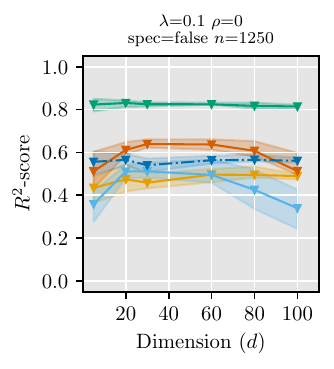}\&\\

            \incFigure{0.1cm}{0cm}{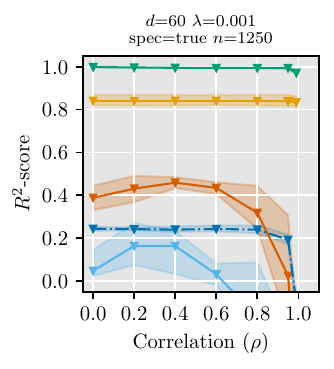}\&
            \incFigure{1.38cm}{0cm}{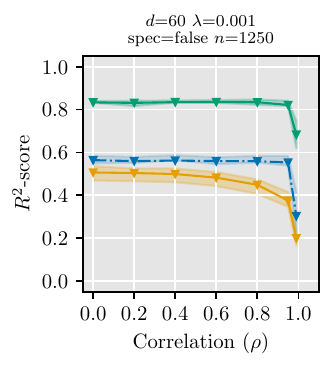}\&
            \incFigure{0.55cm}{0cm}{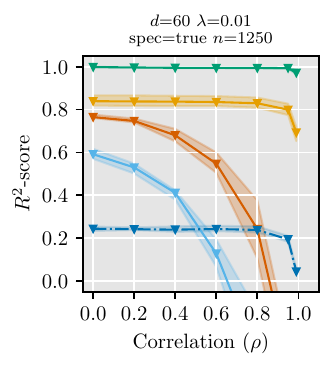}\&
            \incFigure{1.38cm}{0cm}{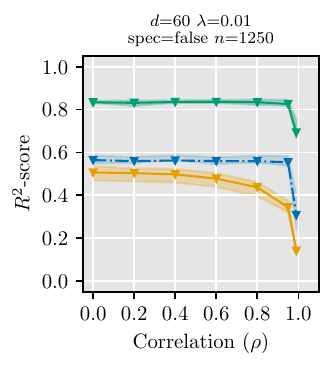}\&
            \incFigure{0.55cm}{0cm}{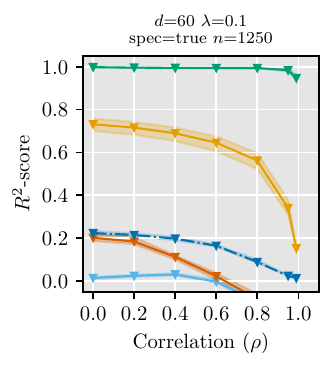}\&
            \incFigure{1.38cm}{0cm}{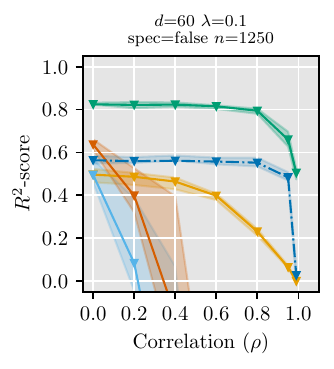}\&\\

            \incFigure{0.1cm}{0cm}{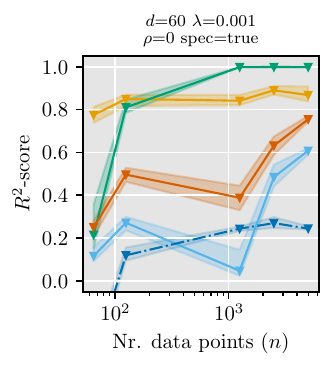}\&
            \incFigure{1.38cm}{0cm}{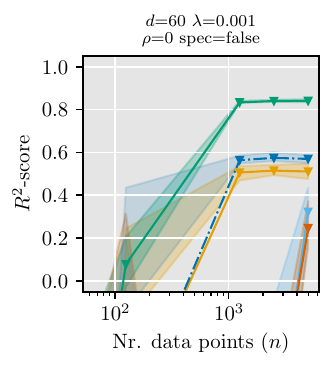}\&
            \incFigure{0.55cm}{0cm}{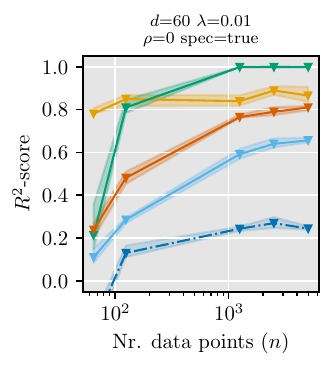}\&
            \incFigure{1.38cm}{0cm}{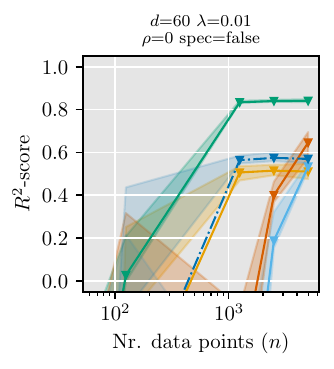}\&
            \incFigure{0.55cm}{0cm}{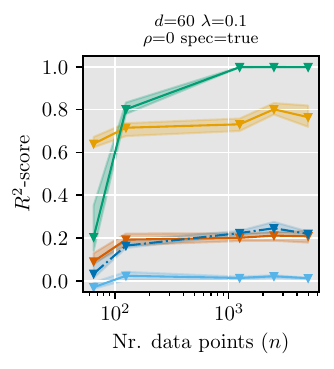}\&
            \incFigure{1.38cm}{0cm}{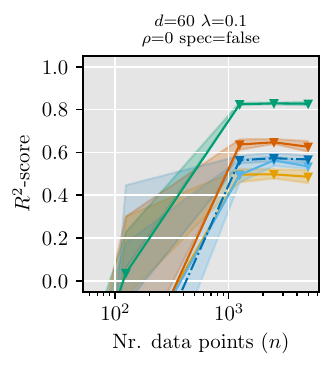}\&\\
        };
    \end{tikzpicture}
}
    \caption{$R^2$-score on the diagonal using Kernels for all parameters
        considered in the experiments. Each plot is paired, where the left is
        wellspecified case and the right is the misspecified case.}
    \label{fig:kernel-r2}
\end{figure}
\begin{figure}
    \def\incFigure#1#2#3{%
        \node{\scalebox{1}{\includegraphics[
            trim={{#1} 0 {#2} 0.2cm},
            clip
        ]{./figs/kernel/time/#3.pdf}}};
    }
    \centerfloat
    \resizebox{1.2\textwidth}{!}{
    \begin{tikzpicture} 
        \node (legend) {\includegraphics[scale=2.5]{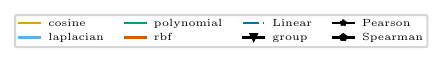}};
        \matrix[below=-0.2cm of legend.south, 
            row sep=-0.1cm, 
            column sep=-0.2cm,
            ampersand replacement=\&] (plotmatrix) {
            \incFigure{0.25cm}{0cm}{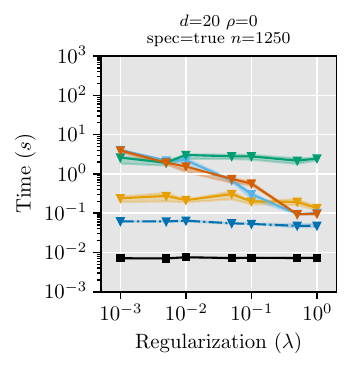}\&[-0.22cm]
            \incFigure{1.68cm}{0cm}{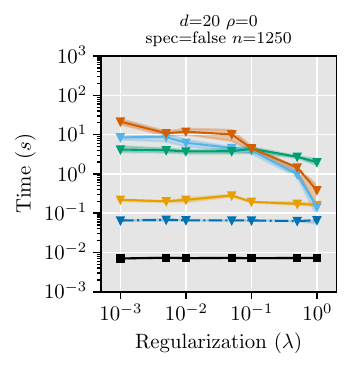}\&
            \incFigure{0.70cm}{0cm}{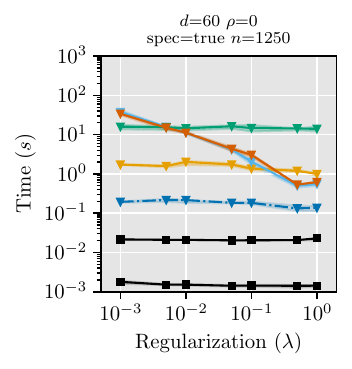}\&[-0.215cm]
            \incFigure{1.68cm}{0cm}{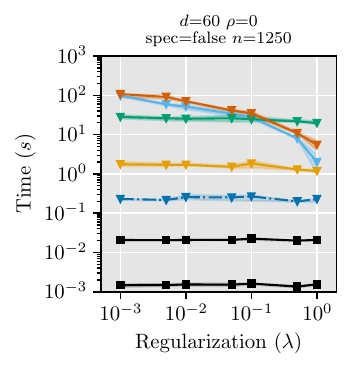}\&
            \incFigure{0.70cm}{0cm}{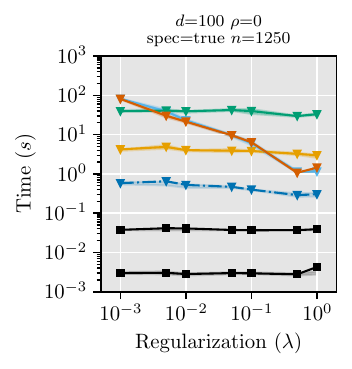}\&[-0.215cm]
            \incFigure{1.68cm}{0cm}{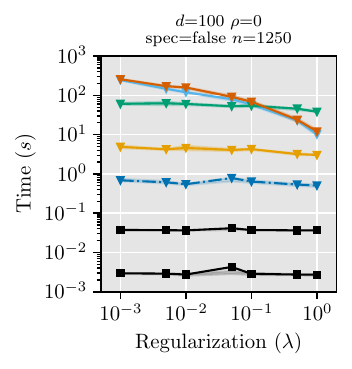}\&\\

            \incFigure{0.25cm}{0cm}{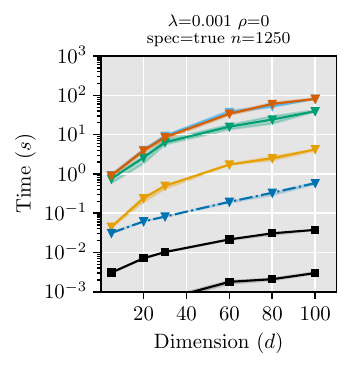}\&
            \incFigure{1.68cm}{0cm}{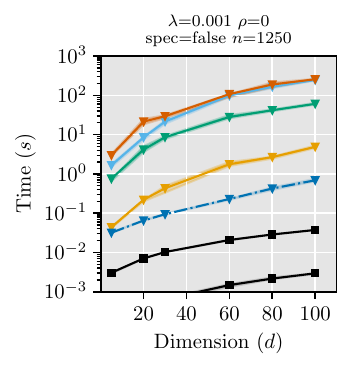}\&
            \incFigure{0.70cm}{0cm}{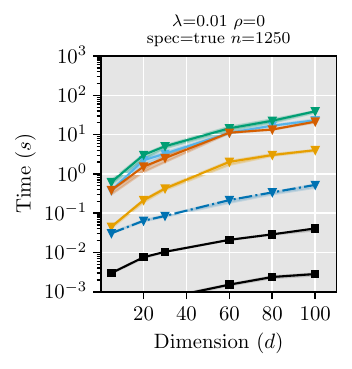}\&
            \incFigure{1.68cm}{0cm}{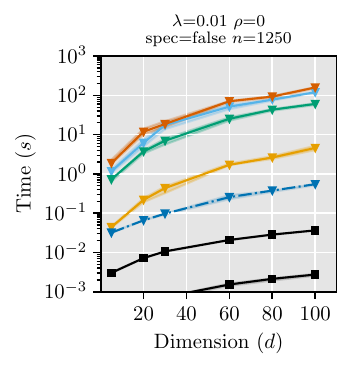}\&
            \incFigure{0.70cm}{0cm}{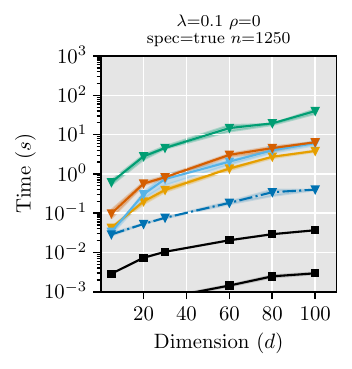}\&
            \incFigure{1.68cm}{0cm}{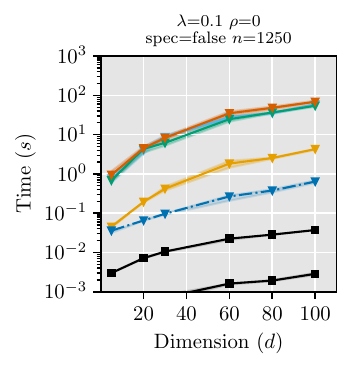}\&\\

            \incFigure{0.25cm}{0cm}{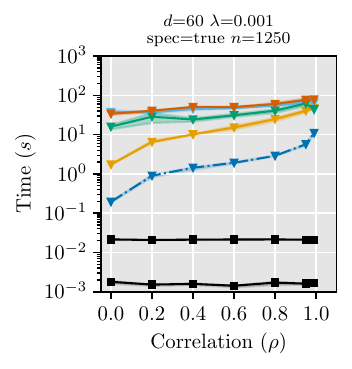}\&
            \incFigure{1.68cm}{0cm}{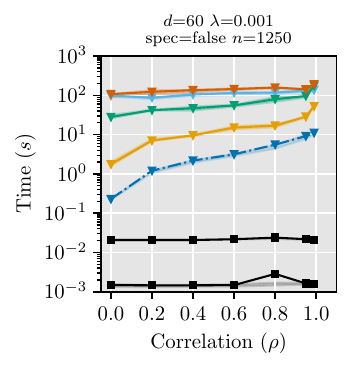}\&
            \incFigure{0.70cm}{0cm}{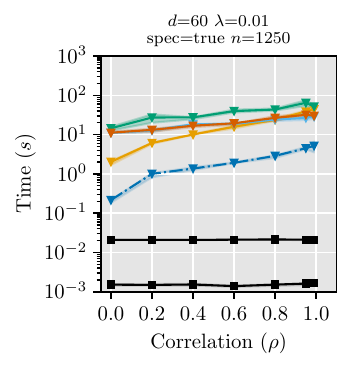}\&
            \incFigure{1.68cm}{0cm}{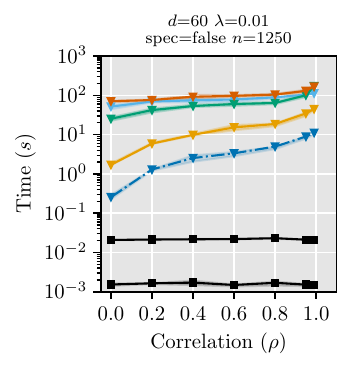}\&
            \incFigure{0.70cm}{0cm}{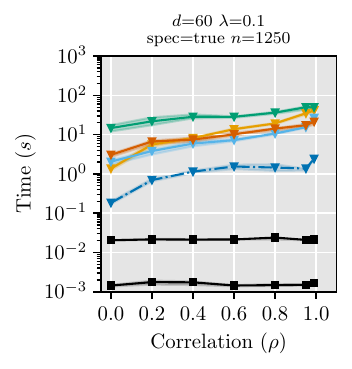}\&
            \incFigure{1.68cm}{0cm}{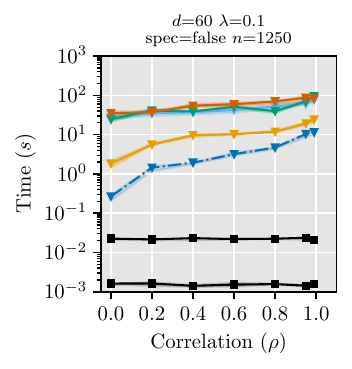}\&\\

            \incFigure{0.25cm}{0cm}{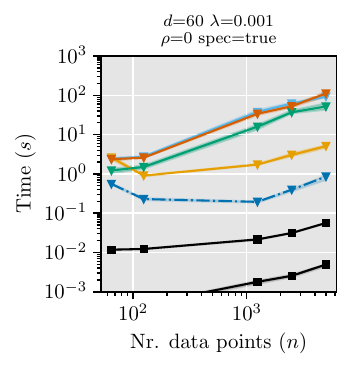}\&
            \incFigure{1.68cm}{0cm}{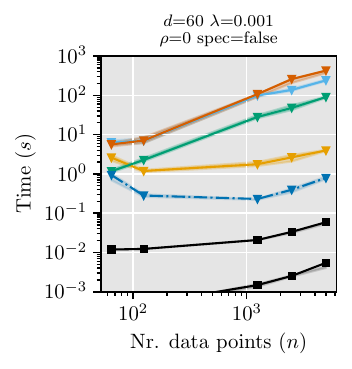}\&
            \incFigure{0.70cm}{0cm}{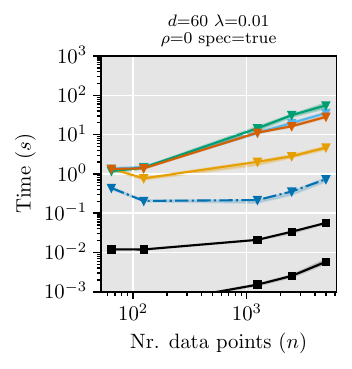}\&
            \incFigure{1.68cm}{0cm}{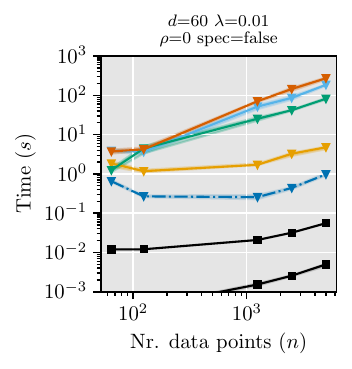}\&
            \incFigure{0.70cm}{0cm}{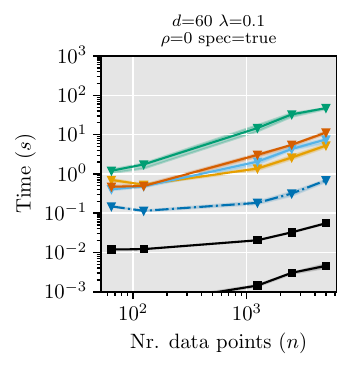}\&
            \incFigure{1.68cm}{0cm}{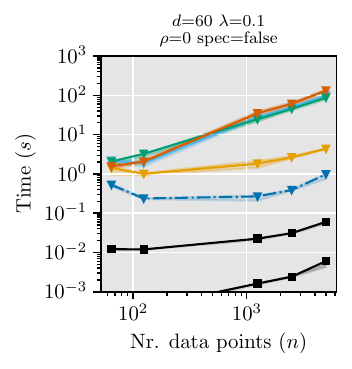}\&\\        
        };
    \end{tikzpicture}
}
    \caption{Execution times using kernels for all parameters
        considered in the experiments. Each plot is paired, where the left is
        wellspecified case and the right is the misspecified case.}
    \label{fig:kernel-time}
\end{figure}
\begin{figure}[t]
    \vspace{-2em}
    \centerfloat
    \resizebox{\textwidth}{!}{%
        \includegraphics[
            trim={0 0 0 0.2cm},
            clip]{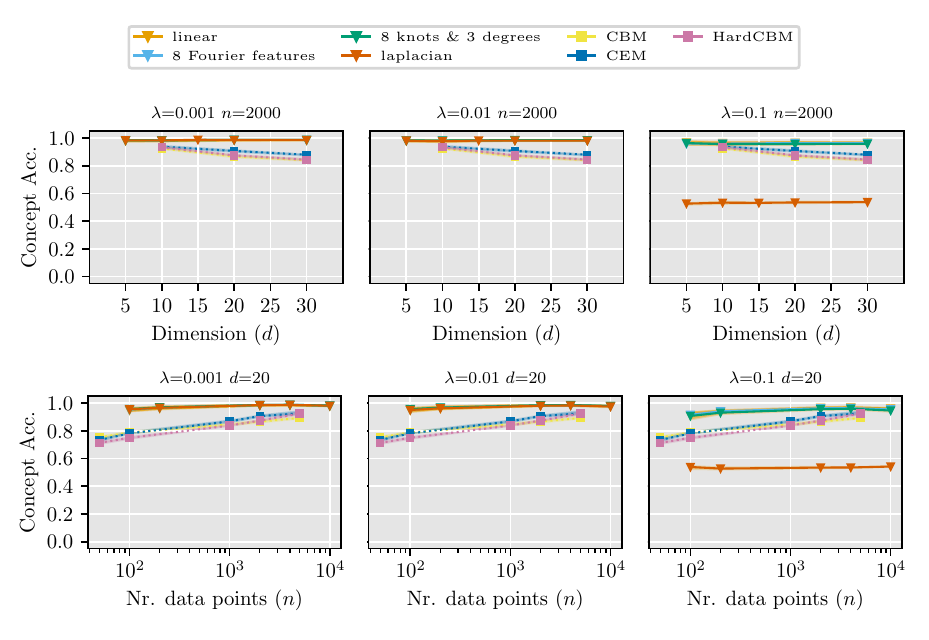}
    }
    \caption{Concept accuracy on the Toy Dataset when the training correlation is $0.5$ and the test correlation 
    is  $0$. Several versions of our estimator are compared to the performance of several concept-based models. }
    \label{fig:ablation-concept-acc}
    \vspace{-2em}
\end{figure}
\begin{figure}[b]
    \centerfloat
    \resizebox{\textwidth}{!}{%
        \includegraphics[
            trim={0 0 0 0.2cm},
            clip]{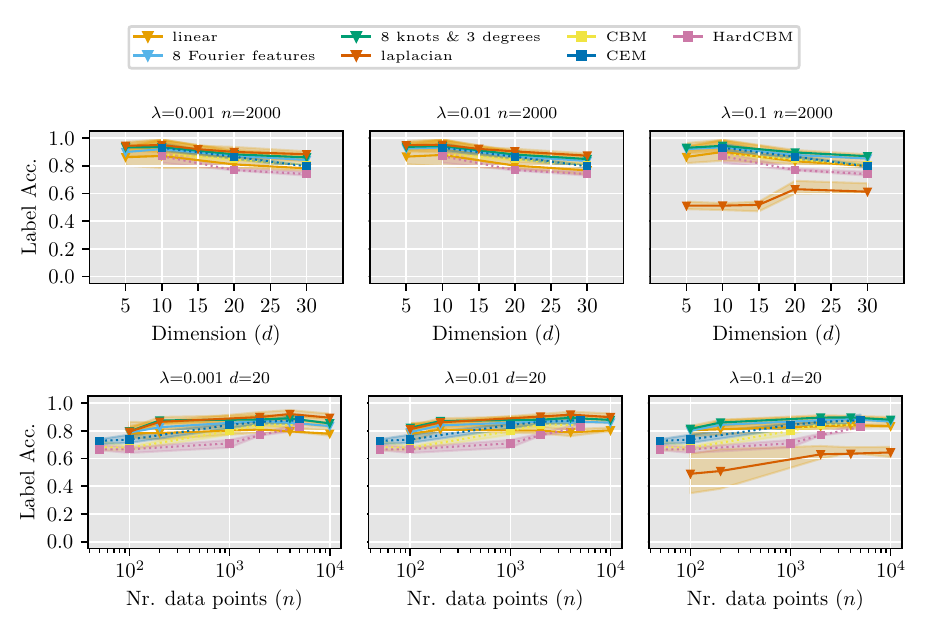}
    }
    \caption{Label accuracy on the Toy Dataset when the training correlation is $0.5$ and the test correlation 
    is  $0$. Several versions of our estimator are compared to the performance of several concept-based models. }
    \label{fig:ablation-label-acc}
\end{figure}
\begin{figure}[t]
    \vspace{-2em}
    \centerfloat
    \resizebox{\textwidth}{!}{%
        \includegraphics[
            trim={0 0 0 0.2cm},
            clip]{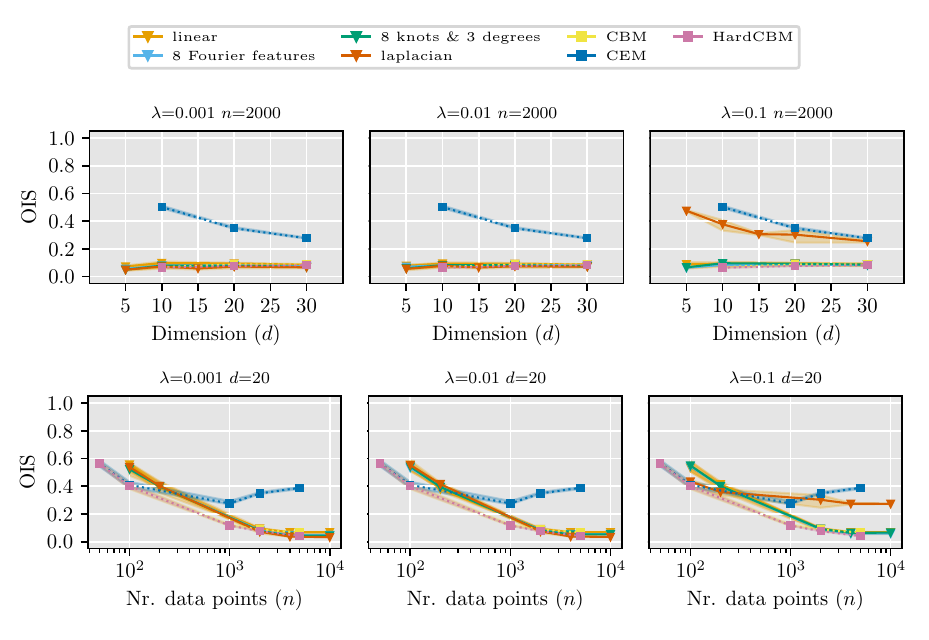}
    }    
    \caption{OIS metric on the Toy Dataset when the training correlation is $0.5$ and the test correlation 
    is  $0$. Several versions of our estimator are compared to the performance of several concept-based models. }
    \label{fig:ablation-ois-score}
    \vspace{-2em}
\end{figure}
\begin{figure}[b]
    \centerfloat
    \resizebox{\textwidth}{!}{%
        \includegraphics[
            trim={0 0 0 0.2cm},
            clip]{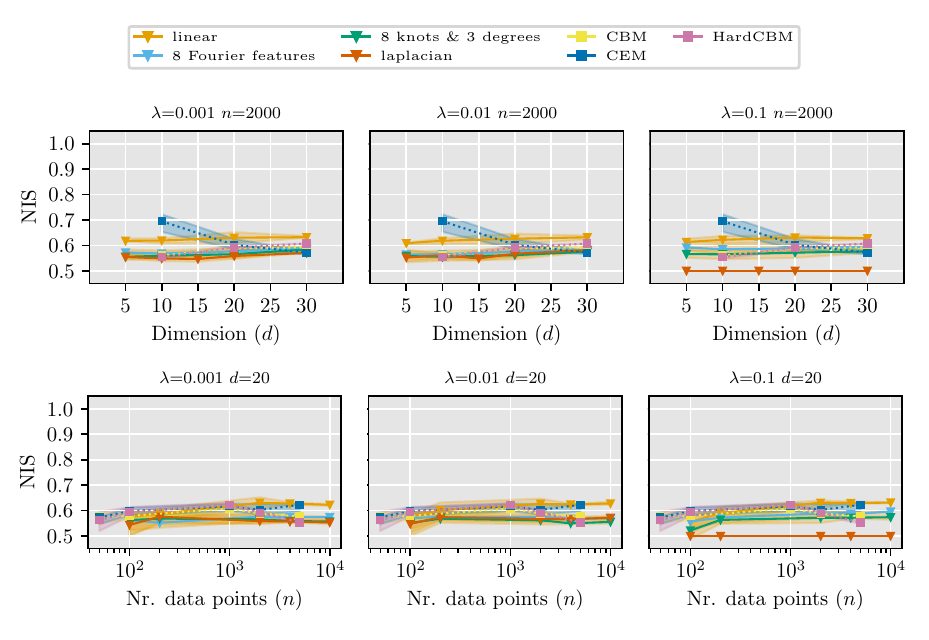}
    }    
    \caption{NIS metric on the Toy Dataset when the training correlation is $0.5$ and the test correlation 
    is  $0$. Several versions of our estimator are compared to the performance of several concept-based models.}
    \label{fig:ablation-nis-score}
\end{figure}
\begin{figure}[t]
    \vspace{-2em}
    \centerfloat
    \resizebox{\textwidth}{!}{%
        \includegraphics[
            trim={0 0 0 0.2cm},
            clip]{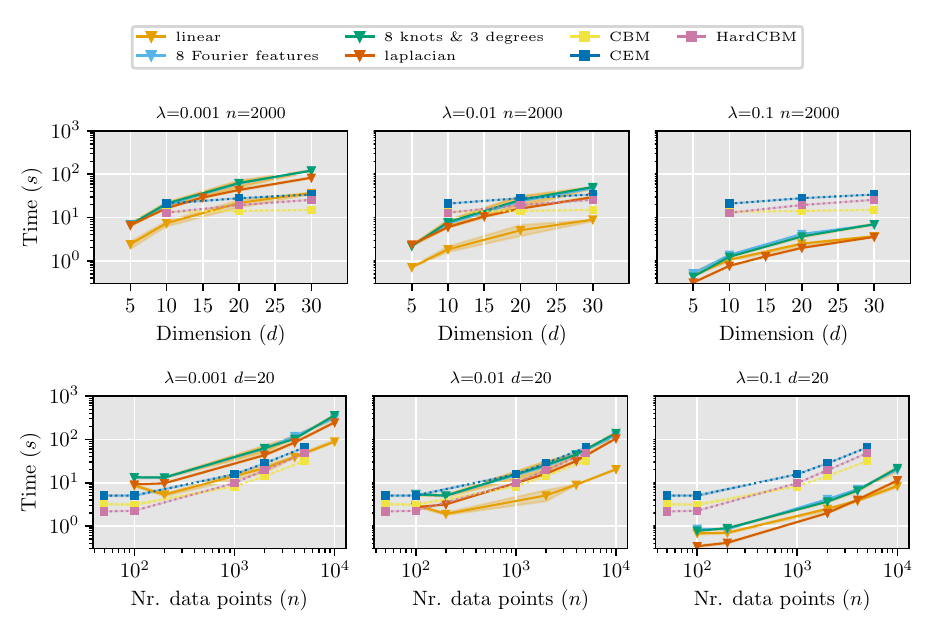}
    }
    \caption{Execution times on the Toy Dataset when the training correlation is $0.5$ and the test correlation 
    is  $0$. Several versions of our estimator are compared to the performance of several concept-based models. }
    \label{fig:ablation-time}
    \vspace{-2em}
\end{figure}

\clearpage
\newpage

\begin{figure}
    \centerfloat
    \resizebox{1.2\textwidth}{!}{%
        \includegraphics{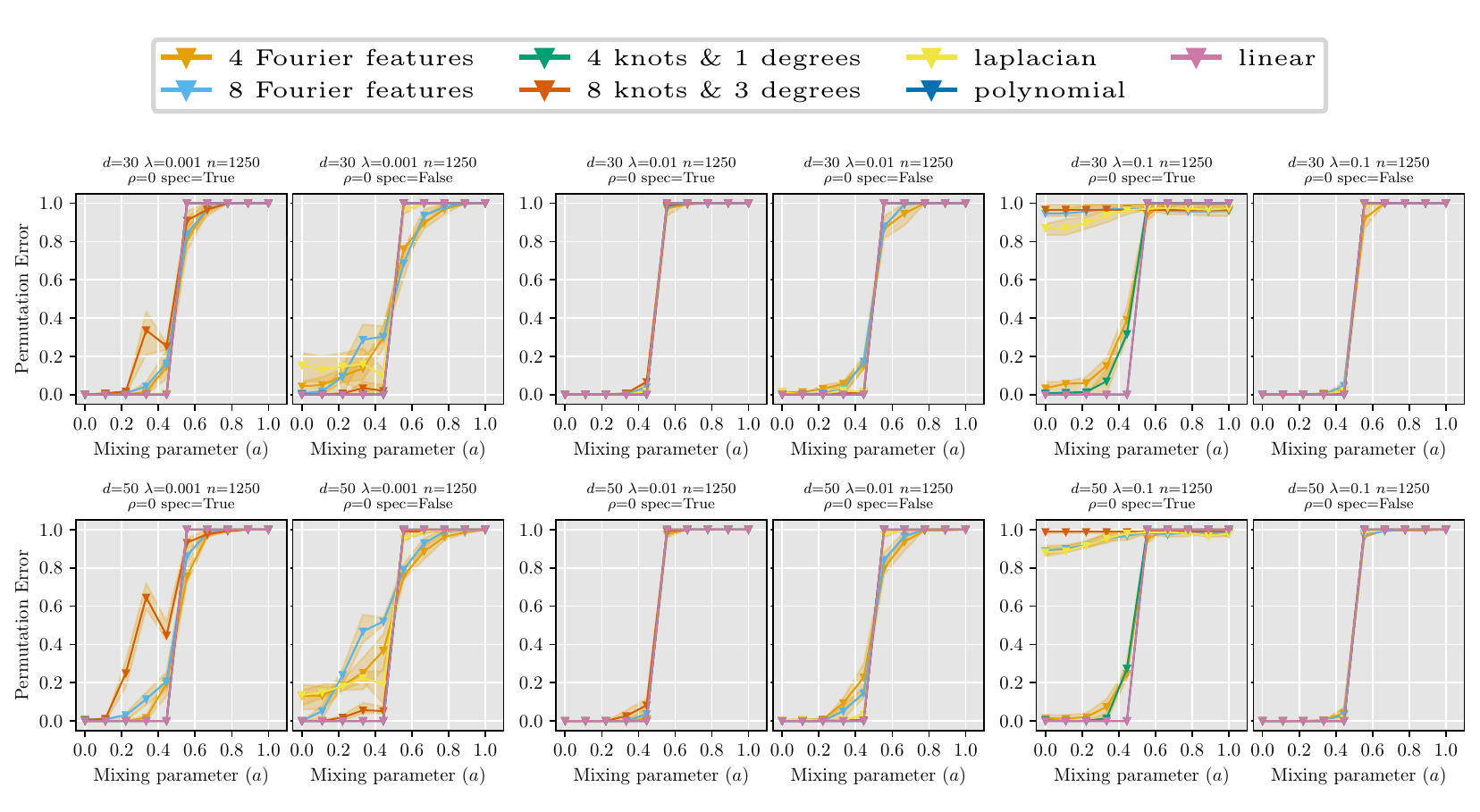}
    }
    \caption{Mean permutation errors on the Toy Dataset when the causal variables are mixed according to a parameter $a$.}
    \label{fig:mixed-perm-error}
\end{figure}

\begin{figure}
    \centerfloat
    \resizebox{1.2\textwidth}{!}{%
        \includegraphics{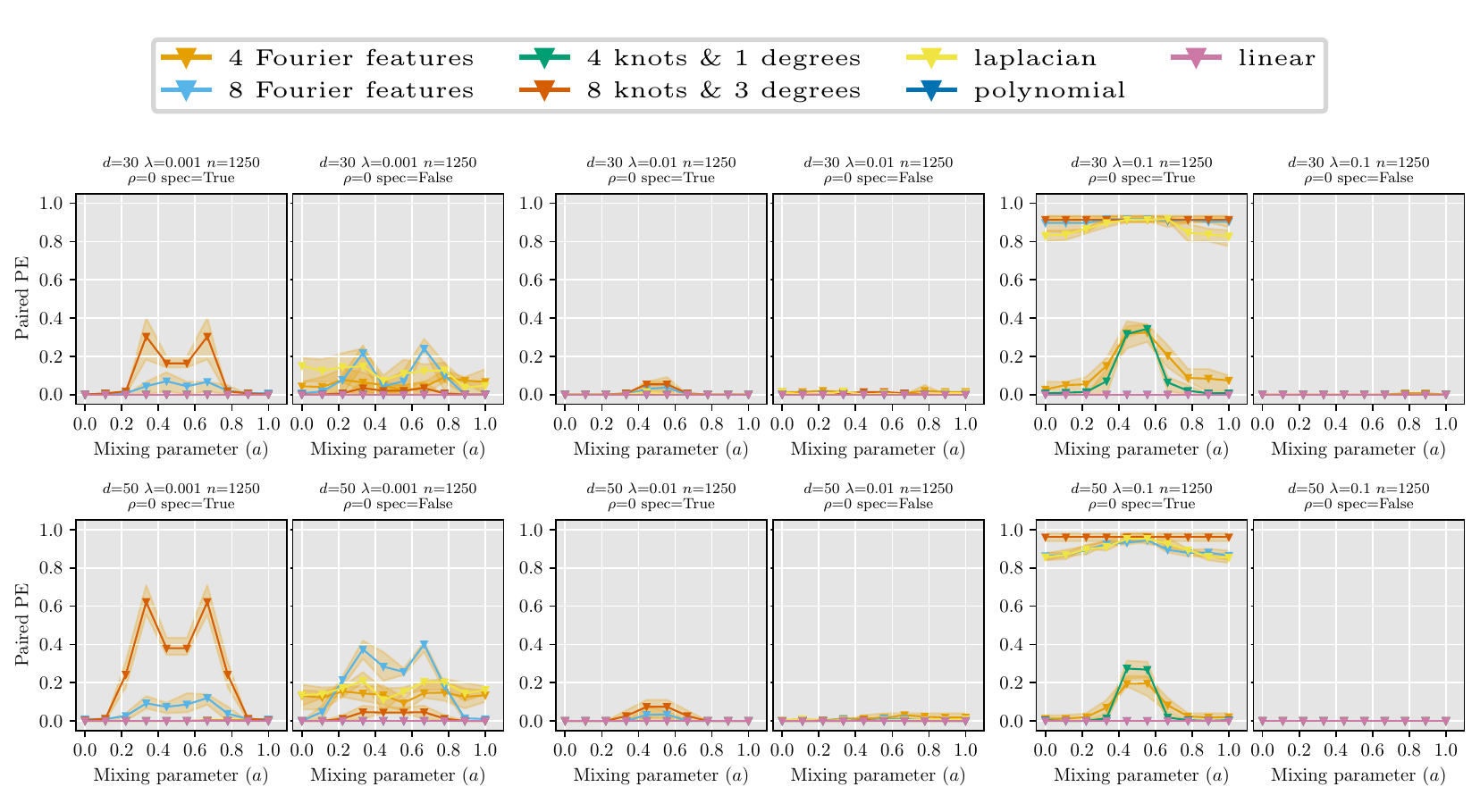}
    }
    \caption{Mean paired permutation errors on the Toy Dataset when the causal variables are mixed according to a parameter $a$.}
    \label{fig:mixed-paired-perm-error}
\end{figure}

\begin{table}[!h]
    \caption{Permutation Errors for the encodings learned in the
    Action/Temporal sparsity datasets and the Temporal Causal3DIdent datasets.
    The $(n)$ indicates
    the number of train and test points used in each column.
    We report the mean and standard deviation over $50$ random seeds and in each
    column we indicate the best method by \textbf{bold}.}
    \label{tbl:perm-error-dms}
    \centering 
    \setlength\tabcolsep{4pt}
    \small
\begin{tabular}{m{2cm}m{1.5cm}cccccc}
        \toprule
        & &  
        \multicolumn{6}{c}{\small Permutation Error $\downarrow$ $(n)$}\\
        \hline
        Model & Method   & 5 & 10 & 20 & 100 & 1000 & 10000\\
\toprule
\bottomrule
\multicolumn{8}{c}{\rule{0pt}{0.3cm}\cellcolor{gray!30}\textbf{Action Sparsity Dataset}}\\
\multirow{8}*{DMS-VAE} & {\notsotiny NN} & \textbf{0.18} {\tiny$\pm$ \textbf{0.02}} & \textbf{0.03} {\tiny$\pm$ \textbf{0.01}} & 0.01 {\tiny$\pm$ 0.01} & \textbf{0.00} {\tiny$\pm$ \textbf{0.00}} & \textbf{0.00} {\tiny$\pm$ \textbf{0.00}} & \textbf{0.00} {\tiny$\pm$ \textbf{0.00}}\\

 & {\notsotiny Pearson} & 0.31 {\tiny$\pm$ 0.02} & 0.04 {\tiny$\pm$ 0.01} & \textbf{0.00} {\tiny$\pm$ \textbf{0.00}} & \textbf{0.00} {\tiny$\pm$ \textbf{0.00}} & \textbf{0.00} {\tiny$\pm$ \textbf{0.00}} & \textbf{0.00} {\tiny$\pm$ \textbf{0.00}}\\

 & {\notsotiny Spearman} & 0.28 {\tiny$\pm$ 0.03} & 0.04 {\tiny$\pm$ 0.01} & 0.00 {\tiny$\pm$ 0.00} & \textbf{0.00} {\tiny$\pm$ \textbf{0.00}} & \textbf{0.00} {\tiny$\pm$ \textbf{0.00}} & \textbf{0.00} {\tiny$\pm$ \textbf{0.00}}\\

 & {\notsotiny Linear} & 0.25 {\tiny$\pm$ 0.03} & 0.03 {\tiny$\pm$ 0.01} & \textbf{0.00} {\tiny$\pm$ \textbf{0.00}} & \textbf{0.00} {\tiny$\pm$ \textbf{0.00}} & \textbf{0.00} {\tiny$\pm$ \textbf{0.00}} & \textbf{0.00} {\tiny$\pm$ \textbf{0.00}}\\

 & {\notsotiny Spline} & 0.36 {\tiny$\pm$ 0.02} & 0.28 {\tiny$\pm$ 0.03} & 0.02 {\tiny$\pm$ 0.01} & \textbf{0.00} {\tiny$\pm$ \textbf{0.00}} & \textbf{0.00} {\tiny$\pm$ \textbf{0.00}} & \textbf{0.00} {\tiny$\pm$ \textbf{0.00}}\\
\hline

    \end{tabular}
\end{table}
\begin{table}[!h]
    \caption{$R^2$ scores on the diagonal for the encodings learned in the Action/Temporal
    sparsity datasets and the Temporal Causal3DIdent datasets. The $(n)$ indicates
    the number of train and test points used in each column. We report the
    mean and standard deviation over $50$ random seeds and in each column 
    we indicate the best method by \textbf{bold}. If a score was below
    $-100$, we indicate this with $\dagger$.}
    \label{tbl:r2-dms}
    \centering 
    \setlength\tabcolsep{2pt}
    \small
    \begin{tabular}{m{2cm}m{1.5cm}cccccc}
        \toprule
        & &  
        \multicolumn{6}{c}{\small $R^2$-score on the diagonal $\uparrow$ $(n)$}\\
        \hline
        Model & Method   & 5 & 10 & 20 & 100 & 1000 & 10000\\
\toprule
\bottomrule
\multicolumn{8}{c}{\rule{0pt}{0.3cm}\cellcolor{gray!30}\textbf{Action Sparsity Dataset}}\\
\multirow{8}*{DMS-VAE} & {\notsotiny NN} & \textbf{0.21} {\tiny$\pm$ \textbf{0.03}} & 0.56 {\tiny$\pm$ 0.01} & 0.68 {\tiny$\pm$ 0.00} & \textbf{0.78} {\tiny$\pm$ \textbf{0.00}} & 0.79 {\tiny$\pm$ 0.00} & \textbf{0.79} {\tiny$\pm$ \textbf{0.00}}\\

 & {\notsotiny Pearson} & -  & -  & -  & -  & -  & - \\

 & {\notsotiny Spearman} & -  & -  & -  & -  & -  & - \\

 & {\notsotiny Linear} & -0.29 {\tiny$\pm$ 0.13} & \textbf{0.56} {\tiny$\pm$ \textbf{0.02}} & \textbf{0.72} {\tiny$\pm$ \textbf{0.00}} & 0.77 {\tiny$\pm$ 0.00} & 0.77 {\tiny$\pm$ 0.00} & 0.77 {\tiny$\pm$ 0.00}\\

 & {\notsotiny Spline}& $\dagger$ & 0.08 {\tiny$\pm$ 0.02} & 0.53 {\tiny$\pm$ 0.01} & 0.78 {\tiny$\pm$ 0.00} & 0.78 {\tiny$\pm$ 0.00} & 0.79 {\tiny$\pm$ 0.00}\\
\hline

    \end{tabular}
\end{table}
\begin{table}[!h]
    \caption{Permutation errors when using binary concepts created based on the 
        ground truth latent variables in the Action/Temporal
    sparsity datasets and the Temporal Causal3DIdent datasets. The $(n)$ indicates
    the number of train and test points used in each column. We report
    the mean and standard deviation over $10$ random seeds in each column and we 
    indicate the best method by \textbf{bold}.}
    \label{tbl:perm-error-bin}
    \centering 
    \setlength\tabcolsep{2pt}
    \small
    \begin{tabular}{m{2cm}m{1.5cm}cccc}
        \toprule
        & &  
        \multicolumn{4}{c}{\small Permutation Errors $\downarrow$ $(n)$}\\
        \hline
        Model & Method   & 20 & 100 & 1000 & 10000\\
\toprule
\bottomrule
\multicolumn{6}{c}{\rule{0pt}{0.3cm}\cellcolor{gray!30}\textbf{Action Sparsity Dataset}}\\
\multirow{5}*{DMS-VAE} & {Linear} & \textbf{0.08} {\scriptsize$\pm$ \textbf{0.03}} & \textbf{0.00} {\scriptsize$\pm$ \textbf{0.00}} & \textbf{0.00} {\scriptsize$\pm$ \textbf{0.00}} & \textbf{0.00} {\scriptsize$\pm$ \textbf{0.00}}\\

 & {Spline} & 0.08 {\scriptsize$\pm$ 0.03} & \textbf{0.00} {\scriptsize$\pm$ \textbf{0.00}} & \textbf{0.00} {\scriptsize$\pm$ \textbf{0.00}} & \textbf{0.00} {\scriptsize$\pm$ \textbf{0.00}}\\

 & {RFF} & 0.08 {\scriptsize$\pm$ 0.03} & \textbf{0.00} {\scriptsize$\pm$ \textbf{0.00}} & \textbf{0.00} {\scriptsize$\pm$ \textbf{0.00}} & \textbf{0.00} {\scriptsize$\pm$ \textbf{0.00}}\\

 & {Laplacian} & 0.08 {\scriptsize$\pm$ 0.03} & \textbf{0.00} {\scriptsize$\pm$ \textbf{0.00}} & \textbf{0.00} {\scriptsize$\pm$ \textbf{0.00}} & \textbf{0.00} {\scriptsize$\pm$ \textbf{0.00}}\\

 & {Two Stage} & 0.08 {\scriptsize$\pm$ 0.03} & 0.04 {\scriptsize$\pm$ 0.03} & \textbf{0.00} {\scriptsize$\pm$ \textbf{0.00}} & \textbf{0.00} {\scriptsize$\pm$ \textbf{0.00}}\\
\hline

    \end{tabular}
\end{table}
\clearpage
\newpage
\begin{table}[!h]
    \caption{Concept Accuracy and Label accuracy 
        when using binary concepts and binary downstream labels created based on the 
        ground truth latent variables in the Action/Temporal
    sparsity datasets and the Temporal Causal3DIdent datasets. The $(n)$ indicates
    the number of train and test points used in each column. We report
    the mean and standard deviation over $10$ random seeds in each column and we 
    indicate the best method by \textbf{bold}.}
    \label{tbl:concept-acc-roc}
    \centerfloat
    \setlength\tabcolsep{2pt}
    \notsotiny
    \begin{tabular}{m{2cm}m{1.5cm}cccc@{\hskip 1em}@{\hskip 1em}cccc}
        \toprule
        & &  
        \multicolumn{4}{c}{\small Concept Acc $\uparrow$ $(n)$} & 
        \multicolumn{4}{c}{\small Label Acc$\uparrow$$(n)$}\\
        \hline
        Model & Method   & 20 & 100 & 1000 & 10000 & 20 & 100 & 1000 & 10000\\
\toprule
\bottomrule
\multicolumn{10}{c}{\rule{0pt}{0.3cm}\cellcolor{gray!30}\textbf{Action Sparsity Dataset}}\\
\multirow{5}*{DMS-VAE} & {Linear} & \textbf{0.83} {\scriptsize$\pm$ \textbf{0.01}} & 0.86 {\scriptsize$\pm$ 0.00} & 0.85 {\scriptsize$\pm$ 0.00} & 0.86 {\scriptsize$\pm$ 0.00} & \textbf{0.77} {\scriptsize$\pm$ \textbf{0.03}} & 0.81 {\scriptsize$\pm$ 0.01} & 0.83 {\scriptsize$\pm$ 0.01} & 0.84 {\scriptsize$\pm$ 0.01}\\

 & {Spline} & 0.83 {\scriptsize$\pm$ 0.01} & \textbf{0.86} {\scriptsize$\pm$ \textbf{0.00}} & 0.85 {\scriptsize$\pm$ 0.00} & 0.86 {\scriptsize$\pm$ 0.00} & 0.72 {\scriptsize$\pm$ 0.03} & \textbf{0.83} {\scriptsize$\pm$ \textbf{0.02}} & \textbf{0.85} {\scriptsize$\pm$ \textbf{0.01}} & 0.86 {\scriptsize$\pm$ 0.01}\\

 & {RFF} & 0.82 {\scriptsize$\pm$ 0.00} & 0.85 {\scriptsize$\pm$ 0.00} & 0.85 {\scriptsize$\pm$ 0.00} & 0.86 {\scriptsize$\pm$ 0.00} & 0.77 {\scriptsize$\pm$ 0.02} & 0.82 {\scriptsize$\pm$ 0.02} & 0.84 {\scriptsize$\pm$ 0.01} & 0.85 {\scriptsize$\pm$ 0.01}\\

 & {Laplacian} & 0.81 {\scriptsize$\pm$ 0.01} & 0.85 {\scriptsize$\pm$ 0.00} & 0.85 {\scriptsize$\pm$ 0.00} & 0.86 {\scriptsize$\pm$ 0.00} & 0.75 {\scriptsize$\pm$ 0.03} & 0.82 {\scriptsize$\pm$ 0.02} & 0.84 {\scriptsize$\pm$ 0.01} & 0.85 {\scriptsize$\pm$ 0.01}\\

 & {Two Stage} & 0.83 {\scriptsize$\pm$ 0.01} & 0.85 {\scriptsize$\pm$ 0.00} & \textbf{0.86} {\scriptsize$\pm$ \textbf{0.00}} & 0.86 {\scriptsize$\pm$ 0.00} & 0.72 {\scriptsize$\pm$ 0.03} & 0.80 {\scriptsize$\pm$ 0.02} & 0.80 {\scriptsize$\pm$ 0.01} & 0.82 {\scriptsize$\pm$ 0.02}\\
\bottomrule
    \end{tabular}
\end{table}
\begin{table}[!h]
    \caption{NIS and OIS scores when using binary concepts created based on the 
        ground truth latent variables in the Action/Temporal
    sparsity datasets and the Temporal Causal3DIdent datasets. The $(n)$ indicates
    the number of train and test points used in each column. We report
    the mean and standard deviation over $10$ random seeds in each column and we 
    indicate the best method by \textbf{bold}.}
    \label{tbl:ois-nis}
    \centerfloat
    \setlength\tabcolsep{2pt}
    \notsotiny
    \begin{tabular}{m{2cm}m{1.5cm}cccc@{\hskip 1em}@{\hskip 1em}cccc}
        \toprule
        & &  
        \multicolumn{4}{c}{\small OIS-score $\downarrow$ $(n)$} & 
        \multicolumn{4}{c}{\small NIS-score $\downarrow$ $(n)$}\\
        \hline
        Model & Method   & 20 & 100 & 1000 & 10000 & 20 & 100 & 1000 & 10000\\
\toprule
\bottomrule
\multicolumn{10}{c}{\rule{0pt}{0.3cm}\cellcolor{gray!30}\textbf{Action Sparsity Dataset}}\\
\multirow{5}*{DMS-VAE} & {Linear} & \textbf{0.63} {\scriptsize$\pm$ \textbf{0.01}} & 0.40 {\scriptsize$\pm$ 0.00} & 0.15 {\scriptsize$\pm$ 0.00} & 0.12 {\scriptsize$\pm$ 0.00} & 0.50 {\scriptsize$\pm$ 0.00} & \textbf{0.50} {\scriptsize$\pm$ \textbf{0.00}} & \textbf{0.50} {\scriptsize$\pm$ \textbf{0.00}} & \textbf{0.50} {\scriptsize$\pm$ \textbf{0.00}}\\

 & {Spline} & 0.63 {\scriptsize$\pm$ 0.01} & 0.40 {\scriptsize$\pm$ 0.00} & 0.15 {\scriptsize$\pm$ 0.00} & 0.12 {\scriptsize$\pm$ 0.00} & 0.46 {\scriptsize$\pm$ 0.02} & 0.50 {\scriptsize$\pm$ 0.00} & 0.50 {\scriptsize$\pm$ 0.00} & 0.50 {\scriptsize$\pm$ 0.00}\\

 & {RFF} & 0.64 {\scriptsize$\pm$ 0.01} & \textbf{0.39} {\scriptsize$\pm$ \textbf{0.00}} & 0.14 {\scriptsize$\pm$ 0.00} & 0.12 {\scriptsize$\pm$ 0.00} & 0.47 {\scriptsize$\pm$ 0.03} & 0.50 {\scriptsize$\pm$ 0.00} & 0.50 {\scriptsize$\pm$ 0.00} & 0.50 {\scriptsize$\pm$ 0.00}\\

 & {Laplacian} & 0.64 {\scriptsize$\pm$ 0.01} & 0.40 {\scriptsize$\pm$ 0.00} & 0.14 {\scriptsize$\pm$ 0.00} & 0.11 {\scriptsize$\pm$ 0.00} & 0.47 {\scriptsize$\pm$ 0.02} & 0.50 {\scriptsize$\pm$ 0.00} & 0.50 {\scriptsize$\pm$ 0.00} & 0.50 {\scriptsize$\pm$ 0.00}\\

 & {Two Stage} & 0.84 {\scriptsize$\pm$ 0.02} & 0.42 {\scriptsize$\pm$ 0.01} & 0.15 {\scriptsize$\pm$ 0.00} & 0.13 {\scriptsize$\pm$ 0.00} & 0.47 {\scriptsize$\pm$ 0.03} & 0.57 {\scriptsize$\pm$ 0.02} & 0.64 {\scriptsize$\pm$ 0.00} & 0.65 {\scriptsize$\pm$ 0.00}\\
\bottomrule
    \end{tabular}
\end{table}

\begin{figure}[t]
    \def\incFigure#1#2{%
        \node{\includegraphics[
            scale=0.85,
            trim={#2cm 0 0 0},
            clip
        ]{./figs/#1/times.pdf}};
        }
    %
    \caption{Execution times of the baseline and multiple versions of our
        estimator on the causal variables and encodings learned based on the
        Action/Temporal Sparsity Dataset and the Temporal Causal3DIdent dataset}
    \label{fig:times-both}
\end{figure}
\begin{figure}[b]
    \def\incFigure#1#2{%
        \node{\includegraphics[
            scale=0.85,
            trim={#2cm 0 0 0},
            clip
        ]{./figs/#1/times_bin.pdf}};
        }
    %
    \caption{Execution times of the concept-based models and multiple versions of our
        estimator on the classification downstream tasks based on the
        Action/Temporal Sparsity Dataset and the Temporal Causal3DIdent dataset}
    \label{fig:times-both-bin}
\end{figure}
\phantom{hey}


\end{document}